\documentclass[twoside,11pt]{article}

\usepackage[preprint]{jmlr2e}

\jmlrheading{1}{2021}{1-48}{4/00}{10/00}{meila00a}{Wei, Gu and Huang}

\ShortHeadings{An Accelerated Variance-Reduced Conditional Gradient Sliding Algorithm}{Wei, Gu and Huang}
\firstpageno{1}

\newcommand{\BlackBox}{\rule{1.5ex}{1.5ex}}  % end of proof
\ifdefined\proof
\renewenvironment{proof}{\par\noindent{\bf Proof\ }}{\hfill\BlackBox\\[2mm]}
\else
\newenvironment{proof}{\par\noindent{\bf Proof\ }}{\hfill\BlackBox\\[2mm]}
\fi
 
\newtheorem{theorem}{Theorem}
\newtheorem{lemma}[theorem]{Lemma}

\newtheorem{corollary}[theorem]{Corollary}
\newtheorem{definition}[theorem]{Definition}

\newtheorem{assum_abbr}[theorem]{A}

\usepackage{amsmath}

\usepackage{color}
\usepackage{pifont}
\usepackage{algorithm,algorithmic}
\usepackage{multirow}
\usepackage{colortbl}
\usepackage{subfigure}

\begin{document}

\title{An Accelerated Variance-Reduced Conditional Gradient Sliding Algorithm for First-order and Zeroth-order Optimization}

\author{\name Xiyuan Wei \email {xywei00@gmail.com} \\
\addr School of Computer \& Software\\
		Nanjing University of Information Science \& Technology\\
		Nanjing, Jiangsu, 210044, China \\
\name Bin Gu \email bin.gu@mbzuai.ac.ae 
\\ \addr MBZUAI, United Arab Emirates  \\
       JD Finance America Corporation \\
\name Heng Huang \email {heng.huang@pitt.edu} \\
\addr Department of Electrical and Computer Engineering\\
       University of Pittsburgh\\
       Pittsburgh, PA, 15261, USA \\
       JD Finance America Corporation 
}

\editor{Editors}

\maketitle

\begin{abstract}
	The conditional gradient algorithm (also known as the Frank-Wolfe algorithm) has recently regained popularity in the machine learning community due to its projection-free property to solve constrained problems. Although many variants of the conditional gradient algorithm have been proposed to improve performance, they depend on first-order information (gradient) to optimize. Naturally, these algorithms are unable to function properly in the field of increasingly popular zeroth-order optimization, where only zeroth-order information (function value) is available. To fill in this gap, we propose a novel Accelerated variance-Reduced Conditional gradient Sliding (ARCS) algorithm for finite-sum problems, which can use either first-order or zeroth-order information to optimize. To the best of our knowledge, ARCS is the first zeroth-order conditional gradient sliding type algorithms solving convex problems in zeroth-order optimization. In first-order optimization, the convergence results of ARCS substantially outperform previous algorithms in terms of the number of gradient query oracle. Finally we validated the superiority of ARCS by experiments on real-world datasets.
\end{abstract}

\section{Introduction}
In this paper, we consider the following constrained finite-sum minimization problem:
\begin{equation}	\label{problem}
	\min_{x\in\mathcal{C}}\left\{f(x) = \frac{1}{n}\sum_{i=1}^{n}f_i(x)\right\}
\end{equation}
where $f$ is ($\tau$-strongly) convex and $L$-smooth, each $f_i$ is $L$-smooth and convex. $\mathcal{C} \subset \mathbb{R}^d$ is a convex set. We are particularly interested in the case where the domain $\mathcal{C}$ admits fast linear optimization. Problem (\ref{problem}) summarizes an extensive number of important learning problems, \emph{e.g.}, matrix completion \citep{zhang2012accelerated}, LASSO regression \citep{tibshirani1996regression}, and sparsity constrained classification \citep{jaggi2013revisiting}. One common approach for solving the constrained problem (\ref{problem}) is the projected gradient algorithm \citep{iusem2003convergence}, which conducts a projection onto the constrained set $\mathcal{C}$ after a gradient step. However, the projection is often expensive to compute for constrained sets, for example, the set of matrices whose nuclear norm is bounded by a positive real number.

The conditional gradient (CG) algorithm (also known as the Frank-Wolfe algorithm \citep{frank1956algorithm}) and its variants are also natural candidates for solving problem (\ref{problem}). Compared to the projected gradient algorithm, CG type algorithms solve a linear optimization subproblem to bound the solution to the constrained set, which does not conduct projection, and solving the subproblem is much faster than conducting a projection. These algorithms thus have better performance due to the projection-free property, and they are gaining popularity in the machine learning community recently. The key step of CG type algorithms can be summarized as follows.
\begin{equation}
	\label{FW}
	\begin{aligned}
		v^s &= \arg\max_{x\in\mathcal{C}}\langle -g^s, x\rangle \\ x^{s} &= (1-\gamma_s)x^{s-1}+\gamma_sv^s
	\end{aligned}
\end{equation}
where $s = 1, 2, ...$ denotes the epoch, $\gamma_s \in [0, 1]$ denotes the step size. The first line of (\ref{FW}) calls a linear oracle to solve the linear optimization subproblem and the second line ensures that $x^s\in\mathcal{C}$ due to the convexity of the constrained set. In the conditional gradient (CG) algorithm, $g^s$ is set to be the gradient $\nabla f(x^{s-1})$.

Formally, we denote gradient query complexity of an algorithm to be the number of calls of gradient query oracle to achieve $\epsilon$-accuracy, \emph{i.e.}, to get an output $x\in\mathcal{C}$ such that $f(x)-\min_{y\in\mathcal{C}}f(y) \leq \epsilon$. The CG algorithm has a gradient query complexity of $\mathcal{O}\left(n\epsilon^{-1}\right)$ for convex problems. \citet{lan2016conditional} proposed a novel variant of the CG algorithm named Conditional Gradient Sliding (CGS) algorithm which calls CG recursively in each iteration to solve a quadratic subproblem. CGS has gradient query complexity of $\mathcal{O}\left(n\epsilon^{-1/2}\right)$ and $\mathcal{O}\left(n\log\left(\epsilon^{-1}\right)\right)$ for convex and strongly-convex problems respectively. SCGS, the stochastic version of CGS, which was also proposed by \citet{lan2016conditional}, has gradient query complexity of $\mathcal{O}\left(\epsilon^{-2}\right)$ for convex problems. The stochastic version of CG was analysed by \citet{hazan2016variance}, which has gradient query complexity of $\mathcal{O}\left(\epsilon^{-3}\right)$ for convex problems. \citet{hazan2016variance} and \citet{yurtsever2019conditional} respectively combines popular variance-reduction techniques with SCGS and proposed STORC and SPIDER CGS. The linear oracle complexity (number of calls of linear oracle) of all these algorithms above is $\mathcal{O}(\epsilon^{-1})$. It can be seen that CGS type algorithms outperform CG type algorithms in terms of gradient query complexity, thus in this paper we focus on CGS type algorithms.

Although the literature is rich, most CGS type algorithms are first-order algorithms, which take advantage of the gradients to optimize. However, in many complex machine learning problems, the explicit gradient of the problem is expensive to compute or even inaccessible, e.g., problems concerning black-box adversarial attacks \citep{chen2017zoo}, bandit optimization \citep{flaxman2005online}, reinforcement learning \citep{choromanski2018structured} and metric learning \citep{kulis2012metric}. Thus first-order algorithms are not applicable to these problems. Zeroth-order algorithm is a promising substitute since it only uses function value to optimize. But zeroth-order conditional gradient sliding type algorithms for the finite-sum problem are understudied. To the best of our knowledge, only \citet{gao2020can} studied the zeroth-order version of SPIDER CGS, but it is only analysed for non-convex problems. Thus there have not been analyses on zeroth-order conditional gradient sliding type algorithms for convex problems.

\begin{table}[tb!]
	\caption{Comparison of conditional gradient sliding type algorithms solving \emph{convex} problems. $D_0 = \mathcal{O}([f(\tilde{x}^0)-f(x^*)]+L\|x^0-x^*\|^2)$. 	\textbf{F} indicates that the result is for the first-order case and \textbf{Z} indicates that the result is for the zeroth-order case. Note that our ARCS is the first zeroth-order conditional gradient sliding type algorithm solving convex problems. $\tilde{\mathcal{O}}$ hides a logarithmic factor.}
	\label{table1}
	\renewcommand{\arraystretch}{1.2}
	\definecolor{MC}{rgb}{.93, .93, .93}
	\centering
	\begin{tabular}{>{\hspace{-.2em}}c<{\hspace{-.2em}}|>{\hspace{-.2em}}c<{\hspace{-.2em}}|>{\hspace{-.2em}}c<{\hspace{-.2em}}|>{\hspace{-.2em}}c<{\hspace{-.2em}}}
		\hline
		\multicolumn{2}{c|}{}& \multicolumn{2}{c}{Oracle Complexity}\\\cline{3-4}
		\multicolumn{2}{c|}{\multirow{-2}{*}{Algorithm}}& Gradient / Function Query & Linear Oracle\\
		\hline
		&CGS \citep{lan2016conditional} & $\mathcal{O}\left(\frac{n}{\sqrt{\epsilon}}\right)$ & $\mathcal{O}\left(\frac{1}{\epsilon}\right)$ \\
		&SCGS \citep{lan2016conditional} & $\mathcal{O}\left(\frac{1}{\epsilon^2}\right)$ & $\mathcal{O}\left(\frac{1}{\epsilon}\right)$ \\
		&SPIDER CGS \citep{yurtsever2019conditional} & $\tilde{\mathcal{O}}\left(n+\frac{1}{\epsilon^2}\right)$ & $\mathcal{O}\left(\frac{1}{\epsilon}\right)$\\
		&STORC \citep{hazan2016variance} & $\tilde{\mathcal{O}}\left(n+\frac{1}{\epsilon^{1.5}}\right)$ & $\mathcal{O}$$\left(\frac{1}{\epsilon}\right)$\\
		\cline{2-4}
		\multirow{-5}{*}{F} & \textbf{Ours (first-order)} &
		$\left\{ \begin{aligned} &\tilde{\mathcal{O}}\left(n\right), &\epsilon\geq \frac{3D_0}n \\ & \tilde{\mathcal{O}}\left(n+\sqrt{\frac{n}{\epsilon}}\right), &\epsilon< \frac{3D_0}n \end{aligned} \right.$
		& $\left\{ \begin{aligned} &\tilde{\mathcal{O}}\left(\frac{1}{\epsilon^2}\right), & \epsilon\geq \frac{3D_0}n \\ &\tilde{\mathcal{O}}\left(\right.n^2+\frac{n}{\epsilon}\left.\right), & \epsilon< \frac{3D_0}n\\ \end{aligned} \right.$\\
		\hline
		\rowcolor{MC}
		Z&\textbf{Ours (zeroth-order)} &
		$\left\{ \begin{aligned} &\tilde{\mathcal{O}}\left(nd\right), &\epsilon\geq \frac{5D_0}n \\ & \tilde{\mathcal{O}}\left(nd+d\sqrt{{n}/{\epsilon}}\right), &\epsilon<\frac{5D_0}n \end{aligned} \right.$
		& $\left\{ \begin{aligned} &\tilde{\mathcal{O}}\left(\frac{1}{\epsilon^2}\right), & \epsilon\geq \frac{5D_0}n \\ &\tilde{\mathcal{O}}\left(\right.n^2+\frac{n}{\epsilon}\left.\right), & \epsilon< \frac{5D_0}n\\ \end{aligned} \right.$\\
		\hline
	\end{tabular}
\end{table}

To fill in the gap, we propose an Accelerated variance-Reduced Conditional gradient Sliding (ARCS) algorithm, which leverages variance-reduction techinque and a novel momentum acceleration technique proposed by \citet{lan2019unified}. Our ARCS algorithm can be used in either first-order or zeroth-order optimization. In first-order optimization, it outperforms all existing conditional gradient type algorithms with respect to gradient query complexity. In zeroth-order optimization, it is the first conditional gradient sliding type algorithm for convex problems. Since zeroth-order algorithms use function values instead of gradients to optimize, it is natural to consider the number of calls of function query oracle to achieve $\epsilon$-accuracy when assessing the performance of zeroth-order algorithms, which we denote to be the function query complexity.

Besides theoretical analyses, we conduct numerical experiments on real-world datasets, and the results also show the optimality of our ARCS in gradient/function query complexity in both first-order and zeroth-order optimization.

\textbf{Contributions.} The  main contributions of this paper are summarized as follows:
\begin{itemize}
	\item We propose an Accelerated variance-Reduced Conditional gradient Sliding (ARCS) algorithm. Our ARCS algorithm is based on the stochastic conditional gradient sliding (SCGS) algorithm and it leverages the variance-reduction technique and a novel momentum acceleration technique. We give convergence results of ARCS in zeroth-order optimization. To the best of our knowledge, our ARCS algorithm is the first zeroth-order conditional gradient sliding type algorithm addressing the convex and strongly-convex finite-sum problems. Numerical experiments also show its optimality.
	
	\item As a by-product, we give convergence results of ARCS in first-order optimization. Both theoretic and numerical results confirm that ARCS have significantly improved gradient query complexities on convex and strongly-convex problems in first-order optimization.
\end{itemize}

\begin{table}[tb!]
	\caption{Comparison of conditional gradient sliding type algorithms solving \emph{strongly convex} problems.  $D_0 = \mathcal{O}([f(\tilde{x}^0)-f(x^*)]+L\|x^0-x^*\|^2)$. \textbf{F} indicates that the result is for the first-order case and \textbf{Z} indicates that the result is for the zeroth-order case. Note that our ARCS is the first zeroth-order conditional gradient sliding type algorithm solving strongly-convex problems. $\tilde{\mathcal{O}}$ hides a logarithmic factor.}
	\label{table2}
	\renewcommand{\arraystretch}{1.2}
	\definecolor{MC}{rgb}{.93, .93, .93}
	\centering
	\begin{tabular}{c|c|c|c}
		\hline
		\multicolumn{2}{c|}{}& \multicolumn{2}{c}{Oracle Complexity}\\
		\cline{3-4}
		\multicolumn{2}{c|}{\multirow{-2}{*}{Algorithm}}& Gradient / Function Query & Linear Oracle\\
		\hline
		&CGS \citep{lan2016conditional} & $\tilde{\mathcal{O}}\left(n\sqrt{\frac{L}{\tau}}\right)$ & $\mathcal{O}\left(\frac{1}{\epsilon}\right)$ \\
		&SCGS \citep{lan2016conditional} & $\mathcal{O}\left(\frac{1}{\epsilon}\right)$ & $\mathcal{O}\left(\frac{1}{\epsilon}\right)$ \\
		&STORC \citep{hazan2016variance} & $\tilde{\mathcal{O}}\left(n+\frac{L^2}{\tau^2}\right)$ & $\mathcal{O}\left(\frac{1}{\epsilon}\right)$\\
		\cline{2-4}
		\multirow{-4}{*}{F}&\textbf{Ours (first-order)} & 
		$\left\{ \begin{aligned} &\begin{aligned}
			\tilde{\mathcal{O}}&\left(n\right), \\ &\epsilon\geq 5D_0/n \textrm{ or } n\geq 3L/4\tau
		\end{aligned} \\ &\begin{aligned}
			\tilde{\mathcal{O}}&\left(n+\sqrt{\frac{nL}{\tau}}\right), \\ &\epsilon< 5D_0/n \textrm{ and } n< 3L/4\tau
		\end{aligned} \end{aligned} \right.$
		& $\left\{ \begin{aligned} &\tilde{\mathcal{O}}\left(\frac{1}{\epsilon^2}\right), & \epsilon\geq \frac{5D_0}n \\ &\tilde{\mathcal{O}}\left(n^2+\frac{n}{\epsilon}\right), & \epsilon< \frac{5D_0}n\\ \end{aligned} \right.$ \\
		\hline
        \rowcolor{MC}
		Z&\textbf{Ours (zeroth-order)} &
        $\left\{ \begin{aligned} &\begin{aligned}
				\tilde{\mathcal{O}}&\left(nd\right), \\ &\epsilon\geq 8D_0/n \textrm{ or } n\geq 3L/4\tau
			\end{aligned} \\ &\begin{aligned}
				\tilde{\mathcal{O}}&\left(nd+d\sqrt{\frac{nL}{\tau}}\right), \\ & \epsilon< 8D_0/n \textrm{ and } n< 3L/4\tau
			\end{aligned} \end{aligned} \right.$
		& $\left\{ \begin{aligned} &\tilde{\mathcal{O}}\left(\frac{1}{\epsilon^2}\right), & \epsilon\geq \frac{8D_0}n \\ &\tilde{\mathcal{O}}\left(n^2+\frac{n}{\epsilon}\right), & \epsilon< \frac{8D_0}n\\ \end{aligned} \right.$ \\
		\hline
	\end{tabular}
\end{table}

\section{Related Works}
\textbf{Conditional Gradient Algorithms.} \citet{frank1956algorithm} proposed the conditional gradient (CG) algorithm, also known as Frank-Wolfe (FW) algorithm, to avoid projection in solving constrained problems. Motivated by removing the influence of “bad” visited vertices, \citet{wolfe1970convergence} proposed away-step Frank-Wolfe (AFW) algorithm. \citet{goldfarb2017linear} proposed ASFW, the stochastic version of AFW. \citet{lan2016conditional} proposed a variant of CG called conditional gradient sliding (CGS) algorithm which calls CG recursively in each iteration until a good solution is obtained. SCGS, the stochastic version of CGS was also proposed by \citet{lan2016conditional}. \citet{hazan2016variance} gave convergence results of the stochastic version of CG, which is called SFW. Also, \citet{hazan2016variance} combined the variance-reduction technique proposed by \citet{johnson2013accelerating} with SFW and SCGS to get SVRF and STORC respectively. \citet{yurtsever2019conditional} combined another variance-reduction technique proposed by \citet{fang2018spider} with SCGS to get SPIDER CGS.

\noindent\textbf{Zeroth-Order Optimization.} Zeroth-order optimization is a classical technique in the optimization community. \citet{nesterov2017random} proposed zeroth-order gradient descent (ZO-GD) algorithm. Then \citet{ghadimi2013stochastic} proposed its stochastic counterpart ZO-SGD. \citet{lian2016comprehensive} proposed an asynchronous zeroth-order stochastic gradient (ASZO) algorithm for parallel optimization. \citet{gu2018faster} further improved the convergence rate of ASZO by combining variance reduction technique with coordinate-wise gradient estimators. \citet{liu2018zeroth} proposed ZO-SVRG based algorithms using three different gradient estimators. \citet{fang2018spider} proposed a SPIDER based zeroth-order method named SPIDER-SZO.\citet{ji2019improved} further improved ZO SVRG based and SPIDER based algorithms. \citet{chen2019zo} proposed zeroth-order adaptive momentum method (ZO-AdaMM). \citet{chen2020accelerated} proposed ZO-Varag which leverages acceleration and variance-reduced technique. \citet{sahu2019towards} proposed zeroth-order versions of (stochastic) conditional gradient method. \citet{balasubramanian2018zeroth} proposed zeroth-order versions of stochastic conditional gradient method and stochastic conditional gradient sliding method. These zeroth-order conditional gradient type algorithms mentioned above did not consider the finite-sum problem (\ref{problem}).

\section{Preliminaries}
For simplicity, we denote $x^* \overset{\textrm{def}}{=} \arg\min_{x\in\mathcal{C}}f(x)$ to be the optimal solution to the problem (\ref{problem}) and denote $\|\cdot\|$ to be the norm associated with inner product in $\mathbb{R}^d$. First we give formal definitions of some basic concepts.

\begin{definition}
	\label{def1}
	For function $f : \mathbb{R}^d \rightarrow \mathbb{R}$, we have
	\begin{itemize}
		\item $f$ is $L$-smooth if $f$ has continuous gradients and $\forall \, x, y \in \mathbb{R}^d$, it satisfies $|f(y)-f(x)-\langle\nabla f(x), y - x\rangle| \leq \frac{L}{2}||y-x||^2$.
		
		\item $f$ is convex if $\forall \, x, y \in \mathbb{R}^d$, it satisfies $f(y) \geq f(x) + \langle\nabla f(x), y - x\rangle $.
		
		\item $f$ is $\tau$-strongly-convex if $f(x)-\frac{\tau}{2}||x||^2$ is convex, \emph{i.e.}, $\forall \, x, y \in \mathbb{R}^d$, it satisfies $f(y) \geq f(x) + \langle\nabla f(x), y - x\rangle + \frac{\tau}{2}||y - x||^2$.
	\end{itemize}
\end{definition}
From Definition \ref{def1} we know if $f$ is convex, then it is $0$-strongly-convex. Next we give assumptions that will be used in our analyses.

\subsection{Assumptions}
\begin{assum_abbr} \label{a1} \label{assum1}
	$f$ is convex and each $f_i, i = 1, ..., n$ is $L$-smooth.
\end{assum_abbr}
\begin{assum_abbr} \label{a2} \label{assum2}
	$f$ is $\tau$-strongly-convex with $\tau > 0$ and each $f_i, i = 1, ..., n$ is $L$-smooth.
\end{assum_abbr}
\begin{assum_abbr} \label{a3}
	For any $x, y\in\mathcal{C}$, there exists $D<\infty$ such that $\|x-y\|\leq D$.
\end{assum_abbr}

Assumption \ref{a3} is standard for the convergence analysis of conditional gradient type algorithms \citep{jaggi2013revisiting,lan2016conditional,hazan2016variance}. Next we specify the oracles that are used in our algorithms.

\subsection{Oracles}
We introduce three oracles called in our algorithm.
\begin{itemize}
	\item Gradient Query Oracle (GQO): GQO returns the gradient of a given component function at point $x$, which is $\nabla f_i(x)$.
	
	\item Function Query Oracle (FQO): FQO returns the value of a given component function at point $x$, which is $f_i(x)$.
	
	\item Linear Oracle (LO): LO sovles the linear programming problem for vector $u$ and returns $\arg\max_{v\in\mathcal{C}}\langle u, v\rangle$.
\end{itemize}
In this paper, we consider the following two cases:
\begin{itemize}
	\item First-order Case: We have access to GQO and LO.
	
	\item Zeroth-order Case: We have access to FQO and LO.
\end{itemize}

\subsection{Zeroth-order Gradient Estimation}
For the zeroth-order case, we only have access to the function query oracle rather than the gradient query oracle. Then we can utilize the difference of the function value at two close points to estimate the gradient. Two gradient estimators are widely used in zeroth-order optimization: the two-point Gaussian random gradient estimator \citep{nesterov2017random} and the coordinate-wise gradient estimator \citep{lian2016comprehensive}. \citet{liu2018zeroth} showed that the coordinate-wise gradient estimator has better performance than the two-point Gaussian random gradient estimator. So we only consider the coordinate-wise gradient estimator in this paper, which is defined as follows:
\begin{equation}    \label{coord_estimator}
	\hat{\nabla}_{coord}f(x) = \sum_{i=1}^{d}\frac{f(x+\mu e_i)-f(x-\mu e_i)}{2\mu}e_i
\end{equation}
where $e_i$ is the $i$-th vector of the standard basis of $\mathbb{R}^d$ and $\mu>0$ is a smoothing parameter.

\section{Algorithms and Analyses}
\citet{lan2016conditional} proposed a novel variant of the conditional gradient algorithm named Conditional Gradient Sliding (CGS) algorithm. CGS calls the linear oracle recursively in each iteration until a good solution is obtained. The idea of CGS can be summarized as follows:
\begin{equation}	\label{cgs}
	\begin{aligned}
		z^s =& (1-\alpha_s)y^{s-1}+\alpha_sx^{s-1}\\
		x^s =& \textrm{CondG}(\nabla f(z^s), x^{s-1}, 0, \gamma_s, 0, \eta_s) \;\;\;\; \textrm{(Algo. \ref{algo2})}\\
		y^s =& (1-\alpha_s)y^{s-1}+\alpha_sx^s\\
	\end{aligned}
\end{equation}
The second line of CGS calls Algorithm \ref{algo2}. In each iteration, the linear oracle is called to produce an output $v_t$ of (\ref{condg}). If the value $V_{g, u, y, \gamma, \tau}(u_t)\leq\eta$, then it sets $u^+=u_t$ and returns. Thus Algorithm \ref{algo2} outputs a solution $u^+$ such that
\begin{equation}
    \max_{x\in\mathcal{C}}\langle\nabla h(u^+), u^+-x\rangle\leq\eta
\end{equation}
where $h$ is a quadratic function defined as
\begin{equation}	\label{condg_problem}
	h(x)\overset{\textrm{def}}{=}\gamma\left[\langle g, x\rangle + \frac{\tau}{2}\|x-y\|^2\right]+\frac{1}{2}\|x-u\|^2
\end{equation}
On the other hand, if $V_{g, u, y, \gamma, \tau}(u_t)>\eta$, then $u_t$ is updated with line search, \emph{i.e.}, $u_{t+1} = (1-\beta_t)u_t+\beta v_t$, where
\begin{equation}
	\beta_t = \arg\min_{\beta\in[0,1]} h((1-\beta)u_t+\beta v_t)
\end{equation}
Denote $u^*=\arg\min_{u\in\mathcal{C}} h(u)$, from the convexity of $h$, the output $u^+$ satisfies
\begin{equation}
	h(u^+)-h(u^*) \leq \langle \nabla h(u^+), u^+-u^*\rangle \leq \eta
\end{equation}
Then it is clear that Algorithm \ref{algo2} is in fact the standard conditional gradient algorithm (\ref{FW}) minimizing $h$. In the CGS algorithm (\ref{cgs}), we have $x^s = \textrm{CondG}(\nabla f(z^s), x^{s-1}, 0, \gamma_s, 0, \eta_s)$, so $h$ in CGS can be rewritten as
\begin{equation}    \label{cgsH}
	h_s'(x)\overset{\textrm{def}}{=}\gamma_s\langle \nabla f(z^s), x\rangle+\frac{1}{2}\|x-x_{s-1}\|^2
\end{equation}
Note that if $\mathcal{C} = \mathbb{R}^d$, then the minimizer of (\ref{cgsH}) has a closed form solution and it is in fact an accelerated gradient descent step. We choose the more complicated form (\ref{condg_problem}) since it gives our algorithm better performance when problem (\ref{problem}) is strongly convex ($\tau > 0$). When problem (\ref{problem}) is convex ($\tau = 0$), (\ref{condg_problem}) is identical to (\ref{cgsH}).

\begin{algorithm}[htb!]
	\caption{CondG Algorithm}
	\label{algo2}
	\begin{algorithmic}[1]
	    \STATE {\bfseries Input:} $(g, u, y, \gamma, \tau, \eta)$
	    \STATE Define $h(x)\overset{\textrm{def}}{=}\gamma\left[\langle g, x\rangle + \frac{\tau}{2}\|x-y\|^2\right]+\frac{1}{2}\|x-u\|^2$
		\STATE Set $u_1=u$.
		\FOR {$t=1, 2, ...$}
		\STATE Let $v_t$ be an optimal solution of the subproblem
		\begin{equation} \label{condg}
		    V_{g, u, y, \gamma, \tau}(u_t) = \max_{x\in\mathcal{C}}\langle\nabla h(u_t), u_t-x\rangle
		\end{equation}
		\IF {$V_{g, u, y, \gamma, \tau}(u_t)\leq\eta$}
		\STATE {\bfseries Output} $u^+=u_t$.
		\ELSE
		\STATE Set $u_{t+1}=(1-\beta_t)u_t+\beta_tv_t$ with $\beta_t = \max\left\{0, \min\left\{1, \frac{\langle\nabla h(u_t), u_t-v_t\rangle}{(\gamma\tau+1)\|u_t-v_t\|^2}\right\}\right\}  \label{alpha}$
		\ENDIF
		\ENDFOR
	\end{algorithmic}
\end{algorithm}

\citet{lan2019unified} proposed a VAriance-Reduced Accelerated Gradient (Varag) algorithm for unconstrained finite-sum problems, which leverages the variance-reduction technique and a novel momentum technique. Inspired by Varag, we combined variance-reduction technique and momentum with the conditional gradient sliding algorithm, and proposed our Accelerated variance-Reduced Conditional gradient Sliding (ARCS) algorithm. The detail of ARCS is described in Algorithm \ref{algo1}.

\begin{algorithm}[htb!]
	\caption{Accelerated variance-Reduced Conditional gradient Sliding (ARCS) algorithm}
	\label{algo1}
	\begin{algorithmic}[1]
		\STATE {\bfseries Input:} $x_0\in\mathcal{C}, \{T_s\}, \{\gamma_s\}, \{\alpha_s\}, \{p_s\}, \{\theta_t\}, \{\eta_{s,t}\}$
		\STATE Set $\tilde{x}^0 = x^0$.
		\FOR {$s=1, 2, ...$}
		\STATE Set $\tilde{x} = \tilde{x}^{s-1}$ and 
		$\tilde{g}=\begin{cases}
			\nabla f(\tilde{x}), & \textit{// for first-order case}\\
			\hat{\nabla}_{coord} f(\tilde{x}),  & \textit{// for zeroth-order case}
		\end{cases}
		\label{pivotal}$
		\STATE Set $x_0 = x^{s-1}$, $\bar{x}_0 = \tilde{x}$ and $T=T_s$.
		\FOR {$t = 1, ..., T$}
		\STATE Pick $i_t\in\{1, ..., n\}$ randomly.
		\STATE Set $\underline{x}_t = \frac{\left[(1+\tau\gamma_s)(1-\alpha_s-p_s)\bar{x}_{t-1}+\alpha_sx_{t-1}+ (1+\tau\gamma_s)p_s\tilde{x}\right]}{(1+\tau\gamma_s(1-\alpha_s))} $
		\STATE \label{grad_blend} $G_t =\begin{cases}
			\nabla f_{i_t}(\underline{x}_t)-\nabla f_{i_t}(\tilde{x})+\tilde{g}, \qquad & \textit{// for first-order case}\\
			\hat{\nabla}_{coord} f_{i_t}(\underline{x}_t)-\hat{\nabla}_{coord} f_{i_t}(\tilde{x})+\tilde{g}, & \textit{// for zeroth-order case}\\
		\end{cases}$
		\STATE $x_t=\textrm{CondG}(G_t, x_{t-1}, \underline{x}_t, \gamma_s, \tau, \eta_{s,t})$ \qquad\; \label{CondG} \textit{// Algorithm \ref{algo2}}
		\STATE $\bar{x}_t = (1-\alpha_s-p_s)\bar{x}_{t-1}+\alpha_sx_t+p_s\tilde{x}$.
		\ENDFOR
		\STATE Set $x^s = x_T$ and $\tilde{x}^s = \sum_{t=1}^{T}(\theta_t\bar{x}_t)/\sum_{t=1}^T\theta_t$.
		\ENDFOR
	\end{algorithmic}
\end{algorithm}

At the beginning of epoch $s$, ARCS computes a full gradient $\tilde{g}$ at point $\tilde{x}^{s-1}$, which is the solution provided by the preceding epoch. Then the full gradient is used repeatedly in each inner loop to form a gradient blending $G_t$. This is the classic variance-reduction technique proposed by \citet{johnson2013accelerating}. Each inner loop maintains three sequences: $\{\underline{x}_t\}, \{x_t\}, \{\bar{x}_t\}$, which is a novel momentum technique proposed by \citet{lan2019unified} and plays an important role in the acceleration scheme. The choice of the additional parameters $\{T_s\}, \{p_s\}, \{a_s\}, \{\gamma_s\}, \{\eta_{s,t}\}, \{\theta_t\}$ will be specified in our convergence analyses for first-order and zeroth-order case, convex and strongly-convex problems respectively. First we provide the convergence results of our ARCS solving \emph{convex} problems. The proof of Theorem \ref{theorem1} is left in the appendix.

\begin{theorem}[Convex]	\label{theorem1}
	Suppose Assumptions \ref{a1} and \ref{a3} holds. Denote $s_0=\lfloor\log n\rfloor+1$, set
	\[T_s =\begin{cases}
		2^{s-1}, &s\leq s_0\\
		T_{s_0}, &s>s_0\\
	\end{cases} ,\; \alpha_s =\begin{cases}
		\frac{1}{2}, &s\leq s_0\\
		\frac{2}{s-s_0+4}, &s>s_0\\
	\end{cases}, \; p_s = \frac{1}{2}, \; \eta_{s,t} = \frac{D_0}{sT_sL}, \; \theta_t =\begin{cases}
		\frac{\gamma_s}{\alpha_s}(\alpha_s+p_s), &t\leq T_s-1\\[2pt]
		\frac{\gamma_s}{\alpha_s}, &t=T_s\\
	\end{cases}\]
	where $D_0$ will be specified below for two cases respectively.
	
	$\bullet$ For the first-order case, set $\gamma_s=\frac{1}{3L\alpha_s}, D_0=4(f(\tilde{x}^0)-f(x^*))+3L\|x^0-x^*\|^2$, we have
	\[\mathbb{E}\left[f(\tilde{x}^S)-f(x^*)\right] \leq
    \left\{\begin{aligned}
		&\frac{3D_0(\log S + 2)}{2^{S+1}}, &S\leq s_0\\[2pt]
		&\frac{48D_0(\log S + 2)}{n(S-s_0+4)^2}, &S>s_0\\
	\end{aligned}\right. \]
	
	$\bullet$ For the zeroth-order case, set $\gamma_s = \frac{1}{5L\alpha_s}, D_0=4(f(\tilde{x}^0)-f(x^*))+5L\|x^0-x^*\|^2$, we have
	\[\begin{aligned}
		\renewcommand{\arraystretch}{2}
		\mathbb{E}\left[f(\tilde{x}^S)-f(x^*)\right] \leq 
        \left\{\begin{aligned}
			&\frac{5D_0(\log S + 2)}{2^{S+1}}+D\mu L\sqrt{\frac{d^2}{2}+2d}+\frac{\mu^2Ld}{2}, &S\leq s_0\\[2pt]
			&\begin{aligned}
				\frac{80D_0(\log S + 2)}{n(S-s_0+4)^2}+\Delta^\mu
			\end{aligned}, &S>s_0\\
		\end{aligned}\right.
	\end{aligned}\]
	where $\Delta^\mu = 2(S-s_0+4)\mu^2Ld+4(S-s_0+4)D\mu L\sqrt{\frac{d^2}{2}+2d}$.
\end{theorem}

\begin{corollary}   \label{coro1}
	With parameters set in Theorem \ref{theorem1}, for \emph{convex} problems, we have ($\tilde{\mathcal{O}}$ hides a logarithmic factor)
	
	$\bullet$ For the first-order case, the gradient query complexity can be bounded as \[N_{GQO}=\begin{cases}
		\tilde{\mathcal{O}}\left(n\log\frac{D_0}{\epsilon}\right), &\epsilon\geq\tilde{\mathcal{O}}\left(\frac{D_0}{n}\right)\\[2pt]
		\tilde{\mathcal{O}}\left(n\log n+\sqrt{\frac{nD_0}{\epsilon}}\right), &\epsilon<\tilde{\mathcal{O}}\left(\frac{D_0}{n}\right)\\
	\end{cases}\]
	
	$\bullet$ For the zeroth-order case, the function query complexity can be bounded as \[N_{FQO}=\begin{cases}
		\tilde{\mathcal{O}}\left(nd\log\frac{D_0}{\epsilon}\right), &\epsilon\geq\tilde{\mathcal{O}}\left(\frac{D_0}{n}\right)\\[2pt]
		\tilde{\mathcal{O}}\left(nd\log n+d\sqrt{\frac{nD_0}{\epsilon}}\right), &\epsilon<\tilde{\mathcal{O}}\left(\frac{D_0}{n}\right)\\
	\end{cases}\]
	
	$\bullet$ For both cases, the linear oracle complexity can be bounded as \[N_{LO} = \begin{cases}
	   \tilde{\mathcal{O}}\left(\frac{1}{\epsilon^2}\right), &\epsilon\geq\tilde{\mathcal{O}}\left(\frac{D_0}{n}\right)\\[2pt]
		\tilde{\mathcal{O}}\left(n^2 + \frac{n}{\epsilon}\right), &\epsilon<\tilde{\mathcal{O}}\left(\frac{D_0}{n}\right)\\
	\end{cases}\]
\end{corollary}

From Table \ref{table1} it can be seen that the known best algorithms with lowest gradient query complexity for solving convex problems are CGS and STORC, whose results are $\mathcal{O}\left(n\epsilon^{-1/2}\right)$ and $\mathcal{O}\left(n\log\left(\epsilon^{-1}\right)+\epsilon^{-3/2}\right)$ respectively. CGS outperforms STORC when $\epsilon < n^{-1}$ and STORC takes the lead otherwise. But it is easy to verify that the gradient query complexity of ARCS is always lower that of CGS and STORC. The gradient query complexity of ARCS is $\tilde{\mathcal{O}}\left(n\log\left(\epsilon^{-1}\right)\right)$ when $\epsilon \geq 3D_0/n$ and $\tilde{\mathcal{O}}\left(n\log\left(n\right) + n^{1/2}\epsilon^{-1/2}\right)$ = $\tilde{\mathcal{O}}\left( n^{1/2}\epsilon^{-1/2}\right)$ otherwise. Thus ARCS outperforms all existing algorithms in terms of gradient query complexity.

However, Theorem \ref{theorem1} (Theorem \ref{theorem2} as well) implies that ARCS has a higher linear oracle complexity than CGS and STORC. To explain this, we make a comparison between ARCS and STORC since they are both accelerated variance-reduced stochastic conditional gradient sliding algorithms. For completeness we include STORC and its key theorems in the appendix. The key differences between ARCS and STORC lie in a) the choice of $\gamma_s$ and $\alpha_s$, b) the choice of $\{\underline{x}_t\}$, $\{x_t\}$ and $\{\bar{x}_t\}$, c) minibatch of stochastic gradients.

To be specific, a) the choice of $\gamma_s$ and $\alpha_s$ contributes most to the difference in convergence results. We have $\alpha\gamma = \mathcal{O}(L^{-1})$ for each inner iteration in both ARCS and STORC. For each epoch (\emph{i.e.}, $s$ is fixed), $\gamma$ and $\alpha$ in ARCS are constant while in STORC, $\alpha$ diminish with a rate of $\mathcal{O}(t^{-1})$. This adds to $\eta_{s,t}$ a factor of $\mathcal{O}(t^{-1})$ so that $\eta_{s,t}$ can be chosen $\mathcal{O}(t)$ times larger. Thus the linear oracle complexity is lowered down (the linear oracle complexity is proportional to $\eta_{s,t}^{-1}$ from \citealt{jaggi2013revisiting}). However, this comes with a price. The decrease of $\alpha$ requires a larger minibatch of stochastic gradients in each inner iteration to lower down the variance. Thus STORC has a higher gradient query complexity, which becomes even higher than CGS when $\epsilon < n^{-1}$. b) the choice of $\{\underline{x}_t\}$, $\{x_t\}$ and $\{\bar{x}_t\}$ leverages the acceleration technique and yields accelerated convergence rates for both ARCS and STORC. c) minibatch of stochastic gradients in STORC is required by the choice of $\gamma_s$ and $\alpha_s$ to lower down the variance of stochastic gradients in the analyses. The points discussed above also work on CGS. In fact, CGS is a deterministic conditional gradient sliding algorithm and it a) benefits from choice of $\gamma_s$ and $\alpha_s$ as STORC, b) maintains similar acceleration sequences $\{z^s\}$, $\{x^s\}$ and $\{y^s\}$ (see (\ref{cgs})). Next we give convergence results of our ARCS solving \emph{strongly-convex} problems.

\begin{theorem}[Strongly-convex]	\label{theorem2}
	Suppose Assumptions \ref{a2} and \ref{a3} hold. Denote $s_0=\lfloor\log n\rfloor+1$, set \[T_s=\begin{cases}
		2^{s-1},& s\leq s_0\\
		T_{s_0},& s> s_0\\
	\end{cases}, \; \alpha_s=\begin{cases}
		\frac{1}{2},& s\leq s_0\\[2pt]
		\min\left\{\sqrt{\frac{n}{4\varsigma}}, \frac{1}{2}\right\},& s> s_0\\
	\end{cases}\] \[p_s=\frac{1}{2}, \; \theta_t=\begin{cases}
		\Gamma_{t-1}-(1-\alpha_s-p_s)\Gamma_{t}, &t\leq T_s-1\\
		\Gamma_{t-1}, &t=T_{s}\\
	\end{cases}, \;\eta_{s,t}=\begin{cases}
		\frac{D_0}{sT_sL},& s\leq s_0\\[2pt]
		\frac{\left(\frac{4}{5}\right)^{s-s_0-1}D_0}{snL},& s>s_0 \textrm{ and } n\geq\varsigma\\[2pt]
		\frac{\left(\frac{1}{\Gamma_{T_{s_0}}}\right)^{s-s_0-1}D_0}{snL},& s>s_0 \textrm{ and } n<\varsigma\\
	\end{cases}\]
	where $\varsigma, \Gamma_{t}, D_0$ will be specified below for two cases respectively.
	
	$\bullet$ For the first-order case, set $\gamma_s=\frac{1}{3L\alpha_s}, \varsigma=\frac{3L}{4\tau}, \Gamma_{t}=\left(1+\tau\gamma_s\right)^{t}, D_0=4(f(\tilde{x}^0)-f(x^*))+3L\|x^0-x^*\|^2$. We have
	\[\begin{aligned}
		\mathbb{E}\left[f(\tilde{x}^{S})-f(x^*)\right]\leq \begin{cases}
			\frac{3D_0(\log S + 2)}{2^{S+1}},& s\leq s_0\\[2pt]
			\left(\frac{4}{5}\right)^{S-s_0}\frac{5D_0(\log S + 2)}{n},& s>s_0 \textrm{ and } n\geq\varsigma\\[2pt]
			\left(1+\frac{1}{2}\sqrt{\frac{n\tau}{3L}}\right)^{-(S-s_0)}\frac{5D_0(\log S + 2)}{n},& s>s_0 \textrm{ and } n<\varsigma\\
		\end{cases}
	\end{aligned}\]
	
	$\bullet$ For the zeroth-order case, set $\gamma_s=\frac{1}{12dL\alpha_s}, \varsigma=\frac{5L}{4\tau}, \Gamma_{t}=\left(1+\frac{\tau\gamma_s}{2}\right)^{t}, D_0=4(f(\tilde{x}^0)-f(x^*))+5L\|x^0-x^*\|^2$. We have
	\[\begin{aligned}
	\small 
		\mathbb{E}\left[f(\tilde{x}^{S})-f(x^*)\right] \leq\hspace{-0.08in}
		\begin{cases}
			\frac{5D_0(\log S + 2)}{2^{S+1}}+\frac{\mu^2L^2d\left(d+4\right)}{4\tau}+2\mu^2Ld,& s\leq s_0\\[2pt]
			\left(\frac{4}{5}\right)^{S-s_0}\frac{8D_0(\log S + 2)}{n}+\Delta_1^\mu,& s>s_0,\; n\geq\varsigma\\[2pt]
			\begin{aligned}
				\left(1+\frac{1}{4}\sqrt{\frac{n\tau}{5L}}\right)^{-(S-s_0)}&\frac{8D_0(\log S + 2)}{n}\\&+\Delta_2^\mu
			\end{aligned}
			,& s>s_0,\; n<\varsigma\\
		\end{cases}
	\end{aligned}\]
	where $\Delta_1^\mu = \frac{5\mu^2L^2d(d+4)}{\tau}+12\mu^2Ld, \Delta_2^\mu = \frac{5\mu^2L^2d(d+4)}{\tau}+4\mu^2Ld\left(1+2\sqrt{\frac{5L}{n\tau}}\right)$.
\end{theorem}

\begin{corollary}
	With parameters set in Theorem \ref{theorem2}, for \emph{strongly-convex} problems, we have ($\tilde{\mathcal{O}}$ hides a logarithmic factor)
	
	$\bullet$ For the first-order case, the gradient query complexity can be bounded as \[
	\begin{aligned}
	    N_{GQO}= \begin{cases}
		\tilde{\mathcal{O}}\left(n\log\left(\frac{D_0}{\epsilon}\right)\right), &\epsilon\geq \tilde{\mathcal{O}}(\frac{D_0}{n}) \textrm{ or } n\geq \varsigma\\[2pt]
		\tilde{\mathcal{O}}\left(n\log n+\sqrt{\frac{nL}{\tau}}\log\left(\frac{D_0}{n\epsilon}\right)\right), &\epsilon< \tilde{\mathcal{O}}(\frac{D_0}{n}) \textrm{ and } n< \varsigma\\
	\end{cases}
	\end{aligned}
	\]
	$\bullet$ For the zeroth-order case, the function query complexity can be bounded as \[
	\begin{aligned}
	    N_{FQO}= \begin{cases}
		\tilde{\mathcal{O}}\left(nd\log\left(\frac{D_0}{\epsilon}\right)\right), &\epsilon\geq\tilde{\mathcal{O}}(\frac{D_0}{n}) \textrm{ or } n\geq\varsigma\\[2pt]
		\tilde{\mathcal{O}}\left(nd\log n+d\sqrt{\frac{nL}{\tau}}\log\left(\frac{D_0}{n\epsilon}\right)\right), &\epsilon<\tilde{\mathcal{O}}(\frac{D_0}{n}) \textrm{ and } n<\varsigma
	\end{cases}
	\end{aligned}
	\]
	$\bullet$ For both cases, the linear oracle complexity can be bounded as \[N_{LO} = \begin{cases}
	   \tilde{\mathcal{O}}\left(\frac{1}{\epsilon^2}\right), &\epsilon\geq\tilde{\mathcal{O}}(D_0/n) \textrm{ or } n\geq\varsigma\\[2pt]
		\tilde{\mathcal{O}}\left(n^2 + \frac{n}{\epsilon}\right), &\epsilon<\tilde{\mathcal{O}}(D_0/n) \textrm{ and } n<\varsigma
	\end{cases}\]
\end{corollary}

From Table \ref{table2} it can be seen that the known best algorithms with lowest gradient query oracle complexity for solving \emph{strongly-convex} problems are CGS and STORC, whose results are $\mathcal{O}\left(nL^{1/2}\tau^{-1/2}\log\left(\epsilon^{-1}\right)\right)$ and $\mathcal{O}\left(\left(n+L^2\tau^{-2}\right)\log\left(\epsilon^{-1}\right)\right)$. CGS outperforms STORC when $n < L^{3/2}\tau^{-3/2}$ and STORC takes the lead otherwise. But it is easy to verify that the gradient query complexity of ARCS is always lower that of CGS and STORC. The gradient query complexity of ARCS is $\tilde{\mathcal{O}}\left(n\log\left(\epsilon^{-1}\right)\right)$ when $\epsilon \geq 5D_0/n$ or $n\geq 3L/4\tau$ and $\tilde{\mathcal{O}}\left(n\log\left(n\right) + n^{1/2}L^{1/2}\tau^{-1/2}\log\left(\epsilon^{-1}\right)\right)$ = $\tilde{\mathcal{O}}\left(n^{1/2}L^{1/2}\tau^{-1/2}\log\left(\epsilon^{-1}\right)\right)$ otherwise. Thus ARCS outperforms all existing algorithms in terms of gradient query complexity. But the linear oracle complexity of ARCS is higher than that of CGS and STORC, which is discussed after Corollary \ref{coro1}.

\section{Experiments}
In this section, we validate the effectiveness of our ARCS with experiments on different machine learning tasks. We conduct two experiments on ARCS and other compared algorithms listed in Table \ref{table1} with five real-world datasets. Specifically, the first experiment is the low-rank matrix completion task, and the second experiment addresses the sparsity-constrained logistic regression problem.

\subsection{Low-Rank Matrix Completion Problem}
In this experiment, we intend to recover a low rank matrix by solving the following matrix completion problem:
\begin{equation}	\label{matrix}
	\begin{aligned}
		\min_{\|X\|_*\leq R}\sum_{(i, j)\in\Omega}\left(X_{i,j}-Y_{i,j}\right)^2
	\end{aligned}
\end{equation}
where $\|\cdot\|_*$ denotes the nuclear norm. $Y\in\mathbb{R}^{d_1\times d_2}$ is a matrix whose elements were partly observed, and $\Omega$ denotes the set of subscripts of observed elements. Following \citet{gu2019asynchronous}, we use the low-rank matrix completion problem to achieve image recovery such that $Y$ in (\ref{matrix}) is the matrix of an incomplete gray-scale image \footnote{The gray-scale images can be found at \url{https://homepages.cae.wisc.edu/~ece533/images/}}, and the solution $X$ is a low rank matrix of the complete image we get. Specifically, we choose five images, which are Barbara ($512\times512$ pixels), Cameraman ($256\times256$ pixels), Goldhill ($512\times512$ pixels), Lena ($512\times512$ pixels) and Mountain ($640\times480$ pixels). To get incomplete images, we eliminate 30\% of the pixels in each of them. Note that for the matrix completion problem (\ref{matrix}) the zeroth-order coordinate-wise gradient estimator (\ref{coord_estimator}) happens to be the true gradient, and the number of function query to construct a coordinate estimator of gradient is $2d$ times of the number of gradient query to construct a true gradient. Thus the figures of results for \emph{zeroth-order} case are exactly the same as that for the \emph{first-order} case, except that the $x$-axis is slightly different. The parameters are set according to Theorem \ref{theorem2} since the problem is quadratic. For the three variance-reduced algorithms, \emph{i.e.}, ARCS, STORC and SPIDER CGS, we use a mini batch of 256 and for SCGS, we set the mini batch according to \citep[Algo.~4]{lan2016conditional} since a mini batch of 256 leads to poor performance of SCGS. The results are shown in Figure \ref{mat}, where (a)-(e) are results for the \emph{first-order} case and (f)-(j) are results for the \emph{zeroth-order} case. It can be seen that our ARCS outperform all other algorithms compared in terms of gradient/function query complexity.

\begin{figure}
	\vspace{-9pt}
	\hspace{-10pt}
	\begin{minipage}[t]{0.5\linewidth}
	    \centering
    	\includegraphics[width = \textwidth]{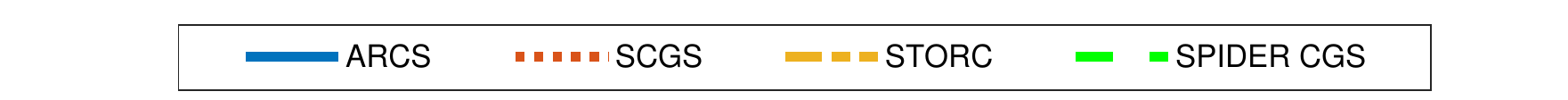}
	\end{minipage}
	\hspace{-12pt}
	\begin{minipage}[t]{0.5\linewidth}
	    \centering
    	\includegraphics[width = \textwidth]{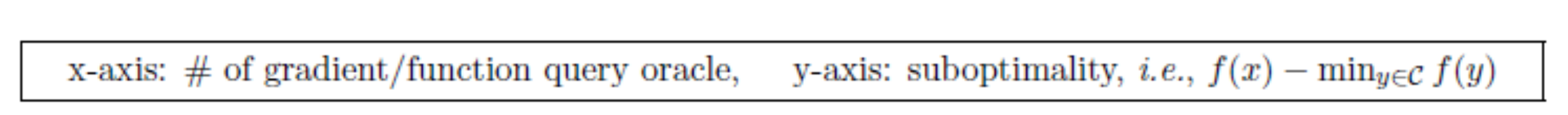}
	\end{minipage}
	\vspace{-14pt}
\end{figure}

\begin{figure}[tb!]
	\centering
	\hspace{-10pt}
	\subfigure[Barbara]{
		\begin{minipage}[t]{0.19\linewidth}
			\centering
			\includegraphics[width = 1\textwidth]{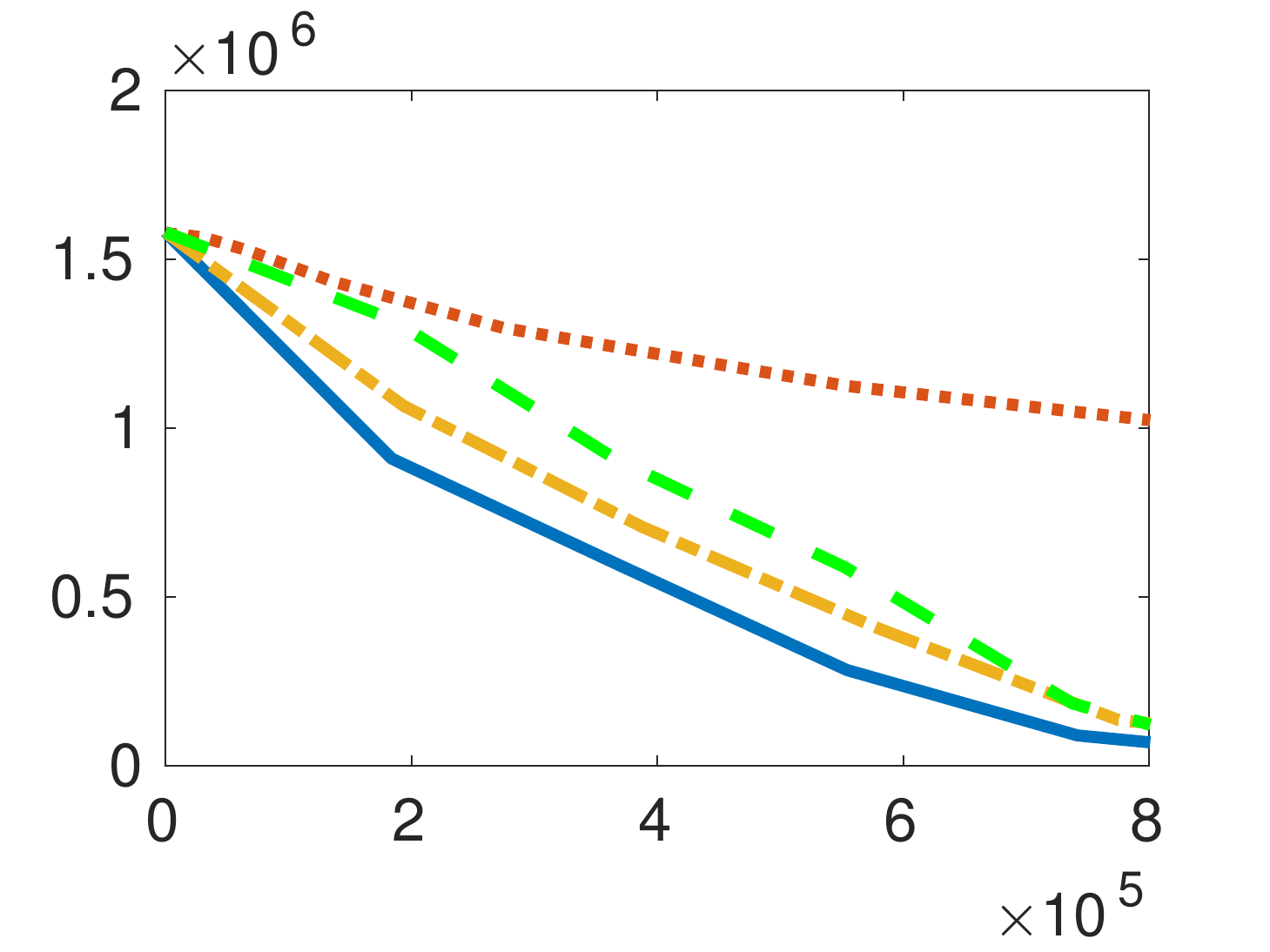}
		\end{minipage}
	}
	\hspace{-12pt}
	\subfigure[Cameraman]{
		\begin{minipage}[t]{0.19\linewidth}
			\centering
			\includegraphics[width = 1\textwidth]{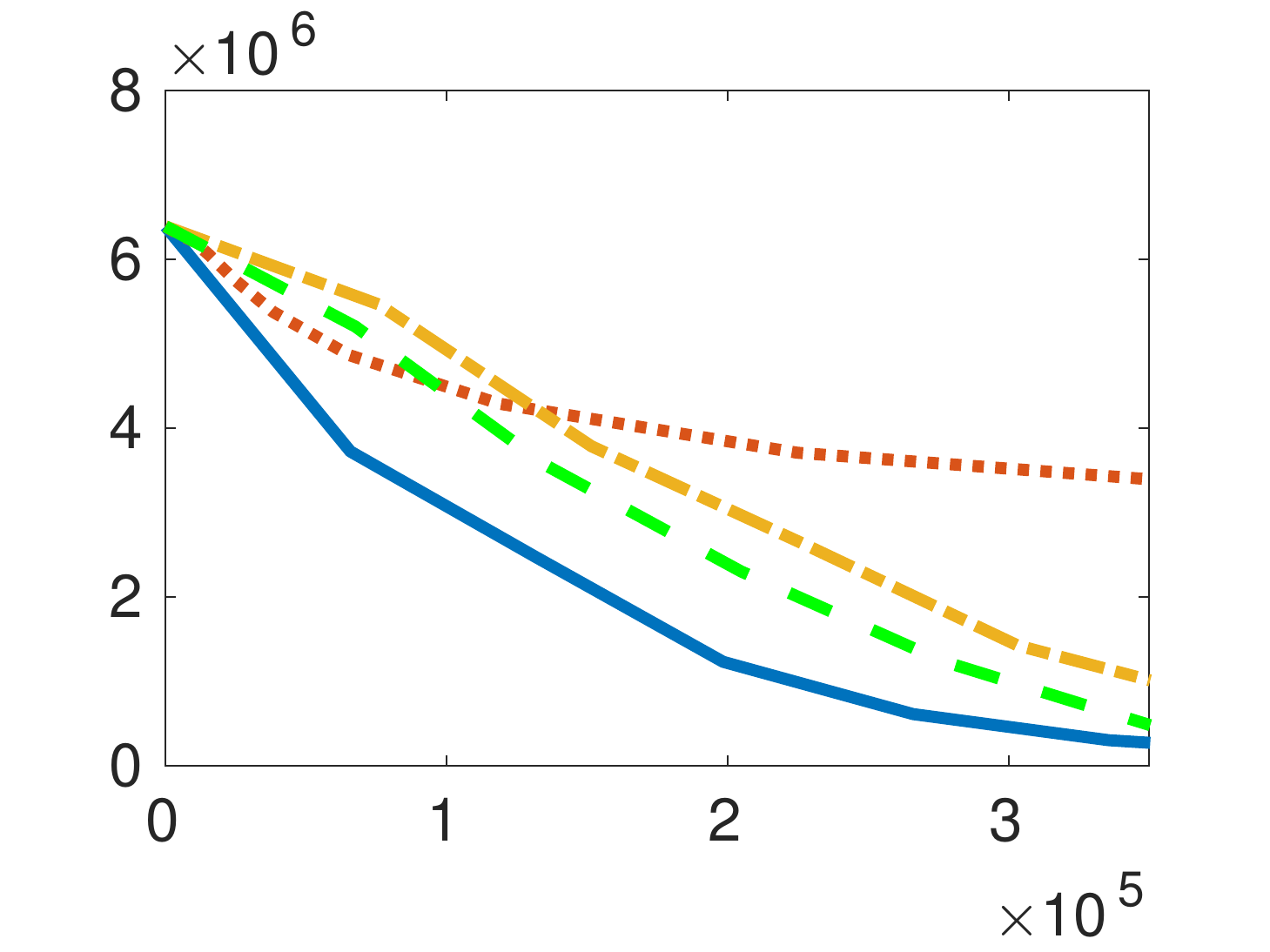}
		\end{minipage}
	}
	\hspace{-12pt}
	\subfigure[Goldhill]{
		\begin{minipage}[t]{0.19\linewidth}
			\centering
			\includegraphics[width = 1\textwidth]{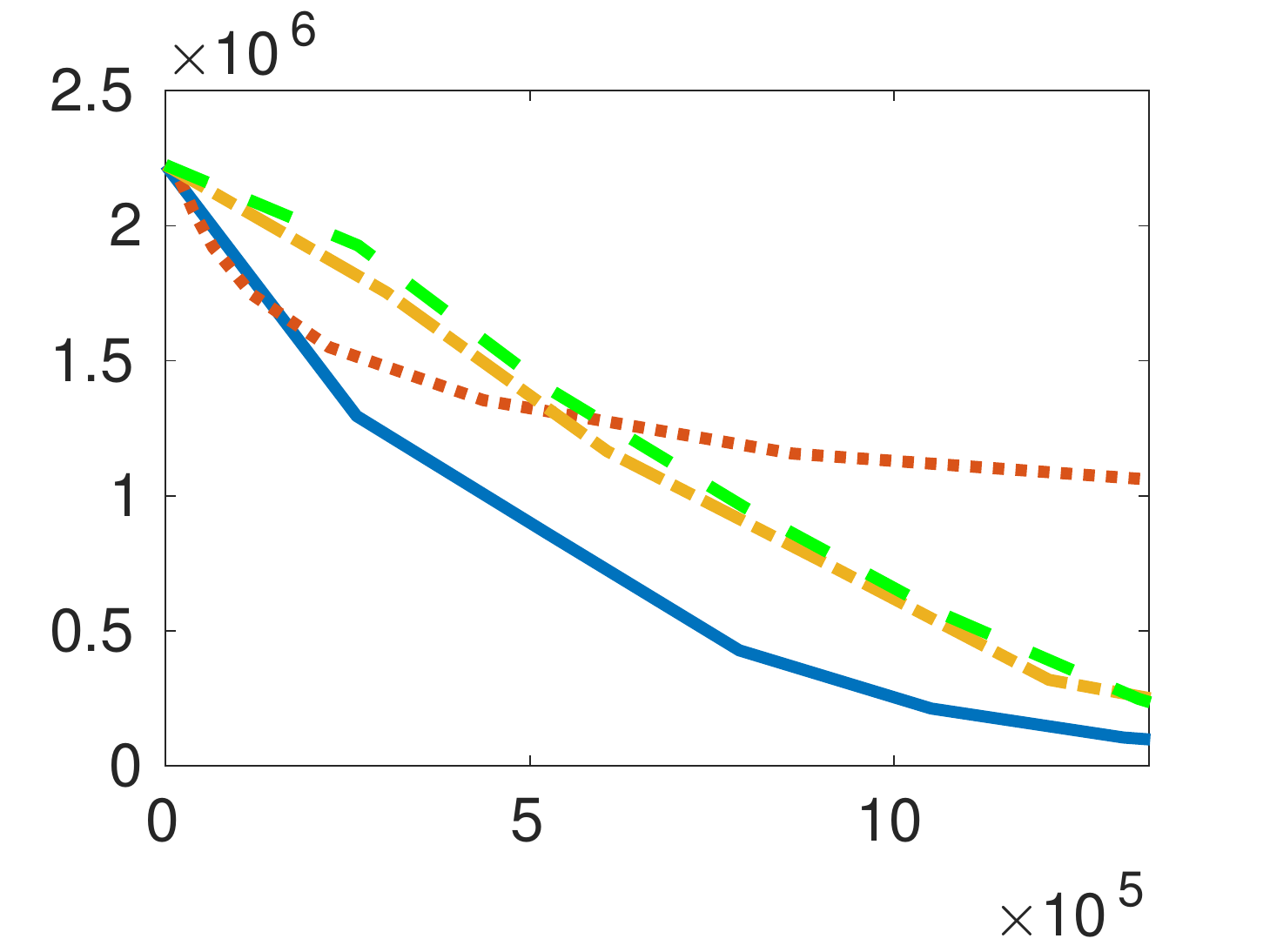}
		\end{minipage}
	}
	\hspace{-12pt}
	\subfigure[Lena]{
		\begin{minipage}[t]{0.19\linewidth}
			\centering
			\includegraphics[width = 1\textwidth]{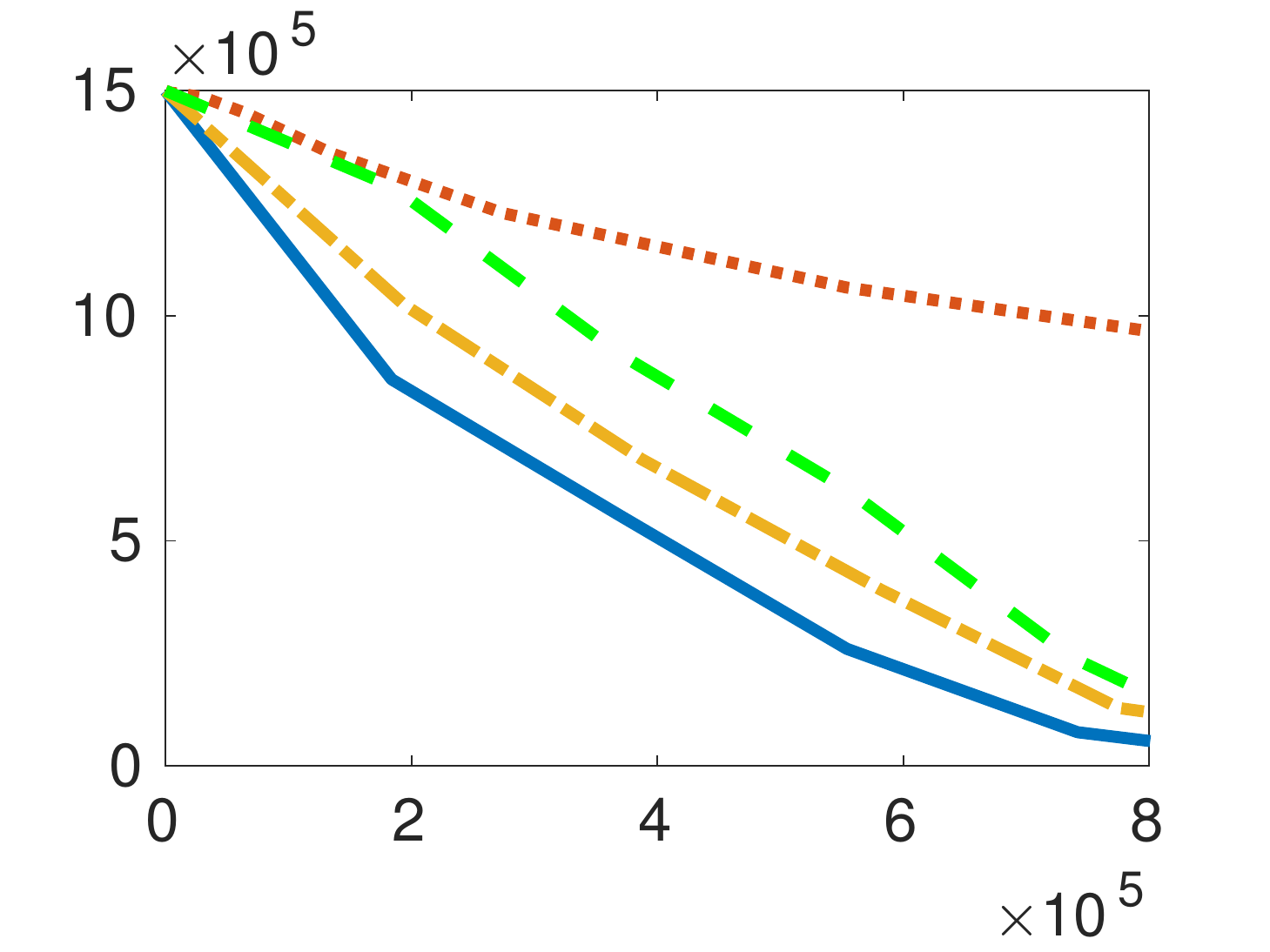}
		\end{minipage}
	}
	\hspace{-12pt}
	\subfigure[Mountain]{
		\begin{minipage}[t]{0.19\linewidth}
			\centering
			\includegraphics[width = 1\textwidth]{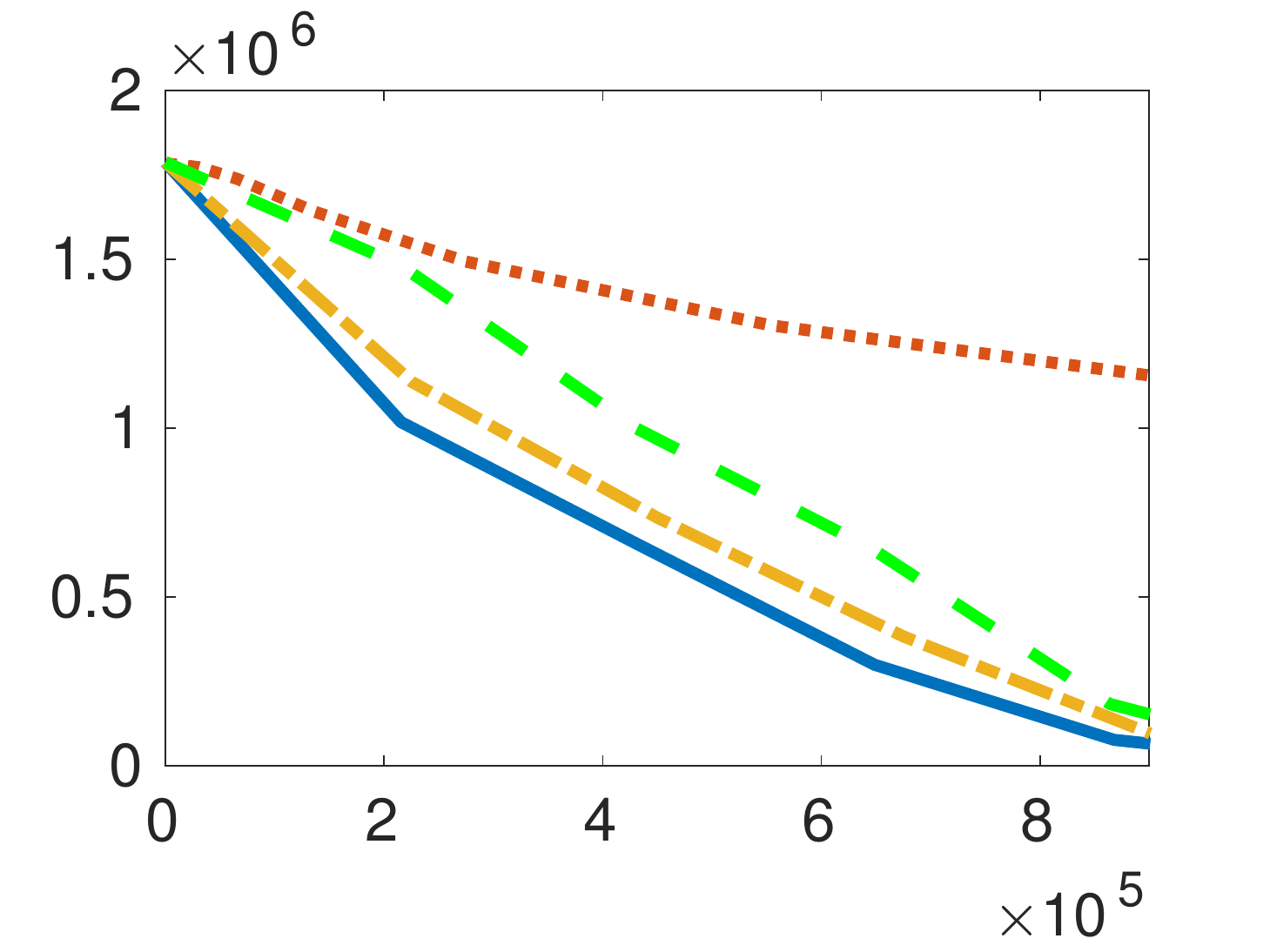}
		\end{minipage}
	}
	\hspace{-10pt}
\vspace{-5pt}
	
	\hspace{-10pt}
	\subfigure[Barbara]{
		\begin{minipage}[t]{0.19\linewidth}
			\centering
			\includegraphics[width = 1\textwidth]{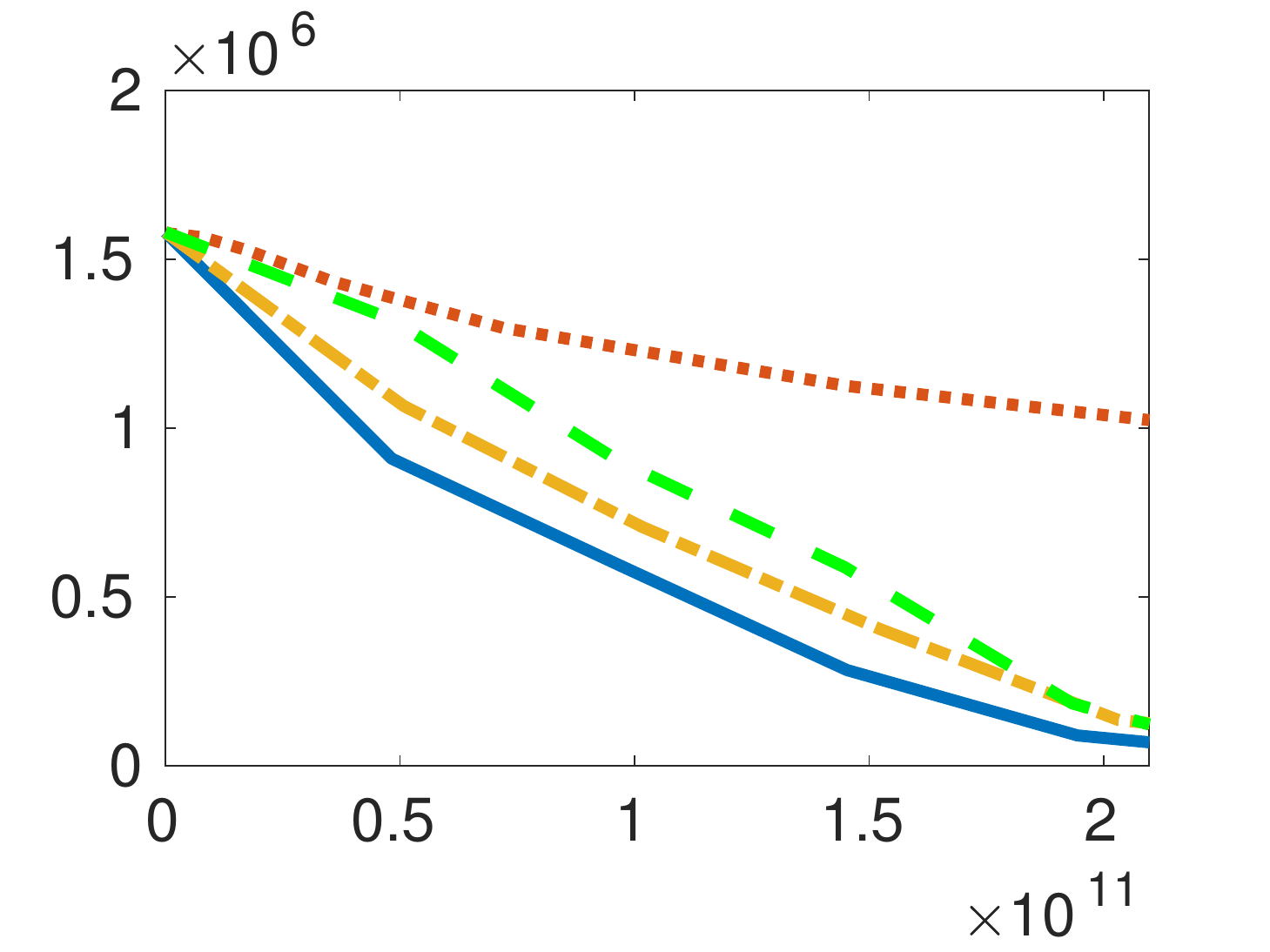}
		\end{minipage}
	}
	\hspace{-12pt}
	\subfigure[Cameraman]{
		\begin{minipage}[t]{0.19\linewidth}
			\centering
			\includegraphics[width = 1\textwidth]{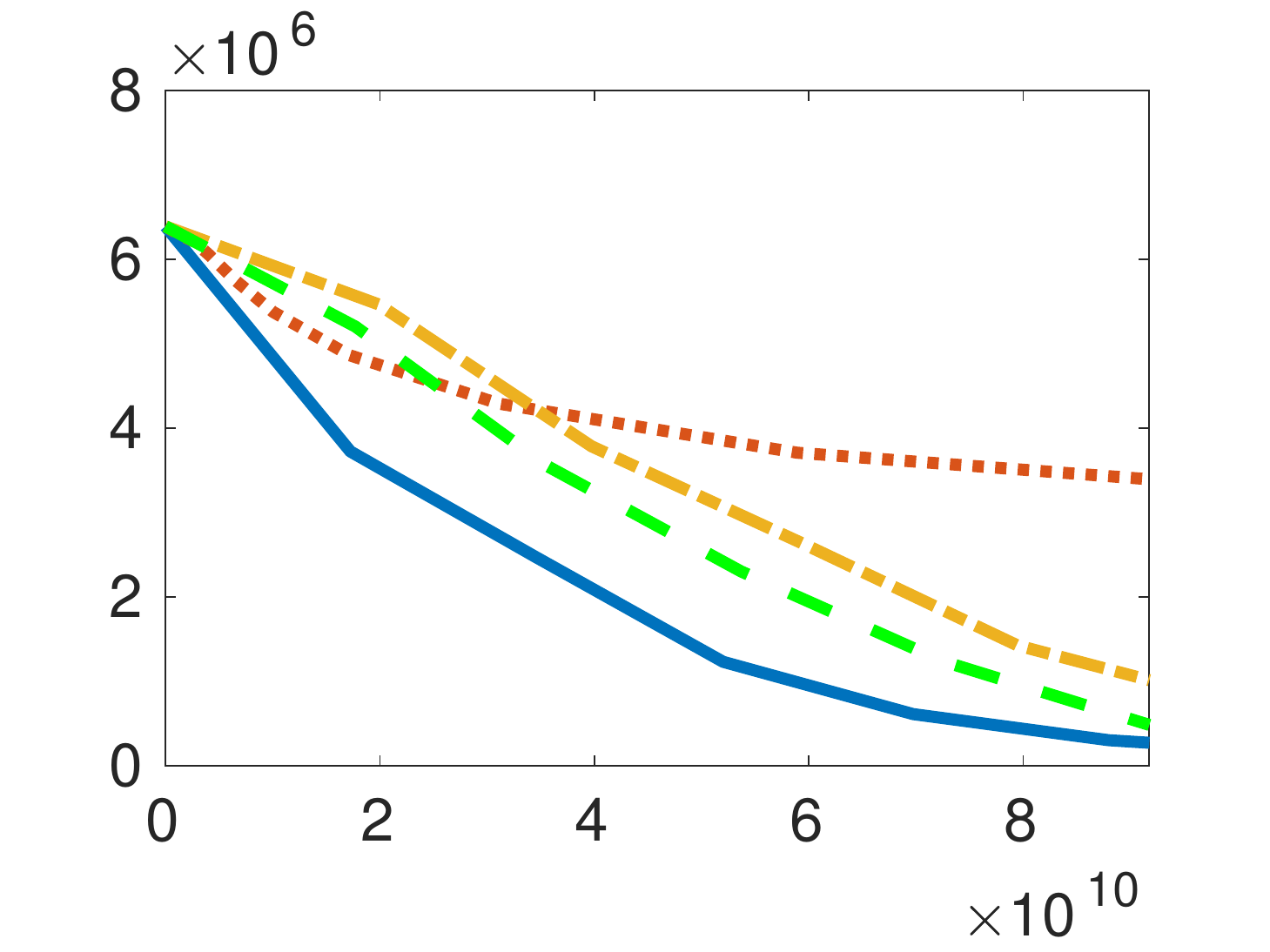}
		\end{minipage}
	}
	\hspace{-12pt}
	\subfigure[Goldhill]{
		\begin{minipage}[t]{0.19\linewidth}
			\centering
			\includegraphics[width = 1\textwidth]{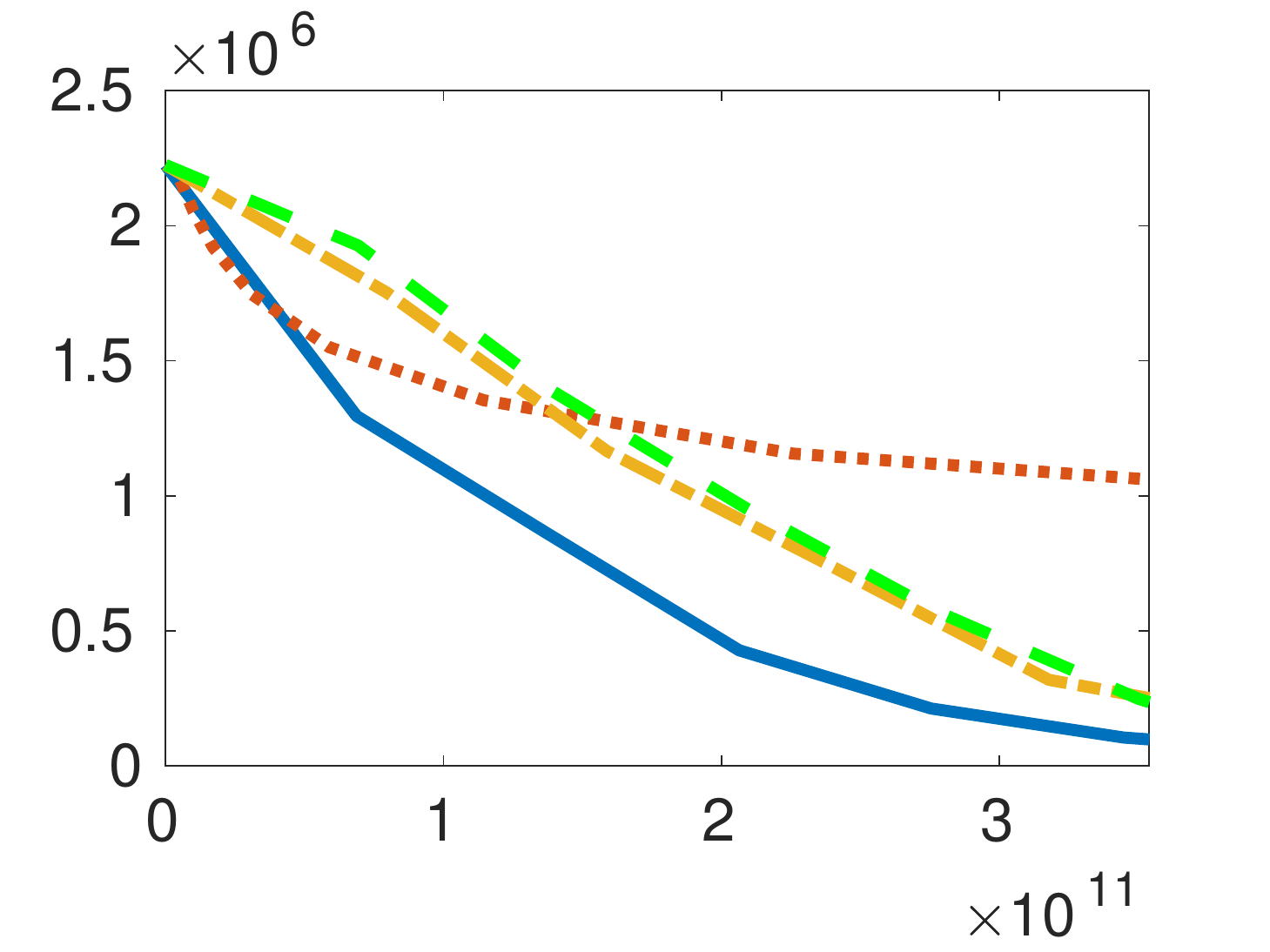}
		\end{minipage}
	}
	\hspace{-12pt}
	\subfigure[Lena]{
		\begin{minipage}[t]{0.19\linewidth}
			\centering
			\includegraphics[width = 1\textwidth]{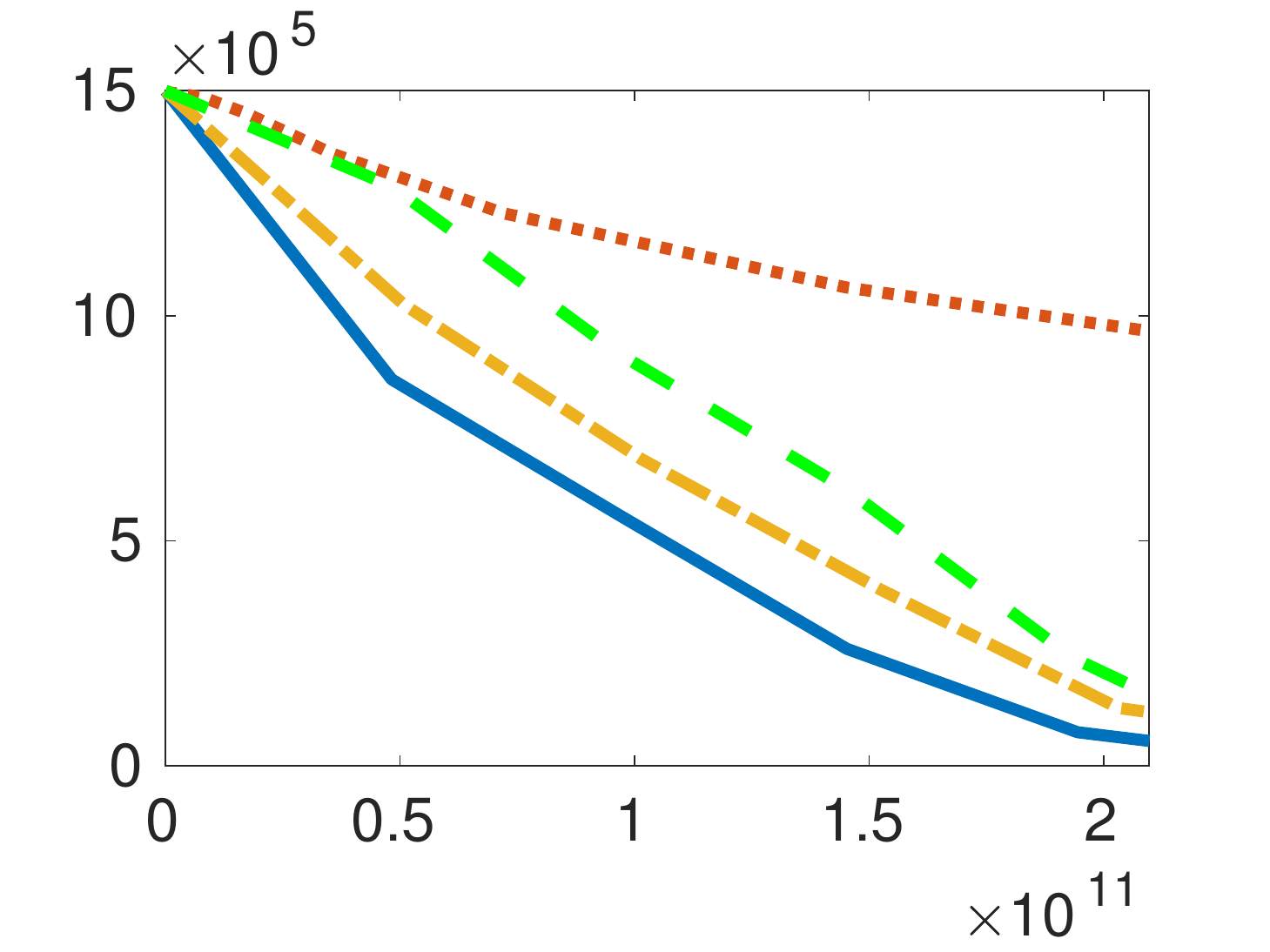}
		\end{minipage}
	}
	\hspace{-12pt}
	\subfigure[Mountain]{
		\begin{minipage}[t]{0.19\linewidth}
			\centering
			\includegraphics[width = 1\textwidth]{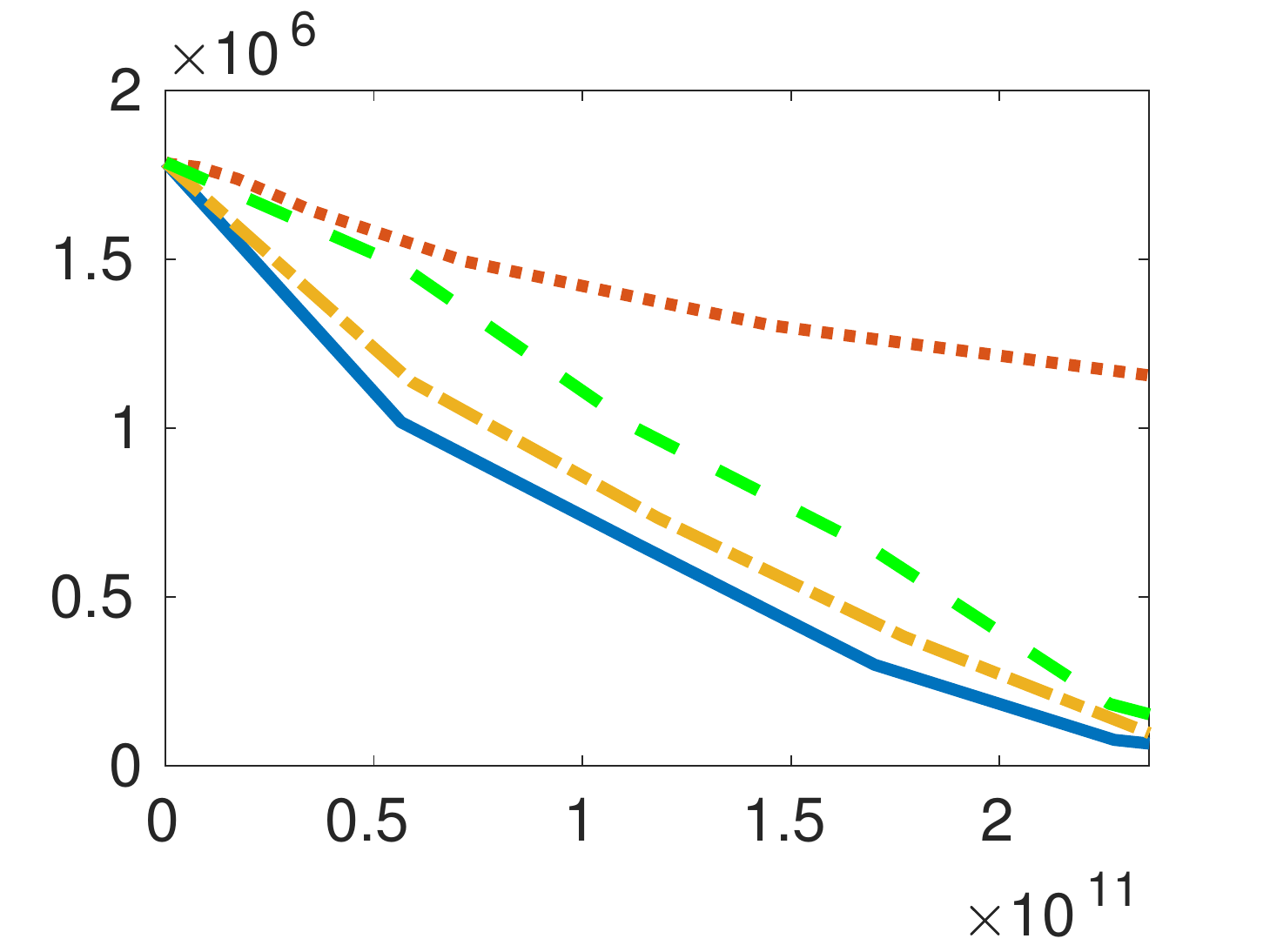}
		\end{minipage}
	}
	\hspace{-10pt}
\vspace{-5pt}
	\caption{Low-rank matrix completion problem. (a)-(e) are results for the \emph{first-order} case and (f)-(j) are results for the \emph{zeroth-order} case. The x-axis represents number of gradient query oracle for (a) - (e) and number of function query oracle for (f) - (j); the y-axis represents suboptimality, \textrm{i.e.}, $f(x)-\min_{y\in\mathcal{C}}f(y)$. The curves of the first-order case and the zeroth-order case look the same since the coordinate-wise gradient estimator equals the true gradient.}
	\label{mat}
\vspace{-8pt}
\end{figure}

\subsection{Sparsity-Constrained Logistic Regression}
In this experiment, we focus on the sparsity-constrained logistic regression:
\begin{equation}
	\min_{\|x\|_1\leq r}\frac{1}{n}\sum_{i=1}^{n} -\left(y_i\log\sigma(-x^Ta_i)+(1-y_i)\log\sigma(x^Ta_i)\right)
	\nonumber
\end{equation}
where $\sigma(z) = 1/\left(1+\exp(-z)\right)$ denotes the sigmoid function, $a_i \in\mathbb{R}^d$ denotes the data and $y_i\in\{0, 1\}$ denotes the corresponding label. We conduct the experiment on five LIBSVM \citep{chang2011libsvm} datasets: a9a ($n=32,561, d=123$), ijcnn1 ($n=49,990, d=22$), mushrooms ($n=8,124, d=112$), phishing ($n=11,055, d=68$) and w8a ($n=49,749, d=300$). We set the parameters according to Theorem \ref{theorem1}. For all the four algorithms, we use a mini batch of 256. The results are shown in Figure \ref{logi}, where (a)-(e) are results for the \emph{first-order} case and (f)-(j) are results for the \emph{zeroth-order} case. For some datasets, our ARCS is slower than SCGS at first but outperforms SCGS later. This corresponds to the gradient query complexity presented in Table \ref{table1}. For the first-order case, the gradient query complexity of ARCS has a dependence on $n$ and $\epsilon^{-1/2}$, while that of SCGS only has a dependence on $\epsilon^{-3}$. At the beginning, $\epsilon$ is relatively big, $\epsilon^{-3}$ is relatively small and $n$ is relatively big, thus the gradient query complexity of ARCS is higher than that of SCGS. When $\epsilon$ diminishes, the gradient query complexity of ARCS gradually becomes lower than that of SCGS.

\begin{figure}
	\vspace{-9pt}
	\hspace{-10pt}
	\begin{minipage}[t]{0.5\linewidth}
	    \centering
    	\includegraphics[width = \textwidth]{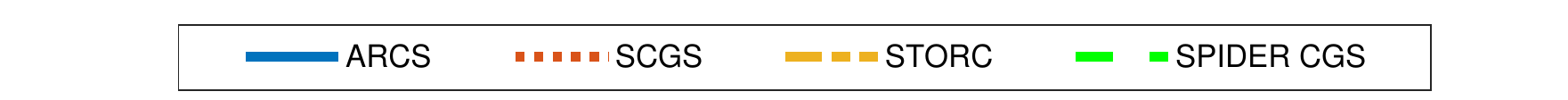}
	\end{minipage}
	\hspace{-12pt}
	\begin{minipage}[t]{0.5\linewidth}
	    \centering
    	\includegraphics[width = \textwidth]{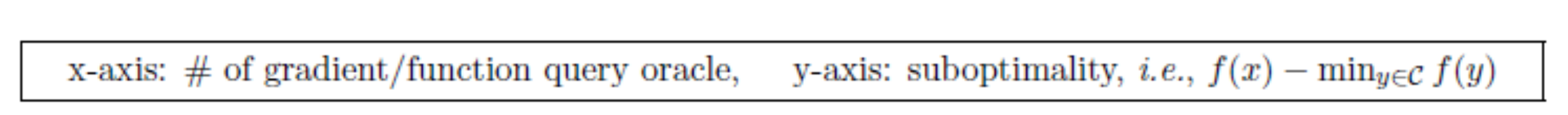}
	\end{minipage}
	\vspace{-14pt}
\end{figure}

\begin{figure}
	\centering
	\hspace{-10pt}
	\subfigure[A9a]{
		\begin{minipage}[t]{0.19\linewidth}
			\centering
			\includegraphics[width = 1\textwidth]{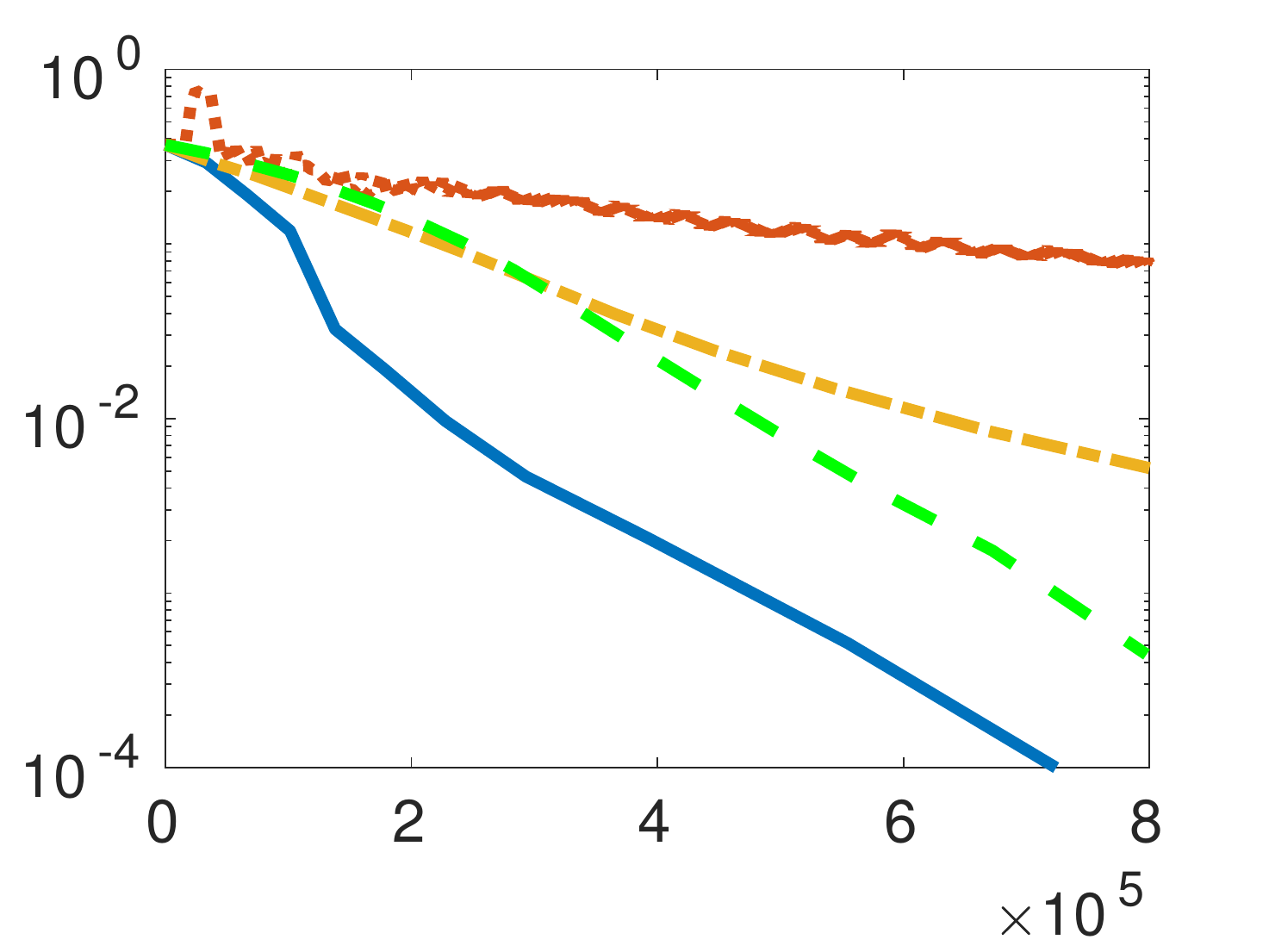}
		\end{minipage}
	}
	\hspace{-12pt}
	\subfigure[Ijcnn1]{
		\begin{minipage}[t]{0.19\linewidth}
			\centering
			\includegraphics[width = 1\textwidth]{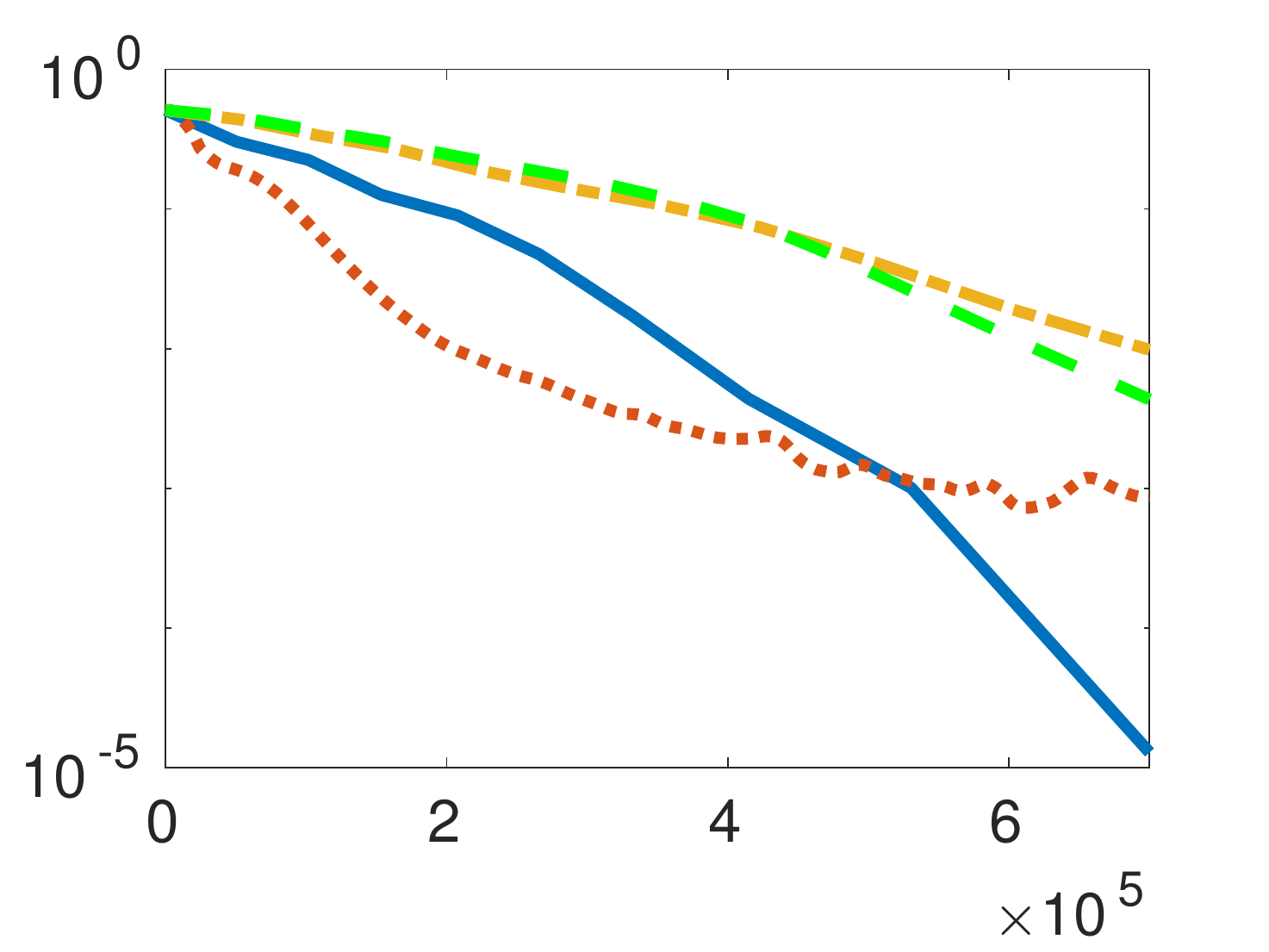}
		\end{minipage}
	}
	\hspace{-12pt}
	\subfigure[Mushrooms]{
		\begin{minipage}[t]{0.19\linewidth}
			\centering
			\includegraphics[width = 1\textwidth]{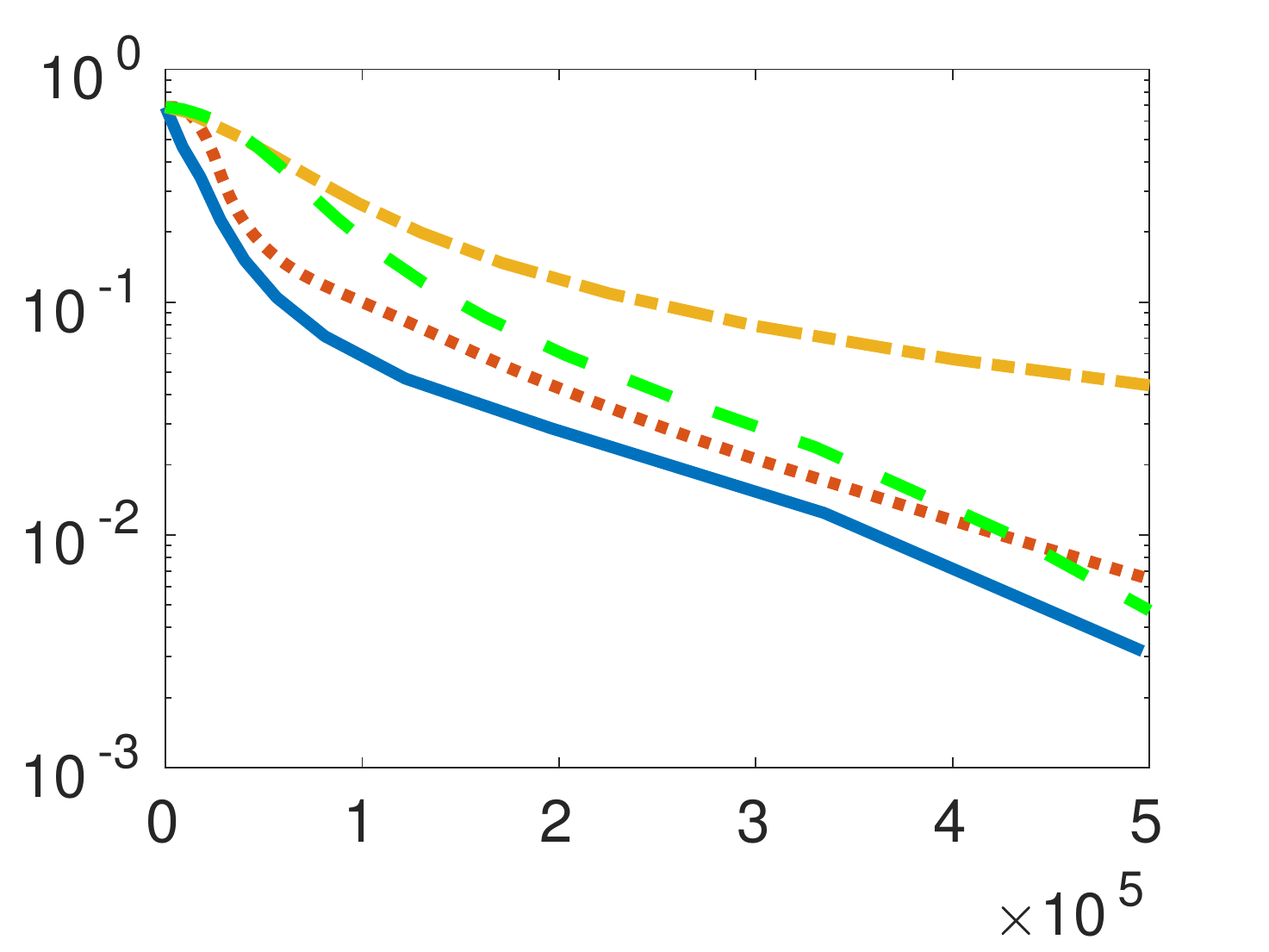}
		\end{minipage}
	}
	\hspace{-12pt}
	\subfigure[Phishing]{
		\begin{minipage}[t]{0.19\linewidth}
			\centering
			\includegraphics[width = 1\textwidth]{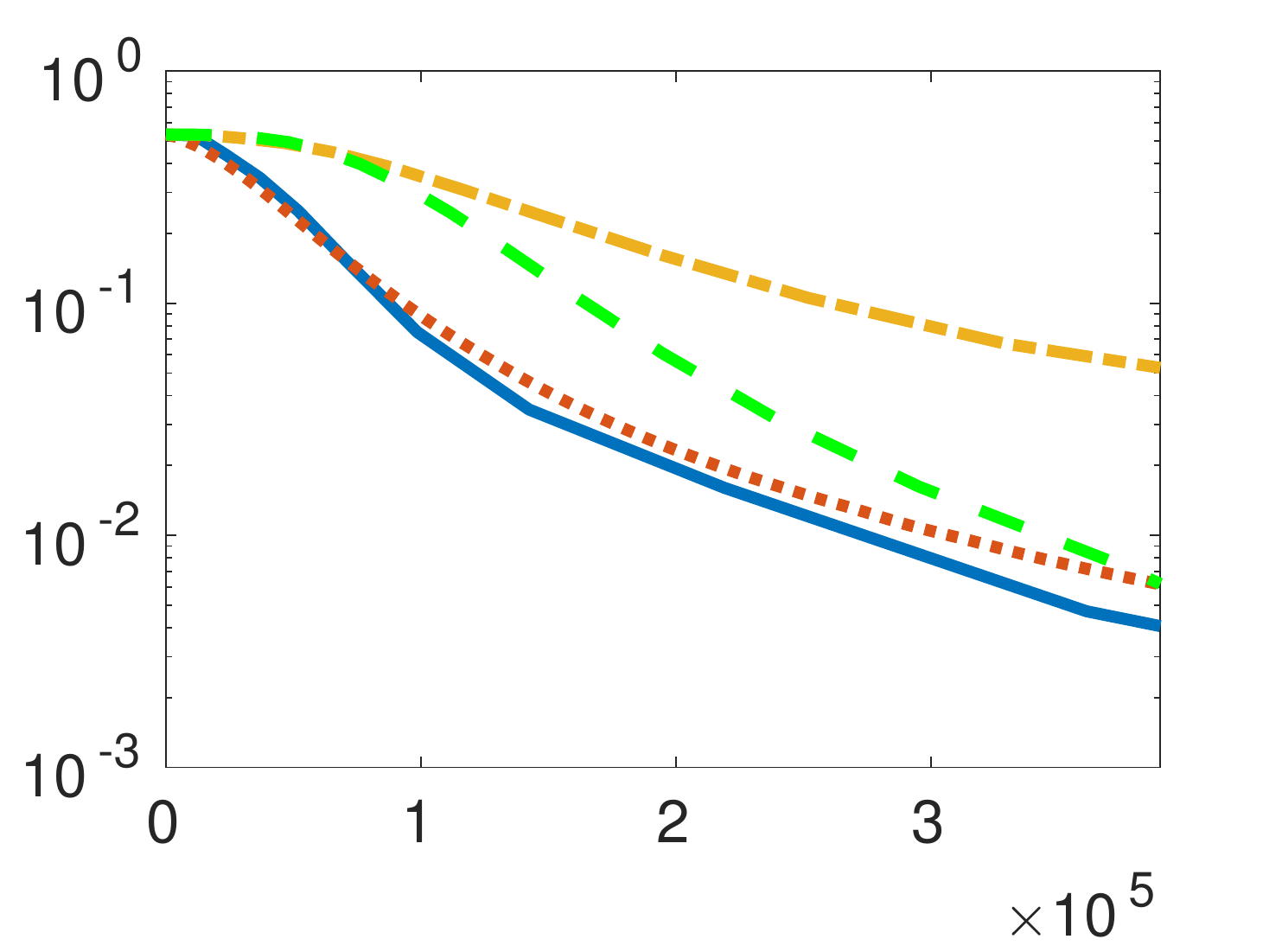}
		\end{minipage}
	}
	\hspace{-12pt}
	\subfigure[W8a]{
		\begin{minipage}[t]{0.19\linewidth}
			\centering
			\includegraphics[width = 1\textwidth]{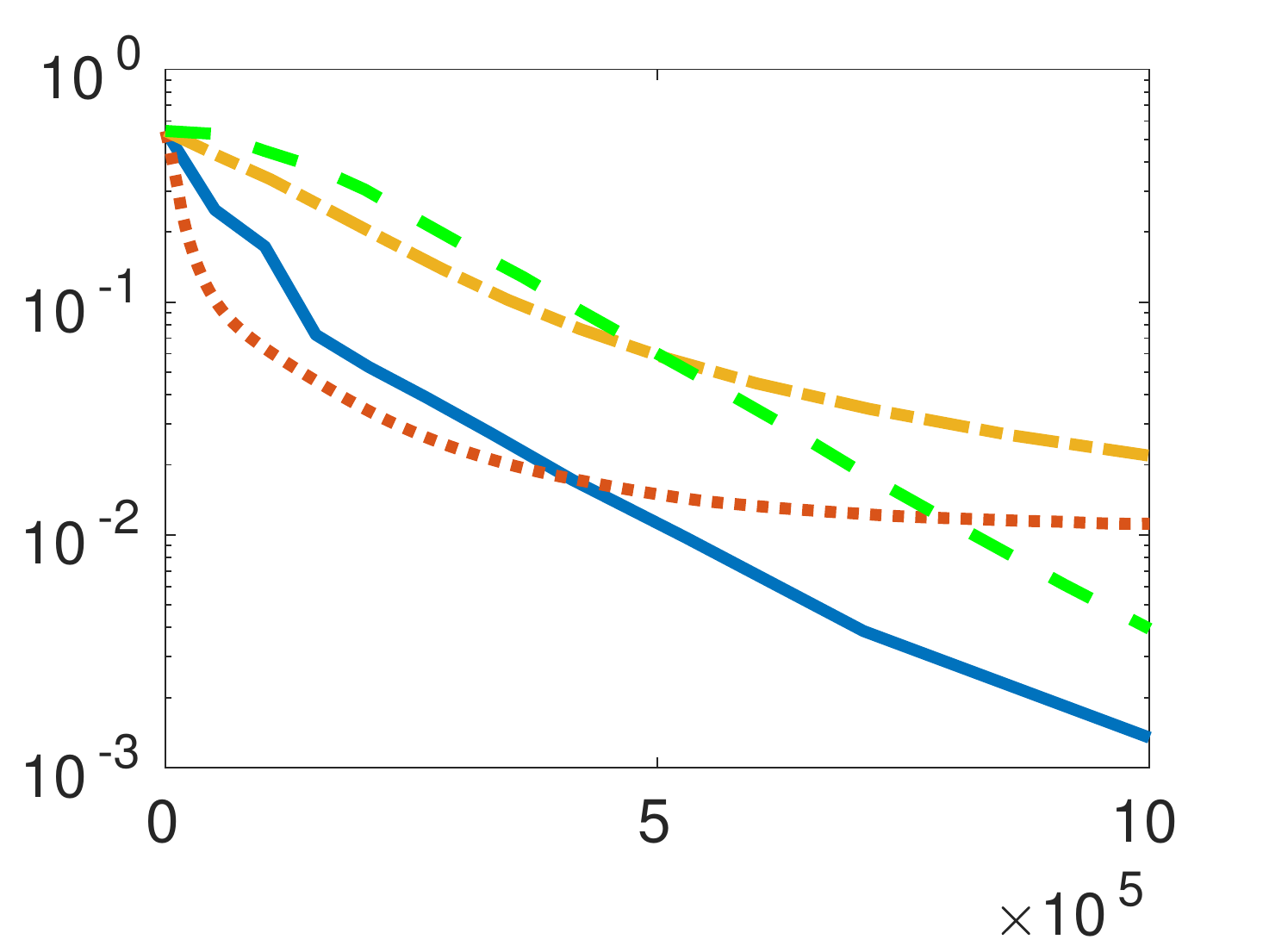}
		\end{minipage}
	}
	\hspace{-10pt}
\vspace{-6pt}
	
	\hspace{-10pt}
	\subfigure[A9a]{
		\begin{minipage}[t]{0.19\linewidth}
			\centering
			\includegraphics[width = 1\textwidth]{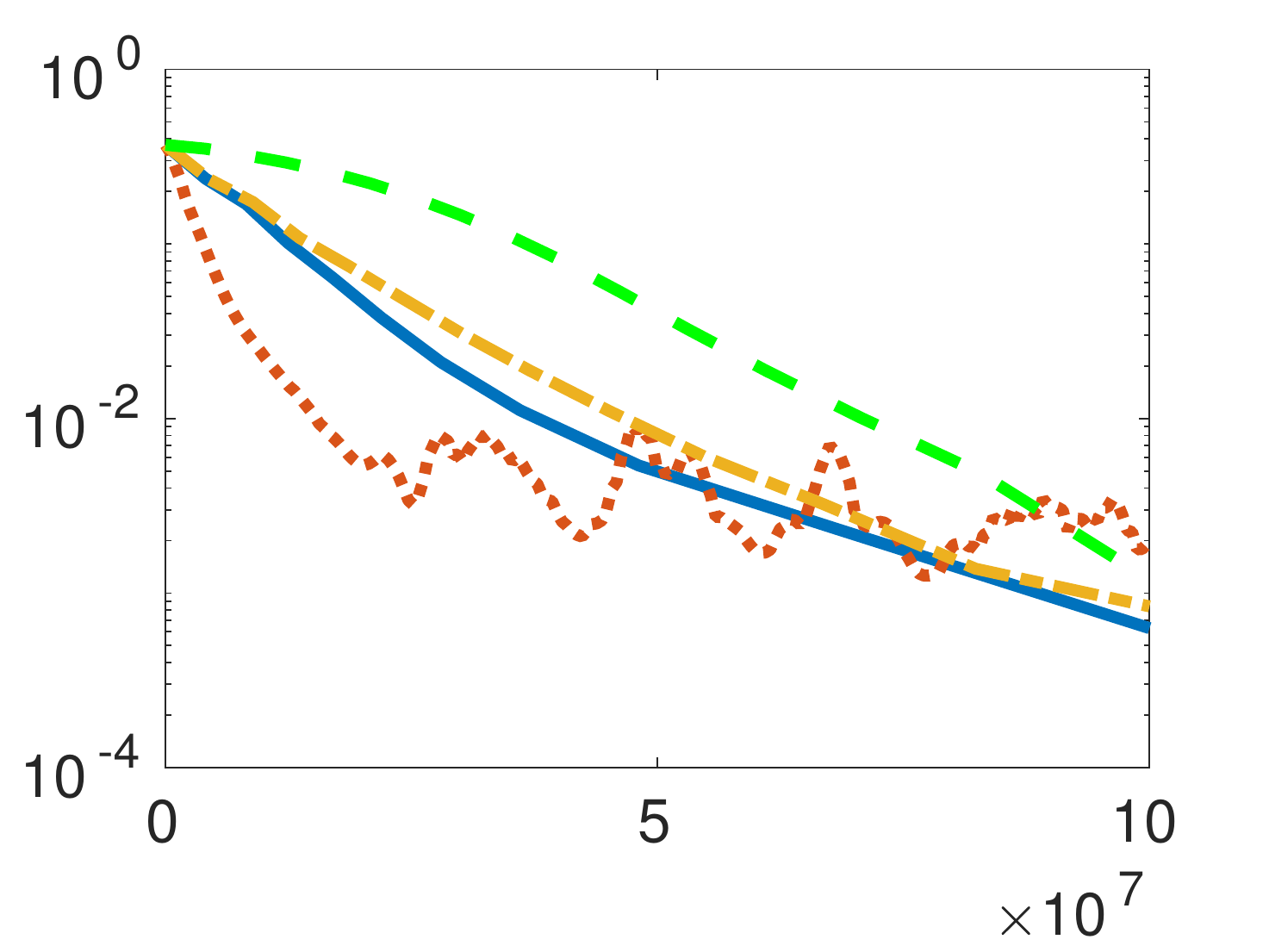}
		\end{minipage}
	}
	\hspace{-12pt}
	\subfigure[Ijcnn1]{
		\begin{minipage}[t]{0.19\linewidth}
			\centering
			\includegraphics[width = 1\textwidth]{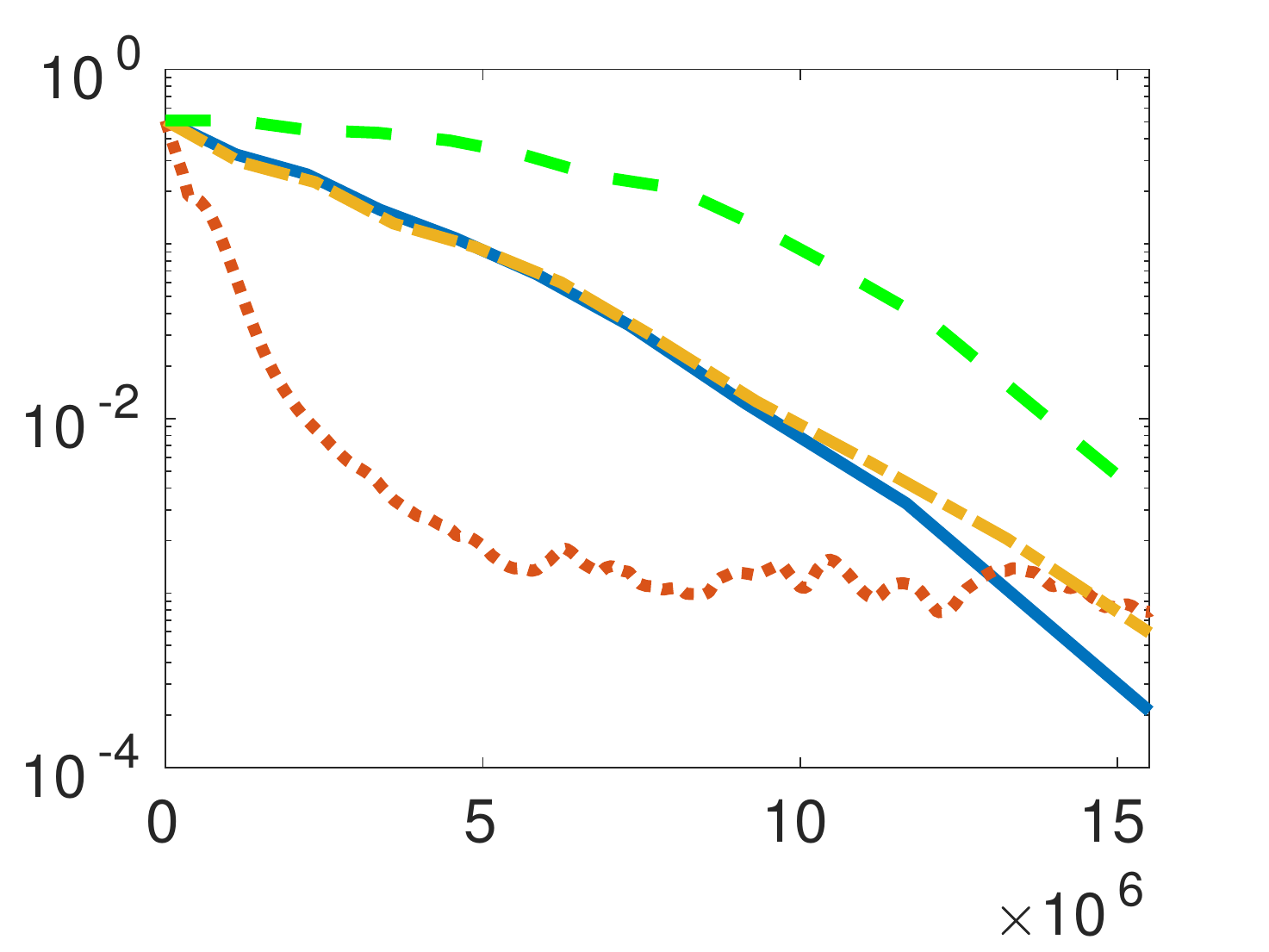}
		\end{minipage}
	}
	\hspace{-12pt}
	\subfigure[Mushrooms]{
		\begin{minipage}[t]{0.19\linewidth}
			\centering
			\includegraphics[width = 1\textwidth]{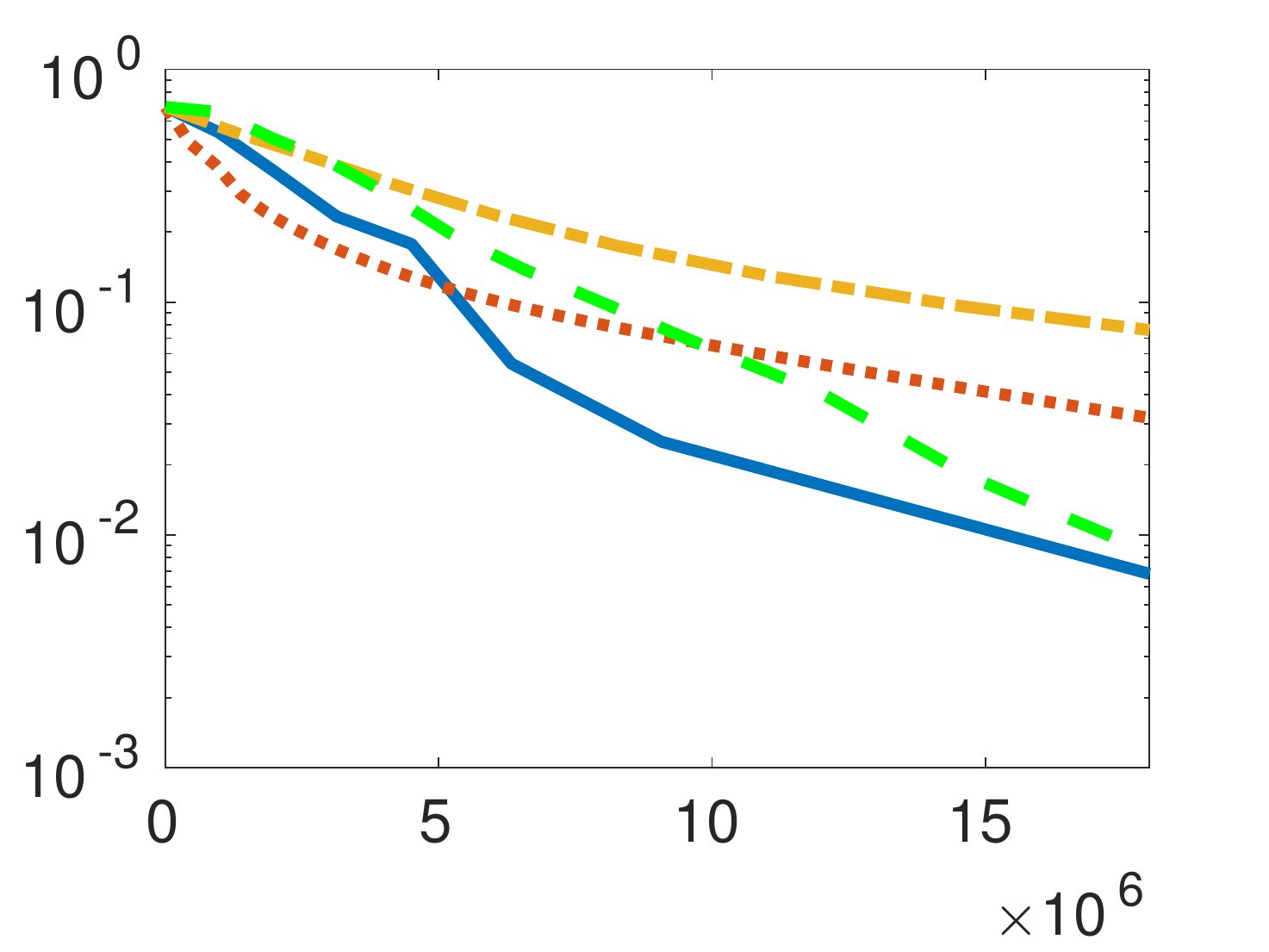}
		\end{minipage}
	}
	\hspace{-12pt}
	\subfigure[Phishing]{
		\begin{minipage}[t]{0.19\linewidth}
			\centering
			\includegraphics[width = 1\textwidth]{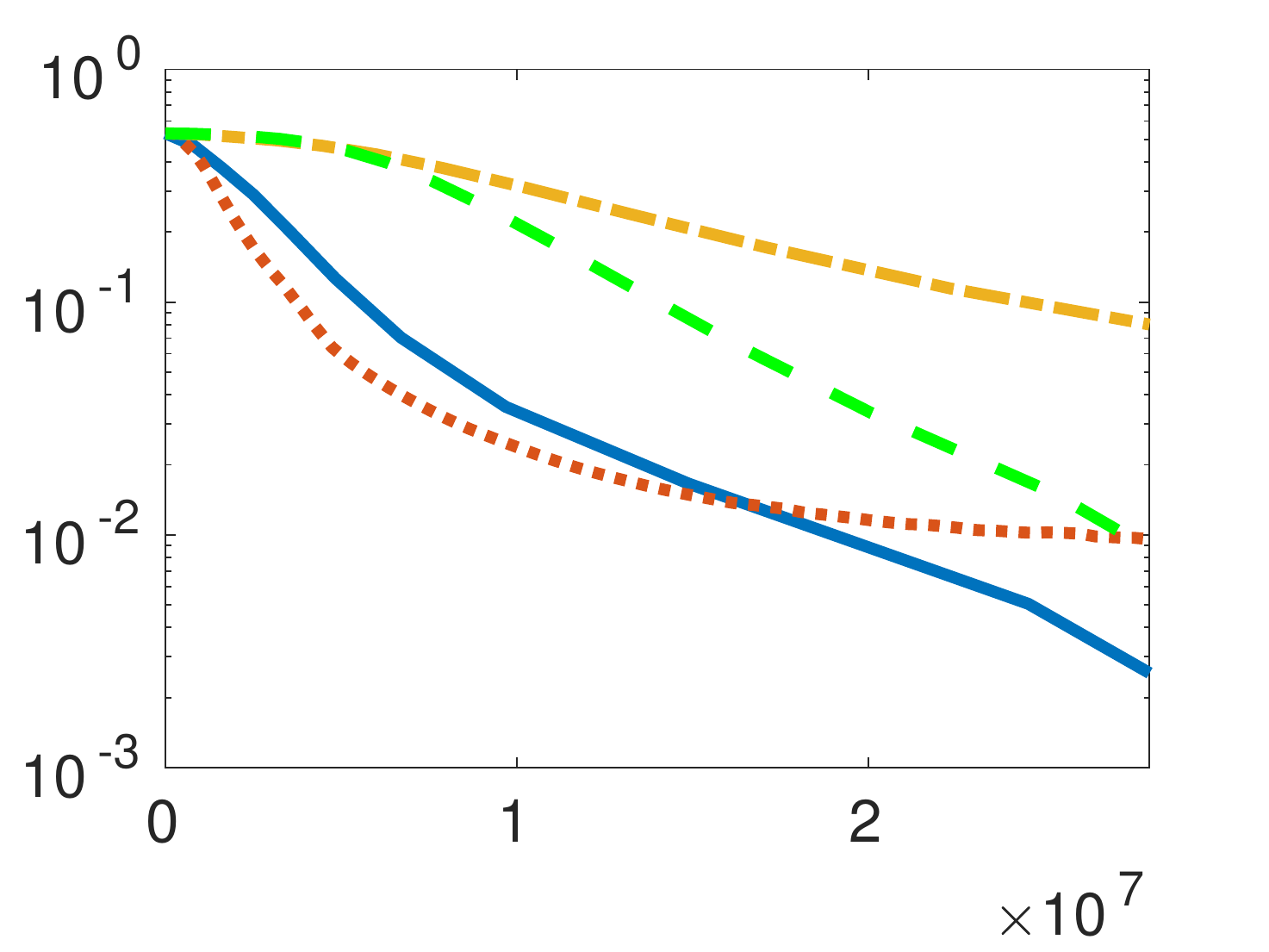}
		\end{minipage}
	}
	\hspace{-12pt}
	\subfigure[W8a]{
		\begin{minipage}[t]{0.19\linewidth}
			\centering
			\includegraphics[width = 1\textwidth]{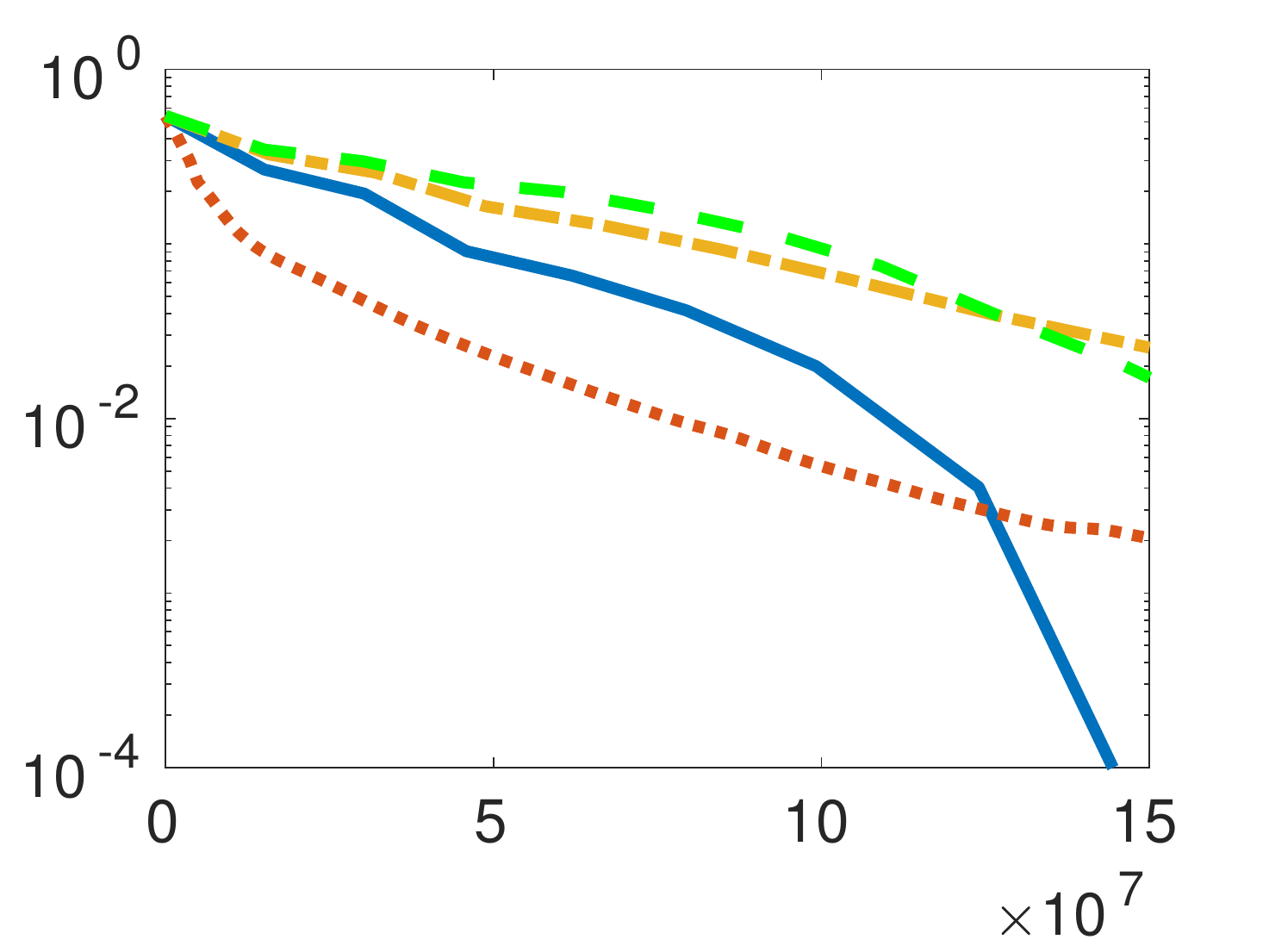}
		\end{minipage}
	}
	\hspace{-10pt}
	\vspace{-6pt}
\caption{Sparsity-constrained logistic regression. (a)-(e) are results for the \emph{first-order} case and (f)-(j) are results for the \emph{zeroth-order} case. The x-axis represents number of gradient query oracle for (a) - (e) and number of function query oracle for (f) - (j); the y-axis represents suboptimality, \textrm{i.e.}, $f(x)-\min_{y\in\mathcal{C}}f(y)$.}
	\label{logi}
\end{figure}

\section{Conclusion}
In this paper, we proposed an Accelerated variance-Reduced Conditional gradient Sliding (ARCS) algorithm for solving constrained finite-sum problems, which combines the variance-reduction technique and a novel momentum with conditional gradient sliding algorithm. Then We give the convergence results of our ARCS under convex and strongly-convex setting. Our ARCS can be used in either first-order (where gradient query oracle is available) or zeroth-order (where function query oracle is available) optimization. In first-order optimization, it outperforms all existing conditional gradient type algorithms with respect to gradient query complexity. In zeroth-order optimization, it is the first conditional gradient sliding type algorithm for convex problems. Finally we conduct numerical experiments with real-world datasets to show the superiority of our ARCS.

\vskip 0.2in
\bibliography{main}

\appendix

% \section*{Appendix Roadmap}

% Appendix \ref{app_a} gives Lemma \ref{lemma1} and \ref{lemma2} that are fundamental in the analysis of our ARCS.

% Appendix \ref{app_b} gives Lemma \ref{lemma8} for convex problems and then gives the proof of Theorem \ref{theorem1}.

% Appendix \ref{app_c} gives Lemma \ref{lemma9} for strongly-convex problems and then gives the proof of Theorem \ref{theorem2}. Specifically, Theorem \ref{theorem2} is the direct result of Theorem \ref{theorem2a}, \ref{theorem2b} and \ref{theorem2c}.

% Appendix \ref{app_d} gives Lemma \ref{lemma7} - \ref{lemma4} concerning the variance-reduced gradient blending $G_t$ in Algorithm \ref{algo1} and the coordinate-wise gradient estimator $\nabla_{coord} f(x)$.

% Appendix \ref{app_e} includes STORC \citep{hazan2016variance} and its key theorems for completeness.

\section{Fundamental Lemmas}	\label{app_a}
For simplicity, we denote
	\begin{equation}
	\begin{aligned}
	&l_{f}(z,x) := f(z)+\langle\nabla f(z), x-z\rangle\\
	&\delta_t:= G_t-\nabla f(\underline{x}_t)\\
	&x_{t-1}^+:= \frac{1}{1+\tau\gamma_s}(x_{t-1}+\tau\gamma_s\underline{x}_t)\\
	\end{aligned}
	\end{equation}
	
With the above notations, we have
	\begin{equation}
	\begin{aligned}
	\bar{x}_t-\underline{x}_t =& (1-\alpha_s-p_s)\bar{x}_{t-1}+\alpha_sx_t+p_s\tilde{x}-\underline{x}_t\\
	\overset{\textrm{\ding{172}}}{=}& \alpha_sx_t+\frac{1}{1+\tau\gamma_s}\left[\left(1+\tau\gamma_s(1-\alpha_s)\right)\underline{x}_t-\alpha_sx_{t-1}\right]-\underline{x}_t = \alpha_s(x_t-x_{t-1}^+)
	\end{aligned}
	\end{equation}
	where \ding{172} comes from the definition of $\underline{x}_t$ in Algorithm \ref{algo1}.
	
	\begin{lemma}	\label{lemma1}
		For any $x\in\mathcal{C}$, we have
		\begin{equation}
		\begin{aligned}
		&\gamma_s\left[l_f(\underline{x}_t, x_t)-l_f(\underline{x}_t, x)\right] \\
		\leq& \frac{\gamma_s\tau}{2}\|x-x_t\|^2+\frac{1}{2}\|x-x_{t-1}\|^2-\frac{1+\gamma_s\tau}{2}\|x-x_t\|^2-\frac{1+\gamma_s\tau}{2}\|x_t-x_{t-1}^+\|^2-\gamma_s\langle\delta_t, x_t-x\rangle+\eta_{s,t}
		\end{aligned}
		\end{equation}
	\end{lemma}
	\begin{proof}
		From Algorithm \ref{algo1}, we have $\langle\nabla h(x_t), x_t-x\rangle \leq \eta_{s,t}$ for any $x\in\mathcal{C}$. Observe that $h$ is $(1+\tau\gamma_s)$-strongly convex, then we have
		\begin{equation}	\label{l1eq1}
			h(x_t)-h(x)+\frac{1+\tau\gamma_s}{2}\|x-x_t\|^2\leq\langle\nabla h(x_t), x_t-x\rangle\leq\eta_{s,t}, \; \forall\, x\in C
		\end{equation}
		which is
		\begin{equation}	\label{l1eq2}
		\begin{aligned}
		&\gamma_s\left[\langle G_t, x_t\rangle+\frac{\tau}{2}\|x_t-\underline{x}_t\|^2\right]+\frac{1}{2}\|x_t-x_{t-1}\|^2 \\
		\leq& \gamma_s\left[\langle G_t, x\rangle+\frac{\tau}{2}\|x-\underline{x}_t\|^2\right]+\frac{1}{2}\|x-x_{t-1}\|^2-\frac{1+\tau\gamma_s}{2}\|x-x_t\|^2+\eta_{s,t}, \; \forall\, x\in C
		\end{aligned}
		\end{equation}
		Rearranging the terms, we get
		\begin{equation}	\label{l1eq3}
		\begin{aligned}
		&\gamma_s\left[\langle G_t, x_t-x\rangle+\frac{\tau}{2}\|x_t-\underline{x}_t\|^2\right]+\frac{1}{2}\|x_t-x_{t-1}\|^2 \\
		\leq& \frac{\gamma_s\tau}{2}\|x-x_t\|^2+\frac{1}{2}\|x-x_{t-1}\|^2-\frac{1+\tau\gamma_s}{2}\|x-x_t\|^2+\eta_{s,t}
		\end{aligned}
		\end{equation}
		With the notation of $l_f(\cdot, \cdot)$, we have
		\begin{equation}	\label{l1eq4}
		\begin{aligned}
		\langle G_t, x_t-x\rangle =& \langle \nabla f(\underline{x}_t), x_t-x\rangle+\langle\delta_t, x_t-x\rangle = l_f(\underline{x}_t, x_t)-l_f(\underline{x}_t, x)+\langle\delta_t, x_t-x\rangle\\
		\end{aligned}
		\end{equation}
		Also we have
		\begin{equation}	\label{l1eq5}
		\begin{aligned}
		&\frac{\gamma_s\tau}{2}\|x_t-\underline{x}_t\|^2+\frac{1}{2}\|x_t-x_{t-1}\|^2\\
		=&\frac{\gamma_s\tau}{2}\|x_t\|^2-\langle x_t, \gamma_s\tau\underline{x}_t\rangle +\frac{\gamma_s\tau}{2}\|\underline{x}_t\|^2+\frac{1}{2}\|x_t\|^2-\langle x_t, x_{t-1}\rangle +\frac{1}{2}\|x_{t-1}\|^2\\
		=&\frac{1+\gamma_s\tau}{2}\|x_t\|^2-(1+\gamma_s\tau)\langle x_t, \frac{1}{1+\gamma_s\tau}(\gamma_s\tau\underline{x}_t+x_{t-1})\rangle+\underbrace{\frac{\gamma_s\tau}{2}\|\underline{x}_t\|^2+\frac{1}{2}\|x_{t-1}\|^2}_{Q_1}\\
		\end{aligned}
		\end{equation}
		For the term $Q_1$, we have
		\begin{equation}	\label{l1eq6}
		\begin{aligned}
		2(1+\gamma_s\tau)Q_1=&(1+\gamma_s\tau)\left(\gamma_s\tau\|\underline{x}_t\|^2+\|x_{t-1}\|^2\right)\\
		=&\gamma_s\tau\|\underline{x}_t\|^2+\gamma_s\tau\|x_{t-1}\|^2+\left(\gamma_s^2\tau^2\|\underline{x}_t\|^2+\|x_{t-1}\|^2\right)\\
		=& \gamma_s\tau\|\underline{x}_t\|^2+\gamma_s\tau\|x_{t-1}\|^2-2\gamma_s\tau\langle \underline{x}_t, x_{t-1}\rangle+\|\gamma_s\tau\underline{x}_t+x_{t-1}\|^2
		\end{aligned}
		\end{equation}
		With the notation of $x_{t-1}^+$, we have
		\begin{equation}	\label{l1eq7}
		2(1+\gamma_s\tau)Q_1 \geq \|\gamma_s\tau\underline{x}_t+x_{t-1}\|^2 = (1+\gamma_s\tau)^2\|x_{t-1}^+\|^2
		\end{equation}
		which is
		\begin{equation}	\label{l1eq8}
		Q_1\geq \frac{1+\gamma_s\tau}{2}\|x_{t-1}^+\|^2
		\end{equation}
		Plugging (\ref{l1eq8}) into (\ref{l1eq5}), we get
		\begin{equation}	\label{l1eq9}
		\begin{aligned}
		&\frac{\gamma_s\tau}{2}\|x_t-\underline{x}_t\|^2+\frac{1}{2}\|x_t-x_{t-1}\|^2\\
		\geq& \frac{1+\gamma_s\tau}{2}\|x_t\|^2-(1+\gamma_s\tau)\langle x_t, \frac{1}{1+\gamma_s\tau}(\gamma_s\tau\underline{x}_t+x_{t-1})\rangle+\frac{1+\gamma_s\tau}{2}\|x_{t-1}^+\|^2 = \frac{1+\gamma_s\tau}{2}\|x_t-x_{t-1}^+\|^2
		\end{aligned}
		\end{equation}
		Plugging (\ref{l1eq4}) and (\ref{l1eq9}) into (\ref{l1eq3}) and rearranging the terms, we get
		\begin{equation}
		\begin{aligned}
		&\gamma_s\left[l_f(\underline{x}_t, x_t)-l_f(\underline{x}_t, x)\right] \\
		\leq& \frac{\gamma_s\tau}{2}\|x-x_t\|^2+\frac{1}{2}\|x-x_{t-1}\|^2-\frac{1+\tau\gamma_s}{2}\|x-x_t\|^2+\eta_{s,t}-\frac{1+\gamma_s\tau}{2}\|x_t-x_{t-1}^+\|^2-\gamma_s\langle\delta_t, x_t-x\rangle
		\end{aligned}
		\end{equation}
		Then we complete the proof.
	\end{proof}

	\begin{lemma}	\label{lemma2}
		Suppose $f$ is $\tau$-strongly ($\tau \geq 0$) convex and each $f_{i\in [n]}$ is $L$-smooth. Conditioning on $x_1, ..., x_{t-1}$
		
		$\bullet$ For the first-order case, assume that $\alpha_s\in[0, 1], p_s\in[0, 1]$ and $\gamma_s>0$ satisfy
		\[1+\tau\gamma_s-L\alpha_s\gamma_s >0, \; 1-\alpha_s-p_s\geq0, \qquad p_s-\frac{L\alpha_s\gamma_s}{1+\tau\gamma_s-L\alpha_s\gamma_s}>0\]
		Then we have
		\[\begin{aligned}
		&\frac{\gamma_s}{\alpha_s}\mathbb{E}\left[f(\bar{x}_t)-f(x)\right]+\frac{1+\tau\gamma_s}{2}\mathbb{E}\left[\|x-x_t\|^2\right]\\
		\leq&\frac{\gamma_s(1-\alpha_s-p_s)}{\alpha_s}\left[f(\bar{x}_{t-1})-f(x)\right]+\frac{\gamma_sp_s}{\alpha_s}\left[f(\tilde{x})-f(x)\right]+\frac{1}{2}\|x-x_{t-1}\|^2+\eta_{s,t}
		\end{aligned}\]
		
		$\bullet$ For the zeroth-order case, assume that $\alpha_s\in[0, 1], p_s\in[0, 1]$ and $\gamma_s>0$ satisfy
		\[1+\tau\gamma_s-L\alpha_s\gamma_s >0, \; 1-\alpha_s-p_s\geq0, \qquad p_s-\frac{4\alpha_s\gamma_sL}{1+\tau\gamma_s-L\alpha_s\gamma_s}>0\]
		Then we have
		\[\begin{aligned}
		&\frac{\gamma_s}{\alpha_s}\mathbb{E}\left[f(\bar{x}_t)-f(x)\right]+\frac{1+\tau\gamma_s}{2}\mathbb{E}\left[\|x-x_t\|^2\right]\\
		\leq&\frac{\gamma_s(1-\alpha_s-p_s)}{\alpha_s}\left[f(\bar{x}_{t-1})-f(x)\right]+\frac{\gamma_sp_s}{\alpha_s}\left[f(\tilde{x})-f(x)\right]+\frac{1}{2}\|x-x_{t-1}\|^2+\eta_{s,t}\\
		&-\gamma_s\langle \hat{\nabla}_{coord}f(\underline{x}_t)-\nabla f(\underline{x}_t), x_t-x\rangle+\frac{6\gamma_s^2\mu^2L^2d}{1+\tau\gamma_s-L\alpha_s\gamma_s}\\
		\end{aligned}\]
	\end{lemma}
	\begin{proof}
		Since $f$ is $\tau$-strongly convex and $L$-smooth, then we have
		\begin{equation}	\label{l2eq7}
		\begin{aligned}
		f(\bar{x}_t)\leq&l_{f}(\underline{x}_t,\bar{x}_t)+\frac{L}{2}\|\bar{x}_t-\underline{x}_t\|^2\\
		\overset{\textrm{\ding{172}}}{=}&(1-\alpha_s-p_s)l_{f}(\underline{x}_t,\bar{x}_{t-1})+\alpha_sl_{f}(\underline{x}_t,x_t)+p_sl_{f}(\underline{x}_t,\tilde{x})+\frac{L\alpha_s^2}{2}\|x_t-x_{t-1}^+\|^2
		\end{aligned}
		\end{equation}
		where \ding{172} comes from the definition of $\bar{x}_t$ in Algorithm \ref{algo1}. Plugging Lemma \ref{lemma1} into (\ref{l2eq7}), we get
		\begin{equation}	\label{l2eq1}
		\begin{aligned}
		f(\bar{x}_t)\leq&(1-\alpha_s-p_s)l_{f}(\underline{x}_t,\bar{x}_{t-1})\\
		&+\alpha_s\left[l_{f}(\underline{x}_t,x)+\frac{\tau}{2}\|x-\underline{x}_t\|^2+\frac{1}{2\gamma_s}\|x-x_{t-1}\|^2-\frac{1+\tau\gamma_s}{2\gamma_s}\|x-x_t\|^2+\frac{\eta_{s,t}}{\gamma_s}\right]\\
		&+p_sl_{f}(\underline{x}_t,\tilde{x})-\frac{\alpha_s}{2\gamma_s}\left(1+\tau\gamma_s-L\alpha_s\gamma_s\right)\|x_t-x_{t-1}^+\|^2-\alpha_s\langle\delta_t, x_t-x\rangle\\
		\overset{\textrm{\ding{172}}}{\leq}& (1-\alpha_s-p_s)f(\bar{x}_{t-1})+\alpha_s\left[f(x)+\frac{1}{2\gamma_s}\|x-x_{t-1}\|^2-\frac{1+\tau\gamma_s}{2\gamma_s}\|x-x_t\|^2+\frac{\eta_{s,t}}{\gamma_s}\right]\\
		&+p_sl_{f}(\underline{x}_t,\tilde{x})-\frac{\alpha_s}{2\gamma_s}\left(1+\tau\gamma_s-L\alpha_s\gamma_s\right)\|x_t-x_{t-1}^+\|^2-\alpha_s\langle\delta_t, x_t-x_{t-1}^+\rangle-\alpha_s\langle\delta_t, x_{t-1}^+-x\rangle\\
		\end{aligned}
		\end{equation}
		where \ding{172} comes from the fact that $f$ is $\tau$-strongly convex. Now we give proof to the first-order and zeroth-order case respectively.
		
		\textbf{First-order Case}: Using the fact that $b\langle u, v\rangle-a\|v\|^2/2\leq b^2\|u\|^2/(2a)$, we have from (\ref{l2eq1})
		\begin{equation}	\label{l2eq4}
		\begin{aligned}
		f(\bar{x}_t)\leq& (1-\alpha_s-p_s)f(\bar{x}_{t-1})+\alpha_s\left[f(x)+\frac{1}{2\gamma_s}\|x-x_{t-1}\|^2-\frac{1+\tau\gamma_s}{2\gamma_s}\|x-x_t\|^2+\frac{\eta_{s,t}}{\gamma_s}\right]\\
		&+\underbrace{p_sl_{f}(\underline{x}_t,\tilde{x})+\frac{\alpha_s\gamma_s\|\delta_t\|^2}{2(1+\tau\gamma_s-L\alpha_s\gamma_s)}-\alpha_s\langle\delta_t, x_{t-1}^+-x\rangle}_{Q_2}\\
		\end{aligned}
		\end{equation}
		Using Lemma \ref{lemma4}, we can bound $Q_2$ as
		\begin{equation}	\label{l2eq2}
		\begin{aligned}
		&\mathbb{E}\left[p_sl_{f}(\underline{x}_t,\tilde{x})+\frac{\alpha_s\gamma_s\|\delta_t\|^2}{2(1+\tau\gamma_s-L\alpha_s\gamma_s)}-\alpha_s\langle\delta_t, x_{t-1}^+-x\rangle\right] = \mathbb{E}\left[p_sl_{f}(\underline{x}_t,\tilde{x})+\frac{\alpha_s\gamma_s\|\delta_t\|^2}{2(1+\tau\gamma_s-L\alpha_s\gamma_s)}\right]\\
		\overset{\textrm{\ding{172}}}{\leq}&p_sl_{f}(\underline{x}_t,\tilde{x})+\frac{L\alpha_s\gamma_s}{1+\tau\gamma_s-L\alpha_s\gamma_s}\left(f(\tilde{x})-l_f(\underline{x}_t, \tilde{x})\right)\\
		=&\left(p_s-\frac{L\alpha_s\gamma_s}{2(1+\tau\gamma_s-L\alpha_s\gamma_s)}\right)l_{f}(\underline{x}_t,\tilde{x})+\frac{L\alpha_s\gamma_s}{1+\tau\gamma_s-L\alpha_s\gamma_s}f(\tilde{x}) \overset{\textrm{\ding{173}}}{\leq} p_sf(\tilde{x})
		\end{aligned}
		\end{equation}
		where \ding{172} comes from Lemma \ref{lemma4} and  \ding{173} comes from the assumption that $p_s-\frac{L\alpha_s\gamma_s}{1+\tau\gamma_s-L\alpha_s\gamma_s}>p_s-\frac{2L\alpha_s\gamma_s}{1+\tau\gamma_s-L\alpha_s\gamma_s}>0$ and the convexity of $f$. Plugging (\ref{l2eq2}) into (\ref{l2eq4}), we get
		\begin{equation}
		\begin{aligned}
		&\mathbb{E}\left[f(\bar{x}_t)+\frac{\alpha_s(1+\tau\gamma_s)}{2\gamma_s}\|x-x_t\|^2\right]\\
		\leq&(1-\alpha_s-p_s)f(\bar{x}_{t-1})+\alpha_sf(x)+p_sf(\tilde{x})+\frac{\alpha_s}{2\gamma_s}\|x-x_{t-1}\|^2+\frac{\alpha_s}{\gamma_s}\eta_{s,t}
		\end{aligned}
		\end{equation}
		Subtracting both sides with $f(x)$ and then multiplying both sides with $\frac{\gamma_s}{\alpha_s}$, we get
		\begin{equation}
		\begin{aligned}
		&\frac{\gamma_s}{\alpha_s}\mathbb{E}\left[f(\bar{x}_t)-f(x)\right]+\frac{1+\tau\gamma_s}{2}\mathbb{E}\left[\|x-x_t\|^2\right]\\
		\leq&\frac{\gamma_s(1-\alpha_s-p_s)}{\alpha_s}\left[f(\bar{x}_{t-1})-f(x)\right]+\frac{\gamma_sp_s}{\alpha_s}\left[f(\tilde{x})-f(x)\right]+\frac{1}{2}\|x-x_{t-1}\|^2+\eta_{s,t}
		\end{aligned}
		\end{equation}
		Then we get the desired result for the first-order case. Next we give proof to the zeroth-order case.
		
		\textbf{Zeroth-order Case}: We can rewrite (\ref{l2eq1}) as
		\begin{equation}	\label{l2eq6}
		\begin{aligned}
		f(\bar{x}_t)\leq& (1-\alpha_s-p_s)f(\bar{x}_{t-1})+\alpha_s\left[f(x)+\frac{1}{2\gamma_s}\|x-x_{t-1}\|^2-\frac{1+\tau\gamma_s}{2\gamma_s}\|x-x_t\|^2+\frac{\eta_{s,t}}{\gamma_s}\right]\\
		&+p_sl_{f}(\underline{x}_t,\tilde{x})-\frac{\alpha_s}{2\gamma_s}\left(1+\tau\gamma_s-L\alpha_s\gamma_s\right)\|x_t-x_{t-1}^+\|^2-\alpha_s\langle\delta_t-\mathbb{E}\left[\delta_t\right], x_t-x_{t-1}^+\rangle\\
		&-\alpha_s\langle\mathbb{E}\left[\delta_t\right], x_t-x_{t-1}^+\rangle-\alpha_s\langle\delta_t, x_{t-1}^+-x\rangle\\
		\overset{\textrm{\ding{172}}}{\leq}& (1-\alpha_s-p_s)f(\bar{x}_{t-1})+\alpha_s\left[f(x)+\frac{1}{2\gamma_s}\|x-x_{t-1}\|^2-\frac{1+\tau\gamma_s}{2\gamma_s}\|x-x_t\|^2+\frac{\eta_{s,t}}{\gamma_s}\right]\\
		&+p_sl_{f}(\underline{x}_t,\tilde{x})+\frac{\alpha_s\gamma_s\|\delta_t-\mathbb{E}\left[\delta_t\right]\|^2}{2(1+\tau\gamma_s-L\alpha_s\gamma_s)}-\alpha_s\langle\mathbb{E}\left[\delta_t\right], x_t-x_{t-1}^+\rangle-\alpha_s\langle\delta_t, x_{t-1}^+-x\rangle\\
		\end{aligned}
		\end{equation}
		where \ding{172} comes from the fact that $b\langle u, v\rangle-a\|v\|^2/2\leq b^2\|u\|^2/(2a)$. Using Lemma \ref{lemma4}, we have
		\begin{equation}	\label{l2eq3}
		\begin{aligned}
		&\mathbb{E}\left[p_sl_{f}(\underline{x}_t,\tilde{x})+\frac{\alpha_s\gamma_s\|\delta_t-\mathbb{E}\left[\delta_t\right]\|^2}{2(1+\tau\gamma_s-L\alpha_s\gamma_s)}\right]\\
		\overset{\textrm{\ding{172}}}{\leq}& p_sl_f(\underline{x}_t, \tilde{x})+\frac{4\alpha_s\gamma_sL}{1+\tau\gamma_s-L\alpha_s\gamma_s}\left(f(\tilde{x})-l_f(\underline{x}_t, \tilde{x})\right)+\frac{6\alpha_s\gamma_s\mu^2L^2d}{1+\tau\gamma_s-L\alpha_s\gamma_s}\\
		=&\left(p_s-\frac{4\alpha_s\gamma_sL}{1+\tau\gamma_s-L\alpha_s\gamma_s}\right)l_f(\underline{x}_t, \tilde{x})+\frac{4\alpha_s\gamma_sL}{1+\tau\gamma_s-L\alpha_s\gamma_s}f(\tilde{x})+\frac{6\alpha_s\gamma_s\mu^2L^2d}{1+\tau\gamma_s-L\alpha_s\gamma_s}\\
		\overset{\textrm{\ding{173}}}{\leq}& p_sf(\tilde{x})+\frac{6\alpha_s\gamma_s\mu^2L^2d}{1+\tau\gamma_s-L\alpha_s\gamma_s}\\
		\end{aligned}
		\end{equation}
		where \ding{172} comes from Lemma \ref{lemma4} and \ding{173} comes from the assumption that $p_s-\frac{4\alpha_s\gamma_sdL}{1+\tau\gamma_s-L\alpha_s\gamma_s}>0$ and the convexity of $f$. Also we have
		\begin{equation}	\label{l2eq5}
		\begin{aligned}
		&\mathbb{E}\left[-\alpha_s\langle\mathbb{E}\left[\delta_t\right], x_t-x_{t-1}^+\rangle-\alpha_s\langle\delta_t, x_{t-1}^+-x\rangle\right] = -\alpha_s\langle\hat{\nabla}_{coord}f(\underline{x}_t)-\nabla f(\underline{x}_t), x_t-x\rangle
		\end{aligned}
		\end{equation}
		which comes from Lemma \ref{lemma4}. Plugging (\ref{l2eq3}), (\ref{l2eq5}) into (\ref{l2eq6}), we get
		\begin{equation}
		\begin{aligned}
		&\mathbb{E}\left[f(\bar{x}_t)+\frac{\alpha_s(1+\tau\gamma_s)}{2\gamma_s}\|x-x_t\|^2\right]\\
		\leq&(1-\alpha_s-p_s)f(\bar{x}_{t-1})+\alpha_sf(x)+p_sf(\tilde{x})+\frac{\alpha_s}{2\gamma_s}\|x-x_{t-1}\|^2+\frac{\alpha_s}{\gamma_s}\eta_{s,t}\\
		&-\alpha_s\langle\hat{\nabla}_{coord}f(\underline{x}_t)-\nabla f(\underline{x}_t), x_t-x\rangle+\frac{6\alpha_s\gamma_s\mu^2L^2d}{1+\tau\gamma_s-L\alpha_s\gamma_s}\\
		\end{aligned}
		\end{equation}
		Subtracting both sides with $f(x)$ and then multiplying both sides with $\frac{\gamma_s}{\alpha_s}$, we get
		\begin{equation}
		\begin{aligned}
		&\frac{\gamma_s}{\alpha_s}\mathbb{E}\left[f(\bar{x}_t)-f(x)\right]+\frac{1+\tau\gamma_s}{2}\mathbb{E}\left[\|x-x_t\|^2\right]\\
		\leq&\frac{\gamma_s(1-\alpha_s-p_s)}{\alpha_s}\left[f(\bar{x}_{t-1})-f(x)\right]+\frac{\gamma_sp_s}{\alpha_s}\left[f(\tilde{x})-f(x)\right]+\frac{1}{2}\|x-x_{t-1}\|^2+\eta_{s,t}\\
		&-\gamma_s\langle \hat{\nabla}_{coord}f(\underline{x}_t)-\nabla f(\underline{x}_t), x_t-x\rangle+\frac{6\gamma_s^2\mu^2L^2d}{1+\tau\gamma_s-L\alpha_s\gamma_s}\\
		\end{aligned}
		\end{equation}
		Then we get the desired result for the zeroth-order case. Then we complete the proof.
	\end{proof}

\section{Proof of Theorem  \ref{theorem1}}	\label{app_b}
	\begin{lemma}	\label{lemma8}
		Suppose Assumption \ref{assum1} holds. Denote $\mathcal{L}_s = \frac{\gamma_s}{\alpha_s}+(T_s-1)\frac{\gamma_s(\alpha_s+p_s)}{\alpha_s}, \mathcal{R}_s = \frac{\gamma_s}{\alpha_s}(1-\alpha_s)+(T_s-1)\frac{\gamma_sp_s}{\alpha_s}$. Set $\theta_t =\begin{cases}
		\frac{\gamma_s}{\alpha_s}(\alpha_s+p_s), &t\leq T_s-1\\
		\frac{\gamma_s}{\alpha_s}, &t=T_s\\
		\end{cases}$
		
		$\bullet$ For the first-order case, assume that $\alpha_s\in[0, 1], p_s\in[0, 1]$ and $\gamma_s>0$ satisfy
		\[1+\tau\gamma_s-L\alpha_s\gamma_s >0, \; 1-\alpha_s-p_s\geq0, \qquad p_s-\frac{L\alpha_s\gamma_s}{1+\tau\gamma_s-L\alpha_s\gamma_s}>0\] Then we have
		\[\begin{aligned}
		\mathcal{L}_s\mathbb{E}\left[f(\tilde{x}^s)-f(x)\right] \leq \mathcal{R}_s\mathbb{E}\left[f(\tilde{x}^{s-1})-f(x)\right]+\mathbb{E}\left[\frac{1}{2}\|x-x^{s-1}\|^2-\frac{1}{2}\|x-x^s\|^2\right] +\sum_{t=1}^{T_s}\eta_{s,t}\\
		\end{aligned}\]
		
		$\bullet$ For the zeroth-order case, assume that $\alpha_s\in[0, 1], p_s\in[0, 1]$ and $\gamma_s>0$ satisfy
		\[1+\tau\gamma_s-L\alpha_s\gamma_s >0, \; 1-\alpha_s-p_s\geq0, \qquad p_s-\frac{4\alpha_s\gamma_sL}{1+\tau\gamma_s-L\alpha_s\gamma_s}>0\] Then we have
		\[\begin{aligned}
		\mathcal{L}_s\mathbb{E}\left[f(\tilde{x}^s)-f(x)\right] \leq& \mathcal{R}_s\mathbb{E}\left[f(\tilde{x}^{s-1})-f(x)\right]+\mathbb{E}\left[\frac{1}{2}\|x-x^{s-1}\|^2-\frac{1}{2}\|x-x^s\|^2\right]+\sum_{t=1}^{T_s}\eta_{s,t}\\
		&+\gamma_sT_sD\mu L\sqrt{d}+T_s\frac{6\gamma_s^2\mu^2L^2d}{1-L\alpha_s\gamma_s}
		\end{aligned}\]
	\end{lemma}
	\begin{proof}
		Define \[\Delta_t = \begin{cases}
		0,\; &\textrm{ for the first-order case}\\
		-\gamma_s\langle \hat{\nabla}_{coord}f(\underline{x}_t)-\nabla f(\underline{x}_t), x_t-x\rangle+\frac{6\gamma_s^2\mu^2L^2d}{1+\tau\gamma_s-L\alpha_s\gamma_s}, \;&\textrm{ for the zeroth-order case}\\
		\end{cases}\]
		From Lemma \ref{lemma2}, we have
		\begin{equation}
		\begin{aligned}
		\frac{\gamma_s}{\alpha_s}\mathbb{E}\left[f(\bar{x}_t)-f(x)\right] \leq&\frac{\gamma_s(1-\alpha_s-p_s)}{\alpha_s}\mathbb{E}\left[f(\bar{x}_{t-1})-f(x)\right]+\frac{\gamma_sp_s}{\alpha_s}\mathbb{E}\left[f(\tilde{x})-f(x)\right]\\
		&+\mathbb{E}\left[\frac{1}{2}\|x-x_{t-1}\|^2-\frac{1}{2}\|x-x_t\|^2\right]+\eta_{s,t}+\Delta_t\\
		\end{aligned}
		\end{equation}
		Summing the above inequality over $t= 1, ..., T_s$, with the definition of $\theta_t$, we have
		\begin{equation}
		\begin{aligned}
		&\sum_{t=1}^{T_s}\theta_t\mathbb{E}\left[f(\bar{x}_t)-f(x)\right]\\
		\leq& \left[\frac{\gamma_s}{\alpha_s}(1-\alpha_s)+(T_s-1)\frac{\gamma_sp_s}{\alpha_s}\right]\mathbb{E}\left[f(\tilde{x})-f(x)\right]+\mathbb{E}\left[\frac{1}{2}\|x-x_0\|^2-\frac{1}{2}\|x-x_{T_s}\|^2\right]+\sum_{t=1}^{T_s}\eta_{s,t}+\sum_{t=1}^{T_s}\Delta_t\\
		\end{aligned}
		\end{equation}
		Note that $\tilde{x}^s = \sum_{t=1}^{T}(\theta_t\bar{x}_t)/\sum_{t=1}^T\theta_t, x^s = x_{T_s}, x^{s-1} = x_0, \tilde{x}^{s-1}=\tilde{x}$, the convexity of $f$ and Jensen's inequality, we have
		\begin{equation}
		\begin{aligned}
		&\sum_{t=1}^{T_s}\theta_t\mathbb{E}\left[f(\tilde{x}^s)-f(x)\right]\\
		\leq& \left[\frac{\gamma_s}{\alpha_s}(1-\alpha_s)+(T_s-1)\frac{\gamma_sp_s}{\alpha_s}\right]\mathbb{E}\left[f(\tilde{x}^{s-1})-f(x)\right]+\mathbb{E}\left[\frac{1}{2}\|x-x^{s-1}\|^2-\frac{1}{2}\|x-x^s\|^2\right]\\
		&+\sum_{t=1}^{T_s}\eta_{s,t}+\sum_{t=1}^{T_s}\Delta_t\\
		\end{aligned}
		\end{equation}
		Denote $\mathcal{L}_s = \frac{\gamma_s}{\alpha_s}+(T_s-1)\frac{\gamma_s(\alpha_s+p_s)}{\alpha_s}, \mathcal{R}_s = \frac{\gamma_s}{\alpha_s}(1-\alpha_s)+(T_s-1)\frac{\gamma_sp_s}{\alpha_s}$, we have
		\begin{equation}	\label{l8eq3}
		\begin{aligned}
		\mathcal{L}_s\mathbb{E}\left[f(\tilde{x}^s)-f(x)\right] \leq \mathcal{R}_s\mathbb{E}\left[f(\tilde{x}^{s-1})-f(x)\right]+\mathbb{E}\left[\frac{1}{2}\|x-x^{s-1}\|^2-\frac{1}{2}\|x-x^s\|^2\right] +\sum_{t=1}^{T_s}\eta_{s,t}+\sum_{t=1}^{T_s}\Delta_t\\
		\end{aligned}
		\end{equation}
		\textbf{First-order Case}: Using the definition of $\Delta_t$ and (\ref{l8eq3}), we have
		\begin{equation}
		\begin{aligned}
		\mathcal{L}_s\mathbb{E}\left[f(\tilde{x}^s)-f(x)\right] \leq \mathcal{R}_s\mathbb{E}\left[f(\tilde{x}^{s-1})-f(x)\right]+\mathbb{E}\left[\frac{1}{2}\|x-x^{s-1}\|^2-\frac{1}{2}\|x-x^s\|^2\right] +\sum_{t=1}^{T_s}\eta_{s,t}\\
		\end{aligned}
		\end{equation}
		Then we get the desired result for first-order case.

		\textbf{Zeroth-order Case}: Using the definition of $\Delta_t$ and (\ref{l8eq3}), we have
		\begin{equation}	\label{l8eq1}
		\begin{aligned}
		\mathcal{L}_s\mathbb{E}\left[f(\tilde{x}^s)-f(x)\right] \leq& \mathcal{R}_s\mathbb{E}\left[f(\tilde{x}^{s-1})-f(x)\right]+\mathbb{E}\left[\frac{1}{2}\|x-x^{s-1}\|^2-\frac{1}{2}\|x-x^s\|^2\right]+\sum_{t=1}^{T_s}\eta_{s,t}\\
		&\underbrace{-\gamma_s\sum_{t=1}^{T_s}\mathbb{E}\left[\langle\hat{\nabla}_{coord}f(\underline{x}_t)-\nabla f(\underline{x}_t), x_t-x\rangle\right]}_{Q_3}+T_s\frac{6\gamma_s^2\mu^2L^2d}{1-L\alpha_s\gamma_s}
		\end{aligned}
		\end{equation}
		Next we have
		\begin{equation}	\label{l8eq2}
		\begin{aligned}
		Q_3=		&-\gamma_s\sum_{t=1}^{T_s}\mathbb{E}\left[\langle\hat{\nabla}_{coord}f(\underline{x}_t)-\nabla f(\underline{x}_t), \; x_t-x\rangle\right]\\
		\overset{\textrm{\ding{172}}}{\leq}& \gamma_s\sum_{t=1}^{T_s}\mathbb{E}\left[\|\hat{\nabla}_{coord}f(\underline{x}_t)-\nabla f(\underline{x}_t)\|\cdot \|x_t-x\|\right] \overset{\textrm{\ding{173}}}{\leq}	\gamma_sT_sD\mu L\sqrt{d}
		\end{aligned}
		\end{equation}
		where \ding{172} comes from Cauchy-Schwartz inequality and \ding{173} comes from Lemma \ref{lemma7} and Assumption \ref{assum1}. Plugging (\ref{l8eq2}) into (\ref{l8eq1}), we get
		\begin{equation}
		\begin{aligned}
		\mathcal{L}_s\mathbb{E}\left[f(\tilde{x}^s)-f(x)\right]
		\leq& \mathcal{R}_s\mathbb{E}\left[f(\tilde{x}^{s-1})-f(x)\right]+\mathbb{E}\left[\frac{1}{2}\|x-x^{s-1}\|^2-\frac{1}{2}\|x-x^s\|^2\right]+\sum_{t=1}^{T_s}\eta_{s,t}\\
		&+\gamma_sT_sD\mu L\sqrt{d}+T_s\frac{6\gamma_s^2\mu^2L^2d}{1-L\alpha_s\gamma_s}
		\end{aligned}
		\end{equation}
		Then we get the desired result for zeroth-order case. Then we complete the proof.
	\end{proof}

	\begin{proof}[ of Theorem \ref{theorem1}]
		We give proof to the first-order case and zeroth-order case respectively.
		
		\textbf{First-order Case}: Lemma \ref{lemma8} implies
		\begin{equation}
		\begin{aligned}
		\mathcal{L}_s\mathbb{E}\left[f(\tilde{x}^s)-f(x)\right]
		\leq \mathcal{R}_s\mathbb{E}\left[f(\tilde{x}^{s-1})-f(x)\right]+\mathbb{E}\left[\frac{1}{2}\|x-x^{s-1}\|^2-\frac{1}{2}\|x-x^s\|^2\right] +\sum_{t=1}^{T_s}\eta_{s,t}\\
		\end{aligned}
		\end{equation}
		Summing the above inequality over $s = 1, ..., S$, and set $x=x^*=\arg\min_{x\in\mathcal{C}}f(x)$, we have
		\begin{equation}	\label{t1eq4}
		\begin{aligned}
		&\mathcal{L}_S\mathbb{E}\left[f(\tilde{x}^S)-f(x^*)\right]+\sum_{s=1}^{S-1}\left(\mathcal{L}_s-\mathcal{R}_{s+1}\right)\mathbb{E}\left[f(\tilde{x})-f(x^*)\right]\\
		\leq& \mathcal{R}_1\mathbb{E}\left[f(\tilde{x}^0)-f(x^*)\right]+\mathbb{E}\left[\frac{1}{2}\|x^*-x^0\|^2-\frac{1}{2}\|x^*-x^S\|^2\right] +\sum_{s=1}^S\sum_{t=1}^{T_s}\eta_{s,t}\\
		\end{aligned}
		\end{equation}
		Next we prove $\mathcal{L}_s-\mathcal{R}_{s+1}\leq0$ for all $s= 1, ..., S-1$. When $s< s_0$, we have $\alpha_{s+1}=\alpha_s=\frac{1}{2}, \gamma_{s+1}=\gamma_s, p_{s+1}=p_s=\frac{1}{2}, T_{s+1}=2T_s$, thus
		\begin{equation}
		\begin{aligned}
		\mathcal{L}_s-\mathcal{R}_{s+1} =&\frac{\gamma_s}{\alpha_s}+(T_s-1)\frac{\gamma_s(\alpha_s+p_s)}{\alpha_s}-\left[\frac{\gamma_{s+1}}{\alpha_{s+1}}(1-\alpha_{s+1})+(T_{s+1}-1)\frac{\gamma_{s+1}p_{s+1}}{\alpha_{s+1}}\right]\\ =&\frac{\gamma_s}{\alpha_s}\left[1+(T_s-1)(\alpha_s+p_s)-(1-\alpha_s)-(2T_s-1)p_s\right] = \frac{\gamma_s}{\alpha_s}\left[T_s(\alpha_s-p_s)\right] = 0
		\end{aligned}
		\end{equation}
		When $s\geq s_0$, we have $\alpha_s=\frac{2}{s-s_0+4}, \gamma_s=\frac{1}{3L\alpha_s}, p_{s+1}=p_s=\frac{1}{2}, T_{s+1}=T_s$, thus
		\begin{equation}
		\begin{aligned}
		\mathcal{L}_s-\mathcal{R}_{s+1}  =&\frac{\gamma_s}{\alpha_s}+(T_s-1)\frac{\gamma_s(\alpha_s+p_s)}{\alpha_s}-\left[\frac{\gamma_{s+1}}{\alpha_{s+1}}(1-\alpha_{s+1})+(T_{s+1}-1)\frac{\gamma_{s+1}p_{s+1}}{\alpha_{s+1}}\right]\\= &\frac{\gamma_s}{\alpha_s}-\frac{\gamma_{s+1}}{\alpha_{s+1}}(1-\alpha_{s+1})+(T_{s_0}-1)\left[\frac{\gamma_s(\alpha_s+p_s)}{\alpha_s}-\frac{\gamma_{s+1}p_{s+1}}{\alpha_{s+1}}\right]\\
		=&\frac{1}{12L}+\frac{(T_{s_0}-1)(2(s-s_0+4)-1)}{24L}\geq0
		\end{aligned}
		\end{equation}
		Thus $\mathcal{L}_s-\mathcal{R}_{s+1}\geq0$ for $s=1, ..., S-1$. Note that $\mathcal{R}_1=\frac{2}{3L}$. Plugging this inequality into (\ref{t1eq4}), we get
		\begin{equation}
		\begin{aligned}
		\mathcal{L}_S\mathbb{E}\left[f(\tilde{x}^S)-f(x^*)\right] \leq&\mathcal{L}_S\mathbb{E}\left[f(\tilde{x}^S)-f(x^*)\right]+\sum_{s=1}^{S-1}\left(\mathcal{L}_s-\mathcal{R}_{s+1}\right)\mathbb{E}\left[f(\tilde{x})-f(x^*)\right]\\
		\leq& \mathcal{R}_1\mathbb{E}\left[f(\tilde{x}^0)-f(x^*)\right]+\mathbb{E}\left[\frac{1}{2}\|x^*-x^0\|^2-\frac{1}{2}\|x^*-x^S\|^2\right] +\sum_{s=1}^S\sum_{t=1}^{T_s}\eta_{s,t}\\
		\leq& \frac{2}{3L}\left[f(\tilde{x}^0)-f(x^*)\right]+\frac{1}{2}\|x^*-x^0\|^2 +\sum_{s=1}^S\sum_{t=1}^{T_s}\eta_{s,t} \overset{\textrm{\ding{172}}}{\leq} \frac{D_0}{6L} +\sum_{s=1}^S\sum_{t=1}^{T_s}\eta_{s,t}\\
		\end{aligned}
		\end{equation}
		where \ding{172} comes from the definition of $D_0$ that $D_0=4(f(\tilde{x}^0)-f(x^*))+3L\|x^0-x^*\|^2$.
		\begin{itemize}
			\item If $S\leq s_0$, then $\mathcal{L}_s = \frac{2^{S+1}}{3L}$, we have
			\begin{equation}	\label{t1eq5}
			\mathbb{E}\left[f(\tilde{x}^S)-f(x^*)\right]\leq \frac{D_0}{2^{S+2}}+\frac{3L}{2^{S+1}}\sum_{s=1}^S\sum_{t=1}^{T_s}\eta_{s,t}
			\end{equation}
			With the choice of $\eta_{s,t}$, we have
			\begin{equation}
			\sum_{s=1}^{S}\sum_{t=1}^{T_s}\eta_{s,t} \overset{\textrm{\ding{172}}}{\leq} \sum_{s=1}^{S}\frac{D_0}{sL}\leq\frac{D_0(\log S + 1)}{L}
			\end{equation}
			where \ding{172} comes from $\sum_{k=1}^{n} \frac{1}{k}\leq \log n + 1$. Plugging this inequality into (\ref{t1eq5}) we get
			\begin{equation}
			\mathbb{E}\left[f(\tilde{x}^S)-f(x^*)\right]\leq \frac{3D_0(\log S + 2)}{2^{S+1}}
			\end{equation}
			
			\item If $S>s_0$, we have
			\begin{equation}
			\begin{aligned}
			\mathcal{L}_S=&\frac{1}{3L\alpha_S^2}\left[1+(T_S-1)(\alpha_S+\frac{1}{2})\right]\\
			=&\frac{(S-s_0+4)(T_{s_0}-1)}{6L}+\frac{(S-s_0+4)^2(T_{s_0}+1)}{24L}
			\overset{\textrm{\ding{172}}}{\geq}&\frac{(S-s_0+4)^2n}{48L}
			\end{aligned}
			\end{equation}
			where \ding{172} comes from $T_{s_0}=2^{\lfloor\log_2n\rfloor+1-1}\geq n/2$. Then we have
			\begin{equation}	\label{t1eq6}
			\begin{aligned}
			\mathbb{E}\left[f(\tilde{x}^S)-f(x^*)\right]\leq \frac{8D_0}{n(S-s_0+4)^2}+\frac{48L}{n(S-s_0+4)^2}\sum_{s=1}^S\sum_{t=1}^{T_s}\eta_{s,t}
			\end{aligned}
			\end{equation}
			With the choice of $\eta_{s,t}$, we have
			\begin{equation}
				\sum_{s=1}^{S}\sum_{t=1}^{T_s}\eta_{s,t} \leq \sum_{s=1}^{S}\frac{D_0}{sL}\leq\frac{D_0(\log S + 1)}{L}
			\end{equation}
			Plugging this inequality into (\ref{t1eq6}) we get
			\begin{equation}
			\begin{aligned}
			\mathbb{E}\left[f(\tilde{x}^S)-f(x^*)\right]\leq\frac{48D_0(\log S + 2)}{n(S-s_0+4)^2}
			\end{aligned}
			\end{equation}
			Then we get the desired result for the first-order case.
		\end{itemize}
		
		\textbf{Zeroth-order Case}: From Lemma \ref{lemma8}, we get
		\begin{equation}
		\begin{aligned}
		\mathcal{L}_s\mathbb{E}\left[f(\tilde{x}^s)-f(x)\right] \leq& \mathcal{R}_s\mathbb{E}\left[f(\tilde{x}^{s-1})-f(x)\right]+\mathbb{E}\left[\frac{1}{2}\|x-x^{s-1}\|^2-\frac{1}{2}\|x-x^s\|^2\right]+\sum_{t=1}^{T_s}\eta_{s,t}\\
		&+\gamma_sT_sD\mu L\sqrt{d}+T_s\frac{6\gamma_s^2\mu^2L^2d}{1-L\alpha_s\gamma_s}\\
		\end{aligned}
		\end{equation}
		Summing the above inequality over $s=1, ..., S$ and set $x=x^*=\arg\min_{x\in\mathcal{C}}f(x)$, we have
		\begin{equation}	\label{t1eq1}
		\begin{aligned}
		&\mathcal{L}_S\mathbb{E}\left[f(\tilde{x}^S)-f(x^*)\right]+\sum_{s=1}^{S-1}\left(\mathcal{L}_s-\mathcal{R}_{s+1}\right)\mathbb{E}\left[f(\tilde{x})-f(x^*)\right]\\
		\leq& \mathcal{R}_1\mathbb{E}\left[f(\tilde{x}^0)-f(x^*)\right]+\mathbb{E}\left[\frac{1}{2}\|x^*-x^0\|^2-\frac{1}{2}\|x^*-x^S\|^2\right]+\sum_{s=1}^{S}\sum_{t=1}^{T_s}\eta_{s,t}\\
		&+\sum_{s=1}^{S}\gamma_sT_sD\mu L\sqrt{d}+\sum_{s=1}^{S}T_s\frac{6\gamma_s^2\mu^2L^2d}{1-L\alpha_s\gamma_s}\\
		\end{aligned}
		\end{equation}
		Next we prove $\mathcal{L}_s-\mathcal{R}_{s+1}\leq0$ for all $s= 1, ..., S-1$. When $s\leq s_0$, we have
		\begin{equation}
		\begin{aligned}
		\mathcal{L}_s-\mathcal{R}_{s+1} = &\frac{\gamma_s}{\alpha_s}\left[1+(T_s-1)(\alpha_s+p_s)-(1-\alpha_s)(2T_s-1)p_s\right] = \frac{\gamma_s}{\alpha_s}\left[T_s(\alpha_s-p_s)\right] = 0
		\end{aligned}
		\end{equation}
		When $s>s_0$, $\frac{\gamma_s}{\alpha_s} = \frac{1}{5L\alpha_s^2}= \frac{(s-s_0+4)^2}{20L}$, we have
		\begin{equation}
		\begin{aligned}
		\mathcal{L}_s-\mathcal{R}_{s+1}  =&\frac{\gamma_s}{\alpha_s}+(T_s-1)\frac{\gamma_s(\alpha_s+p_s)}{\alpha_s}-\left[\frac{\gamma_{s+1}}{\alpha_{s+1}}(1-\alpha_{s+1})+(T_{s+1}-1)\frac{\gamma_{s+1}p_{s+1}}{\alpha_{s+1}}\right]\\= &\frac{\gamma_s}{\alpha_s}-\frac{\gamma_{s+1}}{\alpha_{s+1}}(1-\alpha_{s+1})+(T_{s_0}-1)\left[\frac{\gamma_s(\alpha_s+p_s)}{\alpha_s}-\frac{\gamma_{s+1}p_{s+1}}{\alpha_{s+1}}\right]\\
		=&\frac{1}{20L}+\frac{(T_{s_0}-1)(2(s-s_0+4)-1)}{40L}\geq0
		\end{aligned}
		\end{equation}
		Thus $\mathcal{L}_s-\mathcal{R}_{s+1}\geq0$ for $s=1, ..., S-1$. Plugging this inequality into (\ref{t1eq1}), we get
		\begin{equation}
		\begin{aligned}
		&\mathcal{L}_S\mathbb{E}\left[f(\tilde{x}^S)-f(x^*)\right] \leq\mathcal{L}_S\mathbb{E}\left[f(\tilde{x}^S)-f(x^*)\right]+\sum_{s=1}^{S-1}\left(\mathcal{L}_s-\mathcal{R}_{s+1}\right)\mathbb{E}\left[f(\tilde{x})-f(x^*)\right]\\
		\leq& \mathcal{R}_1\mathbb{E}\left[f(\tilde{x}^0)-f(x^*)\right]+\mathbb{E}\left[\frac{1}{2}\|x^*-x^0\|^2-\frac{1}{2}\|x^*-x^S\|^2\right]+\sum_{s=1}^{S}\sum_{t=1}^{T_s}\eta_{s,t} +\sum_{s=1}^{S}\gamma_sT_sD\mu L\sqrt{d}+\sum_{s=1}^{S}T_s\frac{6\gamma_s^2\mu^2L^2d}{1-L\alpha_s\gamma_s}\\
		\leq&\frac{2}{5L}\left[f(\tilde{x}^0)-f(x^*)\right]+\frac{1}{2}\|x^*-x^0\|^2+\sum_{s=1}^{S}\sum_{t=1}^{T_s}\eta_{s,t} +\sum_{s=1}^{S}\gamma_sT_sD\mu L\sqrt{d}+\sum_{s=1}^{S}T_s\frac{6\gamma_s^2\mu^2L^2d}{1-\frac{1}{5}}\\
		\overset{\textrm{\ding{172}}}{\leq}&\frac{D_0}{10L}+\sum_{s=1}^{S}\sum_{t=1}^{T_s}\eta_{s,t} +\sum_{s=1}^{S}\gamma_sT_sD\mu L\sqrt{d}+\sum_{s=1}^{S}T_s\frac{\gamma_s}{\alpha_s}\frac{3\mu^2Ld}{2}\\
		=&\frac{D_0}{10L}+\sum_{s=1}^{S}\sum_{t=1}^{T_s}\eta_{s,t} +\sum_{s=1}^{S}\gamma_sT_sD\mu L\sqrt{d}+\sum_{s=1}^{S}T_s\frac{\gamma_s}{\alpha_s}\frac{3\mu^2Ld}{2}\\
		\end{aligned}
		\end{equation}
		where \ding{172} comes from the definition of $D_0$ that $D_0 = 4(f(\tilde{x}^0)-f(x^*))+5L\|x^0-x^*\|^2$.
		\begin{itemize}
			\item If $S\leq s_0$, then $\mathcal{L}_S = \frac{2^{S+1}}{5L}, \alpha_S = \frac{1}{2}, \gamma_S = \frac{2}{5L}$. We have
			\begin{equation}	\label{t1eq2}
			\begin{aligned}
			&\mathbb{E}\left[f(\tilde{x}^S)-f(x^*)\right]\\
			\leq&\frac{D_0}{2^{S+2}}+\frac{5L}{2^{S+1}}\sum_{s=1}^{S}\sum_{t=1}^{T_s}\eta_{s,t}+D\mu L\sqrt{d}+3\mu^2Ld\\
			\end{aligned}
			\end{equation}
			With the choice of $\eta_{s,t}$, we have
			\begin{equation}
				\sum_{s=1}^{S}\sum_{t=1}^{T_s}\eta_{s,t} \leq \sum_{s=1}^{S}\frac{D_0}{sL}\leq\frac{D_0(\log S + 1)}{L}
			\end{equation}
			Plugging this inequality into (\ref{t1eq2}) we get
			\begin{equation}
			\begin{aligned}
			\mathbb{E}\left[f(\tilde{x}^S)-f(x^*)\right]\leq \frac{5D_0(\log S + 2)}{2^{S+1}}+D\mu L\sqrt{d}+3\mu^2Ld
			\end{aligned}
			\end{equation}
			
			\item If $S>s_0$, then $\mathcal{L}_s-\mathcal{R}_s = \gamma_sT_s>0$ for $s>s_0$, and $\sum_{s=1}^{s_0}2^{s-1} = 2^{s_0}-1\leq 2n$. We have
			\begin{equation}
			\begin{aligned}
			\mathcal{L}_S=&\frac{1}{5L\alpha_S^2}\left[1+(T_S-1)(\alpha_S+\frac{1}{2})\right]\\
			=&\frac{(S-s_0+4)(T_{s_0}-1)}{10L}+\frac{(S-s_0+4)^2(T_{s_0}+1)}{40L}
			\overset{\textrm{\ding{172}}}{\geq}&\frac{(S-s_0+4)^2n}{80L}
			\end{aligned}
			\end{equation}
			where \ding{172} comes from $T_{s_0}=2^{\lfloor\log_2n\rfloor+1-1}\geq n/2$. And
			\begin{equation}
			\begin{aligned}
			\sum_{s=1}^{S}T_s\frac{\gamma_s}{\alpha_s} =& \sum_{s=1}^{s_0}\frac{2^{s+1}}{5L}+\sum_{s=s_0+1}^S\frac{(s-s_0+4)^2}{20L}2^{s_0-1}\overset{\textrm{\ding{172}}}{\leq} \frac{8n}{5L}+\frac{n(S-s_0+4)^3}{20L} \leq \frac{n(S-s_0+4)^3}{10L}
			\end{aligned}
			\end{equation}
			where \ding{172} comes from $\sum_{i=1}^ni^2\leq n^3$. And
			\begin{equation}
			\begin{aligned}
			\sum_{s=1}^{S}\gamma_sT_s \leq \frac{1}{2}\sum_{s=1}^{S}T_s\frac{\gamma_s}{\alpha_s}\leq \frac{n(S-s_0+4)^3}{20L}
			\end{aligned}
			\end{equation}
			Thus
			\begin{equation}	\label{t1eq3}
			\begin{aligned}
			\mathbb{E}\left[f(\tilde{x}^S)-f(x^*)\right] \leq&\frac{8D_0}{n(S-s_0+4)^2}+\frac{80L}{n(S-s_0+4)^2}\sum_{s=1}^{S}\sum_{t=1}^{T_s}\eta_{s,t}+8(S-s_0+4)D\mu L\sqrt{d}\\
			&+12(S-s_0+4)\mu^2Ld\\
			\end{aligned}
			\end{equation}
			With the choice of $\eta_{s,t}$, we have
			\begin{equation}
				\sum_{s=1}^{S}\sum_{t=1}^{T_s}\eta_{s,t} \leq \sum_{s=1}^{S}\frac{D_0}{sL}\leq\frac{D_0(\log S + 1)}{L}
			\end{equation}
			Plugging this inequality into (\ref{t1eq3}) we get
			\begin{equation}
			\begin{aligned}
			\mathbb{E}\left[f(\tilde{x}^S)-f(x^*)\right] \leq&\frac{80D_0(\log S + 2)}{n(S-s_0+4)^2}+8(S-s_0+4)D\mu L\sqrt{d} +12(S-s_0+4)\mu^2Ld
			\end{aligned}
			\end{equation}
		\end{itemize}
		Then we get the desired result for the zeroth-order case. Then we complete the proof.
	\end{proof}

\section{Proof of Theorem  \ref{theorem2}}	\label{app_c}
	Theorem \ref{theorem2} is a direct result of Theorem \ref{theorem2a}, Theorem \ref{theorem2b} and Theorem \ref{theorem2c}. First we give a refined version of Lemma $\ref{lemma2}$.
	\begin{lemma}	\label{lemma9}
	Suppose Assumption \ref{assum2} holds. Conditioning on $x_1, ..., x_{t-1}$
	
	$\bullet$ For the first-order case, assume that $\alpha_s\in[0, 1], p_s\in[0, 1]$ and $\gamma_s>0$ satisfy
	\[1+\tau\gamma_s-L\alpha_s\gamma_s >0, \; 1-\alpha_s-p_s\geq0, \qquad p_s-\frac{L\alpha_s\gamma_s}{1+\tau\gamma_s-L\alpha_s\gamma_s}>0\]
	Then we have
	\[\begin{aligned}
	&\frac{\gamma_s}{\alpha_s}\mathbb{E}\left[f(\bar{x}_t)-f(x)\right]+\frac{1+\tau\gamma_s}{2}\mathbb{E}\left[\|x-x_t\|^2\right]\\
	\leq&\frac{\gamma_s(1-\alpha_s-p_s)}{\alpha_s}\left[f(\bar{x}_{t-1})-f(x)\right]+\frac{\gamma_sp_s}{\alpha_s}\left[f(\tilde{x})-f(x)\right]+\frac{1}{2}\|x-x_{t-1}\|^2+\eta_{s,t}
	\end{aligned}\]
	
	$\bullet$ For the zeroth-order case, assume that $\alpha_s\in[0, 1], p_s\in[0, 1]$ and $\gamma_s>0$ satisfy
	\[1+\tau\gamma_s-L\alpha_s\gamma_s >0, \; 1-\alpha_s-p_s\geq0, \qquad p_s-\frac{4\alpha_s\gamma_sdL}{1+\tau\gamma_s-L\alpha_s\gamma_s}>0\]
	Then we have
	\[\begin{aligned}
	&\frac{\gamma_s}{\alpha_s}\mathbb{E}\left[f(\bar{x}_t)-f(x)\right]+\frac{1+(1-c)\tau\gamma_s}{2}\mathbb{E}\left[\|x-x_t\|^2\right]\\
	\leq&\frac{\gamma_s(1-\alpha_s-p_s)}{\alpha_s}\left[f(\bar{x}_{t-1})-f(x)\right]+\frac{\gamma_sp_s}{\alpha_s}\left[f(\tilde{x})-f(x)\right]+\frac{1}{2}\|x-x_{t-1}\|^2+\eta_{s,t} +\frac{\gamma_s\mu^2L^2d}{2c\tau}+\frac{6\gamma_s^2\mu^2L^2d}{1+\tau\gamma_s-L\alpha_s\gamma_s}\\
	\end{aligned}\]
	\end{lemma}
	\begin{proof}
		For the first-order case, the result is the same as that in Lemma \ref{lemma2}. Now we give proof to the result of the zeroth-order case. From Lemma \ref{lemma2} we have
		\begin{equation}	\label{l9eq1}
		\begin{aligned}
		&\frac{\gamma_s}{\alpha_s}\mathbb{E}\left[f(\bar{x}_t)-f(x)\right]+\frac{1+\tau\gamma_s}{2}\mathbb{E}\left[\|x-x_t\|^2\right]\\
		\leq&\frac{\gamma_s(1-\alpha_s-p_s)}{\alpha_s}\left[f(\bar{x}_{t-1})-f(x)\right]+\frac{\gamma_sp_s}{\alpha_s}\left[f(\tilde{x})-f(x)\right]+\frac{1}{2}\|x-x_{t-1}\|^2+\eta_{s,t}\\
		&-\gamma_s\langle \hat{\nabla}_{coord}f(\underline{x}_t)-\nabla f(\underline{x}_t), x_t-x\rangle+\frac{6\gamma_s^2\mu^2L^2d}{1+\tau\gamma_s-L\alpha_s\gamma_s}\\
		\end{aligned}
		\end{equation}
		From $b\langle u, v\rangle-\frac{a}{2}\|v\|^2\leq\frac{b^2}{2a}\|u\|^2$ we have for $c>0$
		\begin{equation}	\label{l9eq2}
		-\gamma_s\langle \hat{\nabla}_{coord}f(\underline{x}_t)-\nabla f(\underline{x}_t), x_t-x\rangle-\frac{c\tau\gamma_s}{2}\|x_t-x\|^2\leq\frac{\gamma_s}{2c\tau}\|\hat{\nabla}_{coord}f(\underline{x}_t)-\nabla f(\underline{x}_t)\|^2
		\end{equation}
		Plugging (\ref{l9eq2}) into (\ref{l9eq1}), we get
		\begin{equation}
		\begin{aligned}
		&\frac{\gamma_s}{\alpha_s}\mathbb{E}\left[f(\bar{x}_t)-f(x)\right]+\frac{1+(1-c)\tau\gamma_s}{2}\mathbb{E}\left[\|x-x_t\|^2\right]\\
		\leq&\frac{\gamma_s(1-\alpha_s-p_s)}{\alpha_s}\left[f(\bar{x}_{t-1})-f(x)\right]+\frac{\gamma_sp_s}{\alpha_s}\left[f(\tilde{x})-f(x)\right]+\frac{1}{2}\|x-x_{t-1}\|^2+\eta_{s,t}\\
		&+\frac{\gamma_s}{2c\tau}\mathbb{E}\left[\|\hat{\nabla}_{coord}f(\underline{x}_t)-\nabla f(\underline{x}_t)\|^2\right]+\frac{6\gamma_s^2\mu^2L^2d}{1+\tau\gamma_s-L\alpha_s\gamma_s}\\
		\overset{\textrm{\ding{172}}}{\leq}&\frac{\gamma_s(1-\alpha_s-p_s)}{\alpha_s}\left[f(\bar{x}_{t-1})-f(x)\right]+\frac{\gamma_sp_s}{\alpha_s}\left[f(\tilde{x})-f(x)\right]+\frac{1}{2}\|x-x_{t-1}\|^2+\eta_{s,t}\\
		&+\frac{\gamma_s\mu^2L^2d}{2c\tau}+\frac{6\gamma_s^2\mu^2L^2d}{1+\tau\gamma_s-L\alpha_s\gamma_s}\\
		\end{aligned}
		\end{equation}
		where \ding{172} comes from Lemma \ref{lemma7}. Then we complete the proof.
	\end{proof}

	\begin{theorem}	\label{theorem2a}
		Suppose Assumption \ref{assum2} holds. Denote $s_0= \lfloor\log n\rfloor+1$. Suppose $s\leq s_0$, set $\{T_s\}, \{\alpha_s\}, \{p_s\}, \{\eta_{s,t}\}, \{\theta_t\}$ as \[T_s = 2^{s-1}, \; \alpha_s=\frac{1}{2}, \; p_s = \frac{1}{2}, \; \eta_{s,t}=\frac{D_0}{sT_sL}, \; \theta_t=\begin{cases}
		\Gamma_{t-1}-(1-\alpha_s-p_s)\Gamma_{t}, &t\leq T_s-1\\
		\Gamma_{t-1}, &t=T_{s}\\
		\end{cases}\] where $D_0, \Gamma_{t}$ will be specified below for two cases respectively.
		
		$\bullet$ For the first-order case, set $\gamma_s=\frac{1}{3L\alpha_s}, \Gamma_{t}=\left(1+\tau\gamma_s\right)^{t}, D_0=4(f(\tilde{x}^0)-f(x^*))+3L\|x^0-x^*\|^2$, we have \[\begin{aligned}
		\mathbb{E}\left[f(\tilde{x}^{S})-f(x^*)\right]\leq \frac{3D_0(\log S + 2)}{2^{S+1}}
		\end{aligned}\]
		
		$\bullet$ For the zeroth-order case, set $\gamma_s=\frac{1}{5L\alpha_s}, \Gamma_{t}=\left(1+\frac{\tau\gamma_s}{2}\right)^{t}, D_0=4(f(\tilde{x}^0)-f(x^*))+5L\|x^0-x^*\|^2$, we have \[\begin{aligned}
		&\mathbb{E}\left[f(\tilde{x}^S)-f(x^*)\right] \leq& \frac{5D_0(\log S + 2)}{2^{S+1}}+\frac{\mu^2L^2d}{2\tau}+3\mu^2Ld\\
		\end{aligned}\]
	\end{theorem}
	\begin{proof}
		We give proof to the first-order case and zeroth-order case respectively.
		
		\textbf{First-order Case}: We have $\alpha_s=p_s=\frac{1}{2}, \gamma_s=\frac{2}{3L}, T_s=2^{s-1}$. Summing up Lemma \ref{lemma9} for $t=1, ..., T_s$, we get
		\begin{equation}
		\begin{aligned}
		&\sum_{t=1}^{T_s}\frac{\gamma_s}{\alpha_s}\mathbb{E}\left[f(\bar{x}_t)-f(x)\right]+\frac{1}{2}\mathbb{E}\left[\|x_{T_s}-x\|^2\right]+\sum_{t=1}^{T_s}\frac{\tau\gamma_s}{2}\|x_t-x\|^2\\
		\leq&T_s\frac{\gamma_sp_s}{\alpha_s}\left[f(\tilde{x})-f(x)\right]+\frac{1}{2}\|x_0-x\|^2+\sum_{t=1}^{T_s}\eta_{s,t}
		\end{aligned}
		\end{equation}
		From the definition of $\tilde{x}^s, \tilde{x}^{s-1}, x^s, x^{s-1}$, the fact that $\frac{\gamma_s}{\alpha_s}=\frac{4}{3L}$, the convexity of $f$ and Jensen's inequality, we have
		\begin{equation}	\label{t2aeq3}
		\begin{aligned}
		\frac{4T_s}{3L}\mathbb{E}\left[f(\tilde{x}^{s})-f(x)\right]+\frac{1}{2}\mathbb{E}\left[\|x^s-x\|^2\right]\leq& \frac{4T_s}{6L}\left[f(\tilde{x}^{s-1})-f(x)\right]+\frac{1}{2}\|x^{s-1}-x\|^2+\sum_{t=1}^{T_s}\eta_{s,t}\\
		=& \frac{4T_{s-1}}{3L}\left[f(\tilde{x}^{s-1})-f(x)\right]+\frac{1}{2}\|x^{s-1}-x\|^2+\sum_{t=1}^{T_s}\eta_{s,t}\\
		\end{aligned}
		\end{equation}
		Summing up (\ref{t2aeq3}) for $s= 1, ..., S$, we get
		\begin{equation}	\label{t2aeq6}
		\begin{aligned}
		\frac{4T_S}{3L}\mathbb{E}\left[f(\tilde{x}^{S})-f(x)\right]+\frac{1}{2}\mathbb{E}\left[\|x^S-x\|^2\right]\leq \frac{2}{3L}\left[f(\tilde{x}^{0})-f(x)\right]+\frac{1}{2}\|x^{0}-x\|^2+\sum_{s=1}^S\sum_{t=1}^{T_s}\eta_{s,t}\\
		\end{aligned}
		\end{equation}
		Set $x=x^*$, we get
		\begin{equation}	\label{t2aeq4}
		\begin{aligned}
		\mathbb{E}\left[f(\tilde{x}^{S})-f(x^*)\right]+\frac{3L}{8T_S}\mathbb{E}\left[\|x^S-x^*\|^2\right]\leq \frac{f(\tilde{x}^{0})-f(x^*)}{2T_S}+\frac{3L}{8T_S}\|x^{0}-x^*\|^2+\frac{3L}{4T_S}\sum_{s=1}^S\sum_{t=1}^{T_s}\eta_{s,t}\\
		\end{aligned}
		\end{equation}
		With the choice of $\eta_{s,t}$, we have
		\begin{equation}
			\sum_{s=1}^{S}\sum_{t=1}^{T_s}\eta_{s,t} \leq \sum_{s=1}^{S}\frac{D_0}{sL}\leq\frac{D_0(\log S + 1)}{L}
		\end{equation}
		Plugging this inequality into (\ref{t2aeq4}), with the definition of $D_0$ and $T_s=2^{s-1}$, we get
		\begin{equation}	\label{t2aeq9}
		\begin{aligned}
		\mathbb{E}\left[f(\tilde{x}^{S})-f(x^*)\right]\leq\mathbb{E}\left[f(\tilde{x}^{S})-f(x^*)\right]+\frac{3L}{8T_s}\mathbb{E}\left[\|x^S-x^*\|^2\right]\leq \frac{3D_0(\log S + 2)}{2^{S+1}}
		\end{aligned}
		\end{equation}
		Then we get the desired result for the first-order case.
		
		\textbf{Zeroth-order Case}: We have $\alpha_s=p_s=\frac{1}{2}, \gamma_s=\frac{2}{5L}, T_s=2^{s-1}$. Summing up Lemma \ref{lemma9} with $c=1$ for $t=1, ..., T_s$, we get
		\begin{equation}
		\begin{aligned}
		&\sum_{t=1}^{T_s}\frac{\gamma_s}{\alpha_s}\mathbb{E}\left[f(\bar{x}_t)-f(x)\right]+\frac{1}{2}\|x_{T_s}-x\|^2\\
		\leq& T_s\frac{\gamma_sp_s}{\alpha_s}\left[f(\tilde{x})-f(x)\right]+\frac{1}{2}\|x_0-x\|^2+\sum_{t=1}^{T_s}\eta_{s,t}+T_s\frac{\gamma_s\mu^2L^2d}{2\tau}+T_s\frac{6\mu^2d}{5}
		\end{aligned}
		\end{equation}
		From the definition of $\tilde{x}^s, \tilde{x}^{s-1}, x^s, x^{s-1}$, the fact that $\frac{\gamma_s}{\alpha_s}=\frac{4}{5L}$ , the convexity of $f$ and Jensen's inequality, we have
		\begin{equation}	\label{t2aeq1}
		\begin{aligned}
		&\frac{4T_s}{5L}\mathbb{E}\left[f(\tilde{x}^s)-f(x)\right]+\frac{1}{2}\mathbb{E}\left[\|x^s-x\|^2\right]\leq\sum_{t=1}^{T_s}\frac{4}{5L}\mathbb{E}\left[f(\bar{x}_t)-f(x)\right]+\frac{1}{2}\|x_{T_s}-x\|^2\\
		\leq& \frac{2T_s}{5L}\left[f(\tilde{x}^{s-1})-f(x)\right]+\frac{1}{2}\|x^{s-1}-x\|^2+\sum_{t=1}^{T_s}\eta_{s,t}+T_s\frac{\mu^2Ld}{5\tau}+T_s\frac{6\mu^2d}{5}\\
		=& \frac{4T_{s-1}}{5L}\left[f(\tilde{x}^{s-1})-f(x)\right]+\frac{1}{2}\|x^{s-1}-x\|^2+\sum_{t=1}^{T_s}\eta_{s,t}+T_s\frac{\mu^2Ld}{5\tau}+T_s\frac{6\mu^2d}{5}\\
		\end{aligned}
		\end{equation}
		Summing up (\ref{t2aeq1}) for $s=1, ..., S$, we get
		\begin{equation}
		\begin{aligned}
		&\frac{4T_S}{5L}\mathbb{E}\left[f(\tilde{x}^S)-f(x)\right]+\frac{1}{2}\mathbb{E}\left[\|x^S-x\|^2\right]\\
		\leq& \frac{4T_{0}}{5L}\left[f(\tilde{x}^{0})-f(x)\right]+\frac{1}{2}\|x^{0}-x\|^2+\sum_{s=1}^S\sum_{t=1}^{T_s}\eta_{s,t}+\sum_{s=1}^ST_s\frac{\mu^2Ld}{5\tau}+\sum_{s=1}^ST_s\frac{6\mu^2d}{5}\\
		\end{aligned}
		\end{equation}
		Note that $T_s=2^{s-1}$. Setting $x=x^*$, we have
		\begin{equation}	\label{t2aeq2}
		\begin{aligned}
		&\mathbb{E}\left[f(\tilde{x}^S)-f(x^*)\right]+\frac{5L}{8T_S}\mathbb{E}\left[\|x^S-x^*\|^2\right]\\
		\leq&\frac{1}{2^S}\left[f(\tilde{x}^0)-f(x^*)\right]+\frac{5L}{2^{S+2}}\|x^0-x^*\|^2+\frac{5L}{2^{S+1}}\sum_{s=1}^S\sum_{t=1}^{T_s}\eta_{s,t}+\frac{\mu^2L^2d}{2\tau}+3\mu^2Ld\\
		\overset{\textrm{\ding{172}}}{\leq}& \frac{D_0}{2^{S+2}}+\frac{5L}{2^{S+1}}\sum_{s=1}^S\sum_{t=1}^{T_s}\eta_{s,t}+\frac{\mu^2L^2d}{2\tau}+3\mu^2Ld\\
		\end{aligned}
		\end{equation}
		where \ding{172} comes from the definition of $D_0$. With the choice of $\eta_{s,t}$, we have
		\begin{equation}
			\sum_{s=1}^{S}\sum_{t=1}^{T_s}\eta_{s,t} \leq \sum_{s=1}^{S}\frac{D_0}{sL}\leq\frac{D_0(\log S + 1)}{L}
		\end{equation}
		Plugging this inequality into (\ref{t2aeq2}) we get
		\begin{equation}	\label{t2aeq8}
		\begin{aligned}
		\mathbb{E}\left[f(\tilde{x}^S)-f(x^*)\right] \leq&\mathbb{E}\left[f(\tilde{x}^S)-f(x^*)\right]+\frac{5L}{8T_S}\mathbb{E}\left[\|x^S-x^*\|^2\right]\\
		\leq& \frac{5D_0(\log S + 2)}{2^{S+1}}+\frac{\mu^2L^2d}{2\tau}+3\mu^2Ld\\
		\end{aligned}
		\end{equation}
		Then we get the desired result for the zeroth-order case. Then we complete the proof.
	\end{proof}

	\begin{theorem}	\label{theorem2b}
		Suppose Assumption \ref{assum2} holds. Denote $s_0=\lfloor\log n\rfloor+1$. Suppose $s> s_0$, set $\{T_s\}, \{\alpha_s\}, \{p_s\}, \{\eta_{s,t}\}, \{\theta_t\}$ as \[T_s = T_{s_0} = 2^{s_0-1}, \; \alpha_s=\frac{1}{2}, \; p_s = \frac{1}{2}, \; \eta_{s,t}=\frac{\left(\frac{4}{5}\right)^{s-s_0-1}D_0}{snL}, \; \theta_t=\begin{cases}
		\Gamma_{t-1}-(1-\alpha_s-p_s)\Gamma_{t}, &t\leq T_s-1\\
		\Gamma_{t-1}, &t=T_{s}\\
		\end{cases}\]  where $D_0, \Gamma_{t}$ will be specified below for two cases respectively.
		
		$\bullet$ For the first-order case, set $\gamma_s=\frac{1}{3L\alpha_s}, \Gamma_{t}=\left(1+\tau\gamma_s\right)^{t}, D_0=4(f(\tilde{x}^0)-f(x^*))+3L\|x^0-x^*\|^2$ for $s>s_0$. Suppose $n\geq\frac{3L}{4\tau}$,  we have \[\begin{aligned}
		\mathbb{E}\left[f(\tilde{x}^{S})-f(x^*)\right]\leq \left(\frac{4}{5}\right)^{S-s_0}\frac{5D_0(\log S + 2)}{n}
		\end{aligned}\]
		
		$\bullet$ For the zeroth-order case, set $\gamma_s=\frac{1}{5L\alpha_s}, \Gamma_{t}=\left(1+\frac{\tau\gamma_s}{2}\right)^{t}, D_0=4(f(\tilde{x}^0)-f(x^*))+5L\|x^0-x^*\|^2$ for $s>s_0$. Suppose $n\geq\frac{5L}{4\tau}$, we have \[\begin{aligned}
		&\mathbb{E}\left[f(\tilde{x}^S)-f(x^*)\right] \leq& \left(\frac{4}{5}\right)^{S-s_0}\frac{8D_0(\log S + 2)}{n}+\frac{5\mu^2L^2d}{\tau}+18\mu^2Ld
		\end{aligned}\]
	\end{theorem}
	\begin{proof}
		We give proof to the first-order case and zeroth-order case respectively.
		
		\textbf{First-order Case}: We have $\alpha_s=p_s=\frac{1}{2}, \gamma_s = \frac{2}{3L}$. Note that for the first-order case, we have $f=f$. From Lemma \ref{lemma9} we have
		\begin{equation}	\label{t2beq7}
		\begin{aligned}
		\frac{\gamma_s}{\alpha_s}\mathbb{E}\left[f(\bar{x}_t)-f(x)\right]+\frac{1+\tau\gamma_s}{2}\mathbb{E}\left[\|x-x_t\|^2\right] \leq\frac{\gamma_s}{2\alpha_s}\left[f(\tilde{x})-f(x)\right]+\frac{1}{2}\|x-x_{t-1}\|^2+\eta_{s,t}
		\end{aligned}
		\end{equation}
		Multiplying both sides of (\ref{t2beq7}) with $\theta_t=\Gamma_{t-1}=\left(1+\tau\gamma_s\right)^{t-1}$, we get
		\begin{equation}	\label{t2beq8}
		\begin{aligned}
		\frac{\gamma_s}{\alpha_s}\theta_t\mathbb{E}\left[f(\bar{x}_t)-f(x)\right]+\frac{\Gamma_{t}}{2}\mathbb{E}\left[\|x-x_t\|^2\right]\leq\frac{\gamma_s}{2\alpha_s}\theta_t\left[f(\tilde{x})-f(x)\right]+\frac{\Gamma_{t-1}}{2}\|x-x_{t-1}\|^2+\theta_t\eta_{s,t}
		\end{aligned}
		\end{equation}
		Summing up (\ref{t2beq8}) for $t=1, ..., T_s$, we get
		\begin{equation}
		\begin{aligned}
		\frac{\gamma_s}{\alpha_s}\sum_{t=1}^{T_s}\theta_t\mathbb{E}\left[f(\bar{x}_t)-f(x)\right]+\frac{\Gamma_{T_s}}{2}\mathbb{E}\left[\|x_{T_s}-x\|^2\right]\leq\frac{\gamma_s}{2\alpha_s}\sum_{t=1}^{T_s}\theta_t\left[f(\tilde{x})-f(x)\right]+\frac{1}{2}\|x_0-x\|^2+\sum_{t=1}^{T_s}\theta_t\eta_{s,t}
		\end{aligned}
		\end{equation}
		Since we have
		\begin{equation}
		\begin{aligned}
		\Gamma_{T_s}=\left(1+\tau\gamma_s\right)^{T_s}=\left(1+\tau\gamma_s\right)^{T_{s_0}}\geq1+\tau\gamma_sT_{s_0}\geq1+\frac{\tau n}{3L}\overset{\textrm{\ding{172}}}{\geq}\frac{5}{4}
		\end{aligned}
		\end{equation}
		where \ding{172} comes from the assumption that $n\geq\frac{3L}{4\tau}$. Then we get
		\begin{equation}
		\begin{aligned}
		&\frac{5}{4}\left(\frac{\gamma_s}{2\alpha_s}\sum_{t=1}^{T_s}\theta_t\mathbb{E}\left[f(\bar{x}_t)-f(x)\right]+\frac{1}{2}\mathbb{E}\left[\|x_{T_s}-x\|^2\right]\right)\leq\frac{\gamma_s}{\alpha_s}\sum_{t=1}^{T_s}\theta_t\mathbb{E}\left[f(\bar{x}_t)-f(x)\right]+\frac{\Gamma_{T_s}}{2}\mathbb{E}\left[\|x_{T_s}-x\|^2\right]\\
		&\leq\frac{\gamma_s}{2\alpha_s}\sum_{t=1}^{T_s}\theta_t\left[f(\tilde{x})-f(x)\right]+\frac{1}{2}\|x_0-x\|^2+\sum_{t=1}^{T_s}\theta_t\eta_{s,t}
		\end{aligned}
		\end{equation}
		From the definition of $\tilde{x}^s, \tilde{x}^{s-1}, x^s, x^{s-1}$, the fact that $\frac{\gamma_s}{\alpha_s}=\frac{4}{3L}$, the convexity of $f$ and Jensen's inequality, we have
		\begin{equation}
		\begin{aligned}
		&\frac{5}{4}\left(\frac{2}{3L}\mathbb{E}\left[f(\tilde{x}^s)-f(x)\right]+\frac{1}{2\sum_{t=1}^{T_s}\theta_t}\mathbb{E}\left[\|x^s-x\|^2\right]\right)\\
		\leq&\frac{2}{3L}\left[f(\tilde{x}^{s-1})-f(x)\right]+\frac{1}{2\sum_{t=1}^{T_s}\theta_t}\|x^{s-1}-x\|^2+\frac{\sum_{t=1}^{T_s}\theta_t\eta_{s,t}}{\sum_{t=1}^{T_s}\theta_t}\\
		\end{aligned}
		\end{equation}
		Applying the inequality recursively for $s\geq s_0$, we get
		\begin{equation}	\label{t2beq10}
		\begin{aligned}
		&\mathbb{E}\left[f(\tilde{x}^S)-f(x)\right]+\frac{3L}{4\sum_{t=1}^{T_s}\theta_t}\mathbb{E}\left[\|x^S-x\|^2\right]\\
		\leq&\left(\frac{4}{5}\right)^{S-s_0}\left(\left[f(\tilde{x}^{s_0})-f(x)\right]+\frac{3L}{4\sum_{t=1}^{T_s}\theta_t}\|x^{s_0}-x\|^2\right)+\sum_{k=s_0+1}^S\left(\frac{4}{5}\right)^{S+1-k}\frac{3L\sum_{t=1}^{T_s}\theta_t\eta_{k,t}}{2\sum_{t=1}^{T_s}\theta_t}\\
		\overset{\textrm{\ding{172}}}{\leq}&\left(\frac{4}{5}\right)^{S-s_0}\left(\left[f(\tilde{x}^{s_0})-f(x)\right]+\frac{3L}{4T_{s_0}}\|x^{s_0}-x\|^2\right)+\left(\frac{4}{5}\right)^{S-s_0}\frac{3D_0(\log S + 1)}{2n}D_0D_0
		\end{aligned}
		\end{equation}
		where \ding{172} comes from the choice of $\eta_{s,t}$ that \[\sum_{k=s_0+1}^S\left(\frac{4}{5}\right)^{S+1-k}\frac{3L\sum_{t=1}^{T_s}\theta_t\eta_{k,t}}{2\sum_{t=1}^{T_s}\theta_t}=\sum_{k=s_0+1}^S\left(\frac{4}{5}\right)^{S-s_0}\frac{3}{2kn}D_0\leq\left(\frac{4}{5}\right)^{S-s_0}\frac{3D_0(\log S + 1)}{2n}D_0\] and $\sum_{t=1}^{T_s}\theta_t\geq T_s=T_{s_0}$. From (\ref{t2aeq6}) we have
		\begin{equation}	\label{t2beq9}
		\begin{aligned}
		&\mathbb{E}\left[f(\tilde{x}^{s_0})-f(x^*)\right]+\frac{3L}{4T_{s_0}}\mathbb{E}\|x^{s_0}-x^*\|^2\\ \leq & 2\left(\mathbb{E}\left[f(\tilde{x}^{s_0})-f(x^*)\right]+\frac{3L}{8T_{s_0}}\mathbb{E}\|x^{s_0}-x^*\|^2\right)\leq\frac{3D_0(\log s_0 + 2)}{2^{s_0}}
		\end{aligned}
		\end{equation}
		Plugging (\ref{t2beq9}) into (\ref{t2beq10}), setting $x=x^*$, we get
		\begin{equation}
		\begin{aligned}
		&\mathbb{E}\left[f(\tilde{x}^S)-f(x^*)\right]\\
		\leq&\mathbb{E}\left[f(\tilde{x}^S)-f(x^*)\right]+\frac{3L}{4\sum_{t=1}^{T_s}\theta_t}\mathbb{E}\left[\|x^S-x^*\|^2\right]\\
		\leq&\left(\frac{4}{5}\right)^{S-s_0}\left(\left[f(\tilde{x}^{s_0})-f(x^*)\right]+\frac{3L}{4T_{s_0}}\|x^{s_0}-x^*\|^2\right)+\left(\frac{4}{5}\right)^{S-s_0}\frac{3D_0(\log S + 1)}{2n}D_0\\
		\leq&\left(\frac{4}{5}\right)^{S-s_0}\left(\frac{3D_0(\log s_0 + 2)}{2^{s_0}}+\frac{3D_0(\log S + 1)}{2n}D_0\right)\\ \overset{\textrm{\ding{172}}}{\leq}& \left(\frac{4}{5}\right)^{S-s_0}\left(\frac{3D_0(\log s_0 + 2)}{n}+\frac{3D_0(\log S + 1)}{2n}D_0\right)\leq \left(\frac{4}{5}\right)^{S-s_0}\frac{5D_0(\log S + 2)}{n}
		\end{aligned}
		\end{equation}
		where \ding{172} comes from the fact that $2^{s_0}\geq n$. Then we get the desired result for the first-order case.
		
		\textbf{Zeroth-order Case}: We have $\alpha_s=p_s=\frac{1}{2}, \gamma_s=\frac{2}{5L}, T_s=T_{s_0}=2^{s_0-1}$. Setting $c=\frac{1}{2}$ in Lemma \ref{lemma9}, we have
		\begin{equation}	\label{t2beq1}
		\begin{aligned}
		&\frac{\gamma_s}{\alpha_s}\mathbb{E}\left[f(\bar{x}_t)-f(x)\right]+\frac{1+\frac{\tau\gamma_s}{2}}{2}\mathbb{E}\left[\|x-x_t\|^2\right]\\
		\leq&\frac{\gamma_sp_s}{\alpha_s}\left[f(\tilde{x})-f(x)\right]+\frac{1}{2}\|x-x_{t-1}\|^2+\eta_{s,t} +\frac{\gamma_s\mu^2L^2d}{\tau}+\frac{6\gamma_s^2\mu^2L^2d}{1+\tau\gamma_s-L\alpha_s\gamma_s}\\
		\leq&\frac{\gamma_sp_s}{\alpha_s}\left[f(\tilde{x})-f(x)\right]+\frac{1}{2}\|x-x_{t-1}\|^2+\eta_{s,t} +\frac{\gamma_s\mu^2L^2d}{\tau}+\frac{\gamma_s}{\alpha_s}\cdot\frac{3\mu^2Ld}{2}\\
		\end{aligned}
		\end{equation}
		Multiplying both sides of (\ref{t2beq1}) with $\theta_t=\Gamma_{t-1}=\left(1+\frac{\tau\gamma_s}{2}\right)^{t-1}$, we get
		\begin{equation}	\label{t2beq2}
		\begin{aligned}
		&\frac{\gamma_s}{\alpha_s}\theta_t\mathbb{E}\left[f(\bar{x}_t)-f(x)\right]+\frac{\Gamma_{t}}{2}\mathbb{E}\left[\|x-x_t\|^2\right]\\
		\leq&\frac{\gamma_s}{2\alpha_s}\theta_t\left[f(\tilde{x})-f(x)\right]+\frac{\Gamma_{t-1}}{2}\|x-x_{t-1}\|^2+\theta_t\eta_{s,t} +\theta_t\frac{\gamma_s\mu^2L^2d}{\tau}+\theta_t\frac{\gamma_s}{\alpha_s}\cdot\frac{3\mu^2Ld}{2}\\
		\end{aligned}
		\end{equation}
		Summing up (\ref{t2beq2}) for $t=1, ..., T_s$, we get
		\begin{equation}	\label{t2beq3}
		\begin{aligned}
		&\frac{\gamma_s}{\alpha_s}\sum_{t=1}^{T_s}\theta_t\mathbb{E}\left[f(\bar{x}_t)-f(x)\right]+\frac{\Gamma_{T_s}}{2}\mathbb{E}\left[\|x_{T_s}-x\|^2\right]\\
		\leq&\frac{\gamma_s}{2\alpha_s}\sum_{t=1}^{T_s}\theta_t\left[f(\tilde{x})-f(x)\right]+\frac{1}{2}\|x_{0}-x\|^2+\sum_{t=1}^{T_s}\theta_t\eta_{s,t} +\sum_{t=1}^{T_s}\frac{\theta_t\gamma_s\mu^2L^2d}{\tau}+\sum_{t=1}^{T_s}\theta_t\frac{\gamma_s}{\alpha_s}\cdot\frac{3\mu^2Ld}{2}\\
		\end{aligned}
		\end{equation}
		Since we have
		\begin{equation}	\label{t2beq4}
		\begin{aligned}
		\Gamma_{T_s}=&(1+\frac{\tau\gamma_s}{2})^{T_s}=(1+\frac{\tau\gamma_s}{2})^{T_{s_0}}\geq 1+\frac{\tau\gamma_s}{2}T_{s_0}\geq 1+\frac{\tau\gamma_s}{2}\cdot\frac{n}{2} =1+\frac{\tau n}{5L}\overset{\textrm{\ding{172}}}{\geq}\frac{5}{4}
		\end{aligned}
		\end{equation}
		where \ding{172} comes from the assumption that $n\geq\frac{5L}{4\tau}$. Then we get
		\begin{equation}
		\begin{aligned}
		&\frac{5}{4}\left(\frac{\gamma_s}{2\alpha_s}\sum_{t=1}^{T_s}\theta_t\mathbb{E}\left[f(\bar{x}_t)-f(x)\right]+\frac{1}{2}\mathbb{E}\left[\|x_{T_s}-x\|^2\right]\right)\\
		\leq&\frac{\gamma_s}{\alpha_s}\sum_{t=1}^{T_s}\theta_t\mathbb{E}\left[f(\bar{x}_t)-f(x)\right]+\frac{\Gamma_{T_s}}{2}\mathbb{E}\left[\|x_{T_s}-x\|^2\right]\\
		\leq&\frac{\gamma_s}{2\alpha_s}\sum_{t=1}^{T_s}\theta_t\left[f(\tilde{x})-f(x)\right]+\frac{1}{2}\|x_{0}-x\|^2+\sum_{t=1}^{T_s}\theta_t\eta_{s,t} +\sum_{t=1}^{T_s}\frac{\theta_t\gamma_s\mu^2L^2d}{\tau}+\sum_{t=1}^{T_s}\theta_t\frac{\gamma_s}{\alpha_s}\cdot\frac{3\mu^2Ld}{2}\\
		\end{aligned}
		\end{equation}
		From the definition of $\tilde{x}^s, \tilde{x}^{s-1}, x^s, x^{s-1}$, the fact that $\frac{\gamma_s}{\alpha_s}=\frac{4}{5L}$, the convexity of $f$ and Jensen's inequality, we have
		\begin{equation}
		\begin{aligned}
		&\frac{5}{4}\left(\frac{2}{5L}\mathbb{E}\left[f(\tilde{x}^s)-f(x)\right]+\frac{1}{2\sum_{t=1}^{T_s}\theta_t}\mathbb{E}\left[\|x^{s}-x\|^2\right]\right)\\
		\leq&\frac{2}{5L}\left[f(\tilde{x}^{s-1})-f(x)\right]+\frac{1}{2\sum_{t=1}^{T_s}\theta_t}\|x^{s-1}-x\|^2 +\frac{\sum_{t=1}^{T_s}\theta_t\eta_{s,t}}{\sum_{t=1}^{T_s}\theta_t}+\frac{2\mu^2Ld}{5\tau}+\frac{6\mu^2d}{5}\\
		\end{aligned}
		\end{equation}
		Applying the inequality recursively for $s\geq s_0$, we get
		\begin{equation}	\label{t2beq6}
		\begin{aligned}
		&\mathbb{E}\left[f(\tilde{x}^S)-f(x)\right]+\frac{5L}{\sum_{t=1}^{T_s}\theta_t}\mathbb{E}\left[\|x^{S}-x\|^2\right]\\
		\leq&\left(\frac{4}{5}\right)^{S-s_0}\left(\left[f(\tilde{x}^{s_0})-f(x)\right]+\frac{5L}{4\sum_{t=1}^{T_s}\theta_t}\|x^{s_0}-x\|^2\right)\\ &+\sum_{k=s_0+1}^{S}\left(\frac{4}{5}\right)^{S+1-k}\left(\frac{5L\sum_{t=1}^{T_s}\theta_t\eta_{k,t}}{2\sum_{t=1}^{T_s}\theta_t}+\frac{\mu^2L^2d}{\tau}+3\mu^2Ld\right)\\
		\overset{\textrm{\ding{172}}}{\leq}&\left(\frac{4}{5}\right)^{S-s_0}\left(\left[f(\tilde{x}^{s_0})-f(x)\right]+\frac{5L}{4T_{s_0}}\|x^{s_0}-x\|^2\right)\\&+\left(\frac{4}{5}\right)^{S-s_0}\frac{5(\log S + 1)}{2n}D_0+\frac{4\mu^2L^2d}{\tau}+12\mu^2Ld\\
		\end{aligned}
		\end{equation}
		where \ding{172} comes from the choice of $\eta_{s,t}$ that \[\sum_{k=s_0+1}^{S}\left(\frac{4}{5}\right)^{S+1-k}\frac{5L\sum_{t=1}^{T_s}\theta_t\eta_{k,t}}{2\sum_{t=1}^{T_s}\theta_t} =\sum_{k=s_0+1}^{S}\left(\frac{4}{5}\right)^{S-s_0}\frac{5}{2kn}D_0\leq \left(\frac{4}{5}\right)^{S-s_0}\frac{5(\log S + 1)}{2n}D_0\] and \[\sum_{k=s_0+1}^{S}\left(\frac{4}{5}\right)^{S+1-k}\leq\frac{4}{5}\cdot\frac{1}{1-\frac{4}{5}}=4\] and $\sum_{t=1}^{T_s}\theta_t\geq T_s=T_{s_0}$.
		From (\ref{t2aeq8}) we have
		\begin{equation}	\label{t2beq5}
		\begin{aligned}
		\mathbb{E}\left[f(\tilde{x}^s_0)-f(x^*)\right]+\frac{5L}{4T_{s_0}}\mathbb{E}\left[\|x^s_0-x^*\|^2\right] &\leq 2\left(\mathbb{E}\left[f(\tilde{x}^s_0)-f(x^*)\right]+\frac{5L}{8T_{s_0}}\mathbb{E}\left[\|x^s_0-x^*\|^2\right]\right) \\
		&\leq \frac{5D_0(\log s_0 + 2)}{2^{s_0}}+\frac{\mu^2L^2d}{\tau}+6\mu^2Ld\\
		\end{aligned}
		\end{equation}
		Plugging (\ref{t2beq5}) into (\ref{t2beq6}), setting $x=x^*$, we get
		\begin{equation}
		\begin{aligned}
		\mathbb{E}\left[f(\tilde{x}^S)-f(x^*)\right] \leq&\left(\frac{4}{5}\right)^{S-s_0}\left(\frac{5D_0(\log s_0 + 2)}{2^{s_0}}+\frac{\mu^2L^2d}{\tau}+6\mu^2Ld\right)\\&+\left(\frac{4}{5}\right)^{S-s_0}\frac{5(\log S + 1)}{2n}D_0+\frac{4\mu^2L^2d(d+4)}{\tau}+12\mu^2Ld\\
		\overset{\textrm{\ding{172}}}{\leq}& \left(\frac{4}{5}\right)^{S-s_0}\frac{8D_0(\log S + 2)}{n}+\frac{5\mu^2L^2d}{\tau}+18\mu^2Ld\\
		\end{aligned}
		\end{equation}
		where \ding{172} comes from the fact that $2^{s_0}\geq n$. Then we get the desired result for the zeroth-order case. Then we complete the proof.
	\end{proof}

	\begin{theorem}	\label{theorem2c}
		Suppose Assumption \ref{assum2} holds. Denote $s_0=\lfloor\log n\rfloor+1$. Suppose $s> s_0$, set $\{T_s\}, \{p_s\}, \{\eta_{s,t}\}, \{\theta_t\}$ as \[T_s = T_{s_0} = 2^{s_0-1}, \; p_s = \frac{1}{2}, \; \eta_{s,t}=\frac{\left(\frac{1}{\Gamma_{T_{s_0}}}\right)^{s-s_0-1}D_0}{snL}, \; \theta_t=\begin{cases}
		\Gamma_{t-1}-(1-\alpha_s-p_s)\Gamma_{t}, &t\leq T_s-1\\
		\Gamma_{t-1}, &t=T_{s}\\
		\end{cases}\] where $D_0, \Gamma_{t}$ will be specified below for two cases respectively.
		
		$\bullet$ For the first-order case, set $\alpha_s=\sqrt{\frac{n\tau}{3L}}, \gamma_s=\frac{1}{3L\alpha_s}, \Gamma_{t}=\left(1+\tau\gamma_s\right)^{t}, D_0=4(f(\tilde{x}^0)-f(x^*))+3L\|x^0-x^*\|^2$ for $s>s_0$. Suppose $n<\frac{3L}{4\tau}$,  we have \[\begin{aligned}
		\mathbb{E}\left[f(\tilde{x}^{S})-f(x^*)\right]\leq \left(1+\frac{1}{2}\sqrt{\frac{n\tau}{3L}}\right)^{-(S-s_0)}\frac{5D_0(\log S + 2)}{n}
		\end{aligned}\]
		
		$\bullet$ For the zeroth-order case, set $\alpha_s=\sqrt{\frac{n\tau}{5L}}, \gamma_s=\frac{1}{5L\alpha_s}, \Gamma_{t}=\left(1+\frac{\tau\gamma_s}{2}\right)^{t}, D_0=4(f(\tilde{x}^0)-f(x^*))+5L\|x^0-x^*\|^2$ for $s>s_0$. Suppose $n<\frac{5L}{4\tau}$, we have \[\begin{aligned}
		\mathbb{E}\left[f(\tilde{x}^S)-f(x^*)\right] \leq&\left(1+\frac{1}{4}\sqrt{\frac{n\tau}{5L}}\right)^{-(S-s_0)}\frac{8D_0(\log S + 2)}{n}+\left(\frac{\mu^2L^2d}{\tau}+6\mu^2Ld\right)\left(1+8\sqrt{\frac{5L}{n\tau}}\right)
		\end{aligned}\]
	\end{theorem}
	\begin{proof}
		We give proof to the first-order case and zeroth-order case respectively.
		
		\textbf{First-order Case}: We have $\alpha_s=\sqrt{\frac{n\tau}{3L}}, p_s = \frac{1}{2}, \gamma_s=\frac{1}{\sqrt{3n\tau L}}, T_s=T_{s_0}=2^{s_0-1}$. Note that for the first-order case, we have $f=f$. From Lemma \ref{lemma9}, we have
		\begin{equation}
		\begin{aligned}
		&\frac{\gamma_s}{\alpha_s}\mathbb{E}\left[f(\bar{x}_t)-f(x)\right]+\frac{1+\tau\gamma_s}{2}\mathbb{E}\left[\|x-x_t\|^2\right]\\
		\leq&\frac{\gamma_s(1-\alpha_s-p_s)}{\alpha_s}\left[f(\bar{x}_{t-1})-f(x)\right]+\frac{\gamma_sp_s}{\alpha_s}\left[f(\tilde{x})-f(x)\right]+\frac{1}{2}\|x-x_{t-1}\|^2+\eta_{s,t}
		\end{aligned}
		\end{equation}
		Multiplying both sides of the above inequality with $\Gamma_{t-1}=\left(1+\tau\gamma_s\right)^{t-1}$, we get
		\begin{equation}
		\begin{aligned}
		&\frac{\gamma_s}{\alpha_s}\Gamma_{t-1}\mathbb{E}\left[f(\bar{x}_t)-f(x)\right]+\frac{\Gamma_{t}}{2}\mathbb{E}\left[\|x-x_t\|^2\right]\\
		\leq&\frac{\gamma_s(1-\alpha_s-p_s)}{\alpha_s}\Gamma_{t-1}\left[f(\bar{x}_{t-1})-f(x)\right]+\frac{\gamma_sp_s}{\alpha_s}\Gamma_{t-1}\left[f(\tilde{x})-f(x)\right]+\frac{\Gamma_{t-1}}{2}\|x-x_{t-1}\|^2+\Gamma_{t-1}\eta_{s,t}
		\end{aligned}
		\end{equation}
		Summing up the above inequality for $t=1, ..., T_s$, using the definition of $\theta_t$, we get
		\begin{equation}
		\begin{aligned}
		&\frac{\gamma_s}{\alpha_s}\sum_{t=1}^{T_s}\theta_t\mathbb{E}\left[f(\bar{x}_t)-f(x)\right]+\frac{\Gamma_{T_s}}{2}\mathbb{E}\left[\|x_{T_s}-x\|^2\right]\\
		\leq&\frac{\gamma_s}{\alpha_s}\left[1-\alpha_s-p_s+p_s\sum_{t=1}^{T_s}\Gamma_{t-1}\right]\left[f(\tilde{x})-f(x)\right]+\frac{1}{2}\|x_0-x\|^2+\sum_{t=1}^{T_s}\Gamma_{t-1}\eta_{s,t}
		\end{aligned}
		\end{equation}
		From the definition of $\tilde{x}^s, \tilde{x}^{s-1}, x^s, x^{s-1}$, the convexity of $f$ and Jensen's inequality, we have
		\begin{equation}	\label{t2ceq17}
		\begin{aligned}
		&\frac{\gamma_s}{\alpha_s}\sum_{t=1}^{T_s}\theta_t\mathbb{E}\left[f(\tilde{x}^s)-f(x)\right]+\frac{\Gamma_{T_s}}{2}\mathbb{E}\left[\|x^s-x\|^2\right]\\
		\leq&\frac{\gamma_s}{\alpha_s}\left[1-\alpha_s-p_s+p_s\sum_{t=1}^{T_s}\Gamma_{t-1}\right]\left[f(\tilde{x}^{s-1})-f(x)\right]+\frac{1}{2}\|x^{s-1}-x\|^2+\sum_{t=1}^{T_s}\Gamma_{t-1}\eta_{s,t}
		\end{aligned}
		\end{equation}
		From the definition of $\theta_t$, we have
		\begin{equation}	\label{t2ceq15}
		\begin{aligned}
		\sum_{t=1}^{T_{s_0}}\theta_t=&\Gamma_{T_{s_0}-1}+\sum_{t=1}^{T_{s_0}-1}\left(\Gamma_{t-1}-(1-\alpha_s-p_s)\Gamma_{t}\right)\\
		=&\Gamma_{T_{s_0}}(1-\alpha_s-p_s)+\sum_{t=1}^{T_{s_0}}\left(\Gamma_{t-1}-(1-\alpha_s-p_s)\Gamma_{t}\right)\\
		=&\Gamma_{T_{s_0}}(1-\alpha_s-p_s)+\left[1-(1-\alpha_s-p_s)\left(1+\tau\gamma_s\right)\right]\sum_{t=1}^{T_{s_0}}\Gamma_{t-1}\\
		\end{aligned}
		\end{equation}
		Since $T_{s_0}=2^{s_0-1}\leq n$, we have
		\begin{equation}	\label{t2ceq13}
		\alpha_s=\sqrt{\frac{n\tau}{3L}}\geq\sqrt{\frac{T_{s_0}\tau}{3L}}=\tau\sqrt{\frac{T_{s_0}n}{3n\tau L}}\geq\tau\gamma_s T_{s_0}
		\end{equation}
		Then we have
		\begin{equation}	\label{t2ceq14}
		\begin{aligned}
		1-(1-\alpha_s-p_s)\left(1+\tau\gamma_s\right)=&\left(1+\tau\gamma_s\right)\left(\alpha_s-\tau\gamma_s+p_s\right)+\tau^2\gamma_s^2\\
		\overset{\textrm{\ding{172}}}{\geq}&\left(1+\tau\gamma_s\right)\left(\tau\gamma_sT_{s_0}-\tau\gamma_s+p_s\right)\\
		=&p_s\left(1+\tau\gamma_s\right)\left(2(T_{s_0}-1)\tau\gamma_s+1\right)\\
		\overset{\textrm{\ding{173}}}{\geq}&p_s\left(1+\tau\gamma_s\right)^{T_{s_0}}=p_s\Gamma_{T_s}
		\end{aligned}
		\end{equation}
		where \ding{172} comes from (\ref{t2ceq13}) and \ding{173} comes from the fact that $(1+a)^b\leq 1+2ab$, for $b\geq1, ab\in[0, 1]$. Plugging (\ref{t2ceq14}) into (\ref{t2ceq15}), we get
		\begin{equation}	\label{t2ceq16}
		\begin{aligned}
		\sum_{t=1}^{T_{s_0}}\theta_t\geq\Gamma_{T_{s_0}}\left[1-\alpha_s-p_s+p_s\sum_{t=1}^{T_{s_0}}\Gamma_{t-1}\right]
		\end{aligned}
		\end{equation}
		Plugging (\ref{t2ceq16}) into (\ref{t2ceq17}), we get
		\begin{equation}
		\begin{aligned}
		&\Gamma_{T_{s_0}}\left(\frac{\gamma_s}{\alpha_s}\left[1-\alpha_s-p_s+p_s\sum_{t=1}^{T_{s_0}}\Gamma_{t-1}\right]\mathbb{E}\left[f(\tilde{x}^s)-f(x)\right]+\frac{1}{2}\mathbb{E}\left[\|x^s-x\|^2\right]\right)\\
		\leq&\frac{\gamma_s}{\alpha_s}\left[1-\alpha_s-p_s+p_s\sum_{t=1}^{T_s}\Gamma_{t-1}\right]\left[f(\tilde{x}^{s-1})-f(x)\right]+\frac{1}{2}\|x^{s-1}-x\|^2+\sum_{t=1}^{T_s}\Gamma_{t-1}\eta_{s,t}
		\end{aligned}
		\end{equation}
		Since we have
		\begin{equation}
		\frac{\gamma_s}{\alpha_s}\left[1-\alpha_s-p_s+p_s\sum_{t=1}^{T_{s_0}}\Gamma_{t-1}\right]\geq\frac{\gamma_sp_s}{\alpha_s}\sum_{t=1}^{T_{s_0}}\Gamma_{t-1}\geq\frac{\gamma_sp_s}{\alpha_s}T_{s_0}
		\end{equation}
		Dividing both sides of the above inequality with $\Gamma_{T_{s_0}}\frac{\gamma_s}{\alpha_s}\left[1-\alpha_s-p_s+p_s\sum_{t=1}^{T_{s_0}}\Gamma_{t-1}\right]$, we have
		\begin{equation}
		\begin{aligned}
		&\mathbb{E}\left[f(\tilde{x}^s)-f(x)\right]+\frac{\alpha_s}{\gamma_sT_{s_0}}\mathbb{E}\left[\|x^s-x\|^2\right]\\
		\leq&\left(\frac{1}{\Gamma_{T_{s_0}}}\right)\left(\left[f(\tilde{x}^s)-f(x)\right]+\frac{\alpha_s}{\gamma_sT_{s_0}}\left[\|x^s-x\|^2\right]\right)+\left(\frac{1}{\Gamma_{T_{s_0}}}\right)\frac{2\alpha_s\sum_{t=1}^{T_s}\Gamma_{t-1}\eta_{s,t}}{\gamma_s\sum_{t=1}^{T_s}\Gamma_{t-1}}
		\end{aligned}
		\end{equation}
		Summing up the above inequality for $s>s_0$, setting $x=x^*$, we get
		\begin{equation}	\label{t2ceq20}
		\begin{aligned}
		&\mathbb{E}\left[f(\tilde{x}^S)-f(x^*)\right]+\frac{\alpha_S}{\gamma_ST_{s_0}}\mathbb{E}\left[\|x^S-x^*\|^2\right]\\
		\leq&\left(\frac{1}{\Gamma_{T_{s_0}}}\right)^{S-s_0}\left(\left[f(\tilde{x}^{s_0})-f(x^*)\right]+\frac{\alpha_{s_0}}{\gamma_{s_0}T_{s_0}}\left[\|x^{s_0}-x^*\|^2\right]\right)+\sum_{k=s_0+1}^S\left(\frac{1}{\Gamma_{T_{s_0}}}\right)^{S+1-k}\frac{2\alpha_s\sum_{t=1}^{T_s}\Gamma_{t-1}\eta_{k,t}}{\gamma_s\sum_{t=1}^{T_s}\Gamma_{t-1}}\\
		\overset{\textrm{\ding{172}}}{\leq}&\left(\frac{1}{\Gamma_{T_{s_0}}}\right)^{S-s_0}\left(\left[f(\tilde{x}^{s_0})-f(x^*)\right]+\frac{3L}{4T_{s_0}}\left[\|x^{s_0}-x^*\|^2\right]\right)+\sum_{k=s_0+1}^S\left(\frac{1}{\Gamma_{T_{s_0}}}\right)^{S+1-k}\frac{2\alpha_s\sum_{t=1}^{T_s}\Gamma_{t-1}\eta_{k,t}}{\gamma_s\sum_{t=1}^{T_s}\Gamma_{t-1}}
		\end{aligned}
		\end{equation}
		where \ding{172} comes from the fact that $\frac{\alpha_s}{\gamma_s}=n\tau\leq\frac{3L}{4}$. From Theorem \ref{theorem2a} we have
		\begin{equation}	\label{t2ceq19}
		\begin{aligned}
		&\mathbb{E}\left[f(\tilde{x}^{s_0})-f(x^*)\right]+\frac{3L}{4T_{s_0}}\mathbb{E}\left[\|x^{s_0}-x^*\|^2\right]\\
		\leq&2\left(\mathbb{E}\left[f(\tilde{x}^{s_0})-f(x^*)\right]+\frac{3L}{8T_{s_0}}\mathbb{E}\left[\|x^{s_0}-x^*\|^2\right]\right) \overset{\textrm{\ding{172}}}{\leq}\frac{3D_0(\log s_0 + 2)}{2^{s_0}}
		\end{aligned}
		\end{equation}
		where \ding{172} comes from (\ref{t2aeq9}). With the choice of $\eta_{s,t}$, we have
		\begin{equation}	\label{t2ceq18}
		\begin{aligned}
			\sum_{k=s_0+1}^S\left(\frac{1}{\Gamma_{T_{s_0}}}\right)^{S+1-k}\frac{2\alpha_s\sum_{t=1}^{T_s}\Gamma_{t-1}\eta_{k,t}}{\gamma_s\sum_{t=1}^{T_s}\Gamma_{t-1}}=&\sum_{k=s_0+1}^S\left(\frac{1}{\Gamma_{T_{s_0}}}\right)^{S-s_0}\frac{2}{kn}D_0\\\leq&\left(\frac{1}{\Gamma_{T_{s_0}}}\right)^{S-s_0}\frac{2(\log S + 1)}{n}D_0
		\end{aligned}
		\end{equation}
		Plugging (\ref{t2ceq18}), (\ref{t2ceq19}) into (\ref{t2ceq20}), we get
		\begin{equation}	\label{t2ceq21}
		\begin{aligned}
		\mathbb{E}\left[f(\tilde{x}^S)-f(x^*)\right] \leq&\left(\frac{1}{\Gamma_{T_{s_0}}}\right)^{S-s_0}\frac{3D_0(\log s_0 + 2)}{2^{s_0}}+\left(\frac{1}{\Gamma_{T_{s_0}}}\right)^{S-s_0}\frac{2(\log S + 1)}{n}D_0\\
		\overset{\textrm{\ding{172}}}{\leq}&\left(\frac{1}{\Gamma_{T_{s_0}}}\right)^{S-s_0}\frac{5D_0(\log S + 2)}{n}
		\end{aligned}
		\end{equation}
		where \ding{172} comes from $2^{s_0}\geq n$. From the definition of $\Gamma_{T_{s_0}}$, we have
		\begin{equation}	\label{t2ceq22}
		\Gamma_{T_{s_0}}=\left(1+\tau\gamma_s\right)^{T_{s_0}}\geq1+\tau\gamma_sT_{s_0}\geq1+\frac{\tau\gamma_s n}{2}=1+\frac{1}{2}\sqrt{\frac{n\tau}{3L}}
		\end{equation}
		Plugging (\ref{t2ceq22}) into (\ref{t2ceq21}), we get
		\begin{equation}
		\begin{aligned}
		&\mathbb{E}\left[f(\tilde{x}^S)-f(x^*)\right]
		\leq\left(1+\frac{1}{2}\sqrt{\frac{n\tau}{3L}}\right)^{-(S-s_0)}\frac{5D_0(\log S + 2)}{n}
		\end{aligned}
		\end{equation}
		Then we get the desired result for the first-order case.
		
		\textbf{Zeroth-order Case}: We have $\alpha_s=\sqrt{\frac{n\tau}{5L}}, p_s=\frac{1}{2}, \gamma_s=\frac{1}{\sqrt{5n\tau L}}, T_s=T_{s_0}=2^{s_0-1}$. Setting $c=\frac{1}{2}$ in Lemma \ref{lemma9}, we have
		\begin{equation}	\label{t2ceq1}
		\begin{aligned}
		&\frac{\gamma_s}{\alpha_s}\mathbb{E}\left[f(\bar{x}_t)-f(x)\right]+\frac{1+\frac{\tau\gamma_s}{2}}{2}\mathbb{E}\left[\|x-x_t\|^2\right]\\
		\leq&\frac{\gamma_s(1-\alpha_s-p_s)}{\alpha_s}\left[f(\bar{x}_{t-1})-f(x)\right]+\frac{\gamma_sp_s}{\alpha_s}\left[f(\tilde{x})-f(x)\right]+\frac{1}{2}\|x-x_{t-1}\|^2+\eta_{s,t}\\
		&+\frac{\gamma_s\mu^2L^2d}{\tau}+\frac{6\gamma_s^2\mu^2L^2d}{1+\tau\gamma_s-L\alpha_s\gamma_s}\\
		\leq&\frac{\gamma_s(1-\alpha_s-p_s)}{\alpha_s}\left[f(\bar{x}_{t-1})-f(x)\right]+\frac{\gamma_sp_s}{\alpha_s}\left[f(\tilde{x})-f(x)\right]+\frac{1}{2}\|x-x_{t-1}\|^2+\eta_{s,t}\\
		&+\frac{\gamma_s\mu^2L^2d}{\tau}+\frac{\gamma_s}{\alpha_s}\cdot\frac{3\mu^2Ld}{2}\\
		\end{aligned}
		\end{equation}
		Multiplying both sides of the inequality with $\Gamma_{t-1}=\left(1+\frac{\tau\gamma_s}{2}\right)^{t-1}$, we get
		\begin{equation}
		\begin{aligned}
		&\frac{\gamma_s}{\alpha_s}\Gamma_{t-1}\mathbb{E}\left[f(\bar{x}_t)-f(x)\right]+\frac{\Gamma_{t}}{2}\mathbb{E}\left[\|x-x_t\|^2\right]\\
		\leq&\frac{\gamma_s(1-\alpha_s-p_s)}{\alpha_s}\Gamma_{t-1}\left[f(\bar{x}_{t-1})-f(x)\right]+\frac{\gamma_sp_s}{\alpha_s}\Gamma_{t-1}\left[f(\tilde{x})-f(x)\right]+\frac{\Gamma_{t-1}}{2}\|x-x_{t-1}\|^2+\Gamma_{t-1}\eta_{s,t}\\
		&+\Gamma_{t-1}\frac{\gamma_s\mu^2L^2d}{\tau}+\Gamma_{t-1}\frac{\gamma_s}{\alpha_s}\cdot\frac{3\mu^2Ld}{2}\\
		\end{aligned}
		\end{equation}
		Summing up the inequality for $t=1, ..., T_s$, using the definition of $\theta_t$, we get
		\begin{equation}
		\begin{aligned}
		&\frac{\gamma_s}{\alpha_s}\sum_{t=1}^{T_s}\theta_{t}\mathbb{E}\left[f(\bar{x}_t)-f(x)\right]+\frac{\Gamma_{T_s}}{2}\mathbb{E}\left[\|x_{T_s}-x\|^2\right]\\
		\leq&\frac{\gamma_s}{\alpha_s}\left(1-\alpha_s-p_s+p_s\sum_{t=1}^{T_s}\Gamma_{t-1}\right)\left[f(\tilde{x})-f(x)\right]+\frac{1}{2}\|x_0-x\|^2+\sum_{t=1}^{T_s}\Gamma_{t-1}\eta_{s,t}\\
		&+\sum_{t=1}^{T_s}\Gamma_{t-1}\frac{\gamma_s\mu^2L^2d}{\tau}+\sum_{t=1}^{T_s}\Gamma_{t-1}\frac{\gamma_s}{\alpha_s}\cdot\frac{3\mu^2Ld}{2}\\
		\end{aligned}
		\end{equation}
		From the definition of $\tilde{x}^s, \tilde{x}^{s-1}, x^s, x^{s-1}$ and the convexity of $f$, we have
		\begin{equation}	\label{t2ceq6}
		\begin{aligned}
		&\frac{\gamma_s}{\alpha_s}\sum_{t=1}^{T_s}\theta_{t}\mathbb{E}\left[f(\tilde{x}^s)-f(x)\right]+\frac{\Gamma_{T_s}}{2}\mathbb{E}\left[\|x^s-x\|^2\right]\\
		\leq&\frac{\gamma_s}{\alpha_s}\left(1-\alpha_s-p_s+p_s\sum_{t=1}^{T_s}\Gamma_{t-1}\right)\left[f(\tilde{x}^{s-1})-f(x)\right]+\frac{1}{2}\|x^{s-1}-x\|^2+\sum_{t=1}^{T_s}\Gamma_{t-1}\eta_{s,t}\\
		&+\sum_{t=1}^{T_s}\Gamma_{t-1}\frac{\gamma_s\mu^2L^2d}{\tau}+\sum_{t=1}^{T_s}\Gamma_{t-1}\frac{\gamma_s}{\alpha_s}\cdot\frac{3\mu^2Ld}{2}\\
		\end{aligned}
		\end{equation}
		From the definition of $\theta_t$, we have
		\begin{equation}	\label{t2ceq4}
		\begin{aligned}
		\sum_{t=1}^{T_{s_0}}\theta_t=&\Gamma_{T_{s_0}-1}+\sum_{t=1}^{T_{s_0}-1}\left(\Gamma_{t-1}-(1-\alpha_s-p_s)\Gamma_{t}\right)\\
		=&\Gamma_{T_{s_0}}(1-\alpha_s-p_s)+\sum_{t=1}^{T_{s_0}}\left(\Gamma_{t-1}-(1-\alpha_s-p_s)\Gamma_{t}\right)\\
		=&\Gamma_{T_{s_0}}(1-\alpha_s-p_s)+\left[1-(1-\alpha_s-p_s)\left(1+\frac{\tau\gamma_s}{2}\right)\right]\sum_{t=1}^{T_{s_0}}\Gamma_{t-1}\\
		\end{aligned}
		\end{equation}
		Since $T_{s_0}=2^{s_0-1}\leq n$, we have
		\begin{equation}	\label{t2ceq2}
		\begin{aligned}
		\alpha_s=\sqrt{\frac{n\tau}{5L}}\geq\sqrt{\frac{T_{s_0}\tau}{5L}}=\frac{1}{\sqrt{5n\tau L}}\cdot\tau\sqrt{T_{s_0}n}\geq\tau\gamma_sT_{s_0}\geq\frac{\tau\gamma_s T_{s_0}}{2}
		\end{aligned}
		\end{equation}
		Then we have
		\begin{equation}	\label{t2ceq3}
		\begin{aligned}
		1-(1-\alpha_s-p_s)\left(1+\frac{\tau\gamma_s}{2}\right)=&\left(1+\frac{\tau\gamma_s}{2}\right)\left(\alpha_s+p_s-\frac{\tau\gamma_s}{2}\right)+\frac{\tau^2\gamma_s^2}{4}\\
		\overset{\textrm{\ding{172}}}{\geq}&\left(1+\frac{\tau\gamma_s}{2}\right)\left(\frac{\tau\gamma_s}{2}T_{s_0}+p_s-\frac{\tau\gamma_s}{2}\right)\\
		=&p_s\left(1+\frac{\tau\gamma_s}{2}\right)\left(1+2(T_{s_0}-1)\frac{\tau\gamma_s}{2}\right)\\
		\overset{\textrm{\ding{173}}}{\geq}&p_s\left(1+\frac{\tau\gamma_s}{2}\right)^{T_{s_0}}=p_s\Gamma_{T_{s_0}}
		\end{aligned}
		\end{equation}
		where \ding{172} comes from (\ref{t2ceq2}) and \ding{173} comes from the fact that $(1+a)^b\leq1+2ab$, for $b\geq1, ab\in[0, 1]$. Plugging (\ref{t2ceq3}) into (\ref{t2ceq4}), we get
		\begin{equation}	\label{t2ceq5}
		\sum_{t=1}^{T_{s_0}}\theta_t\geq\Gamma_{T_{s_0}}\left[1-\alpha_s-p_s+p_s\sum_{t=1}^{T_{s_0}}\Gamma_{t-1}\right]
		\end{equation}
		Then plugging (\ref{t2ceq5}) into (\ref{t2ceq6}), setting $x=x^*$, we get
		\begin{equation}
		\begin{aligned}
		&\Gamma_{T_s}\left(\frac{\gamma_s}{\alpha_s}\left[1-\alpha_s-p_s+p_s\sum_{t=1}^{T_{s}}\Gamma_{t-1}\right]\mathbb{E}\left[f(\tilde{x}^s)-f(x^*)\right]+\frac{1}{2}\mathbb{E}\left[\|x^s-x^*\|^2\right]\right)\\
		\leq&\frac{\gamma_s}{\alpha_s}\left[1-\alpha_s-p_s+p_s\sum_{t=1}^{T_{s}}\Gamma_{t-1}\right]\left[f(\tilde{x}^{s-1})-f(x^*)\right]+\frac{1}{2}\|x^{s-1}-x^*\|^2+\sum_{t=1}^{T_s}\Gamma_{t-1}\eta_{s,t}\\
		&+\sum_{t=1}^{T_s}\Gamma_{t-1}\frac{\gamma_s\mu^2L^2d}{\tau}+\sum_{t=1}^{T_s}\Gamma_{t-1}\frac{\gamma_s}{\alpha_s}\cdot\frac{3\mu^2Ld}{2}\\
		\end{aligned}
		\end{equation}
		Denote $\alpha=\alpha_s=\sqrt{\frac{n\tau}{5L}}, p=p_s=\frac{1}{2}, \gamma=\gamma_s=\frac{1}{\sqrt{5n\tau L}}$ Rearranging the terms and summing up the above inequality for $s>s_0$, we get
		\begin{equation}	\label{t2ceq8}
		\begin{aligned}
		&\frac{\gamma}{\alpha}\left[1-\alpha-p+p\sum_{t=1}^{T_{s_0}}\Gamma_{t-1}\right]\mathbb{E}\left[f(\tilde{x}^S)-f(x^*)\right]+\frac{1}{2}\mathbb{E}\left[\|x^S-x^*\|^2\right]\\
		\leq&\left(\frac{1}{\Gamma_{T_s}}\right)^{S-s_0}\left(\frac{\gamma}{\alpha}\left[1-\alpha-p+p\sum_{t=1}^{T_{s_0}}\Gamma_{t-1}\right]\left[f(\tilde{x}^{s_0})-f(x^*)\right]+\frac{1}{2}\|x^{s_0}-x^*\|^2\right)\\
		&+\sum_{k=s_0+1}^S\left(\frac{1}{\Gamma_{T_s}}\right)^{S+1-k}\left(\sum_{t=1}^{T_s}\Gamma_{t-1}\eta_{k,t} +\sum_{t=1}^{T_s}\Gamma_{t-1}\frac{\gamma\mu^2L^2d}{\tau}+\sum_{t=1}^{T_s}\Gamma_{t-1}\frac{\gamma}{\alpha}\cdot\frac{3\mu^2Ld}{2}\right)\\
		\end{aligned}
		\end{equation}
		Since we have
		\begin{equation}	\label{t2ceq7}
		\frac{\gamma}{\alpha}\left[1-\alpha-p+p\sum_{t=1}^{T_{s_0}}\Gamma_{t-1}\right]\geq\frac{\gamma p}{\alpha}\sum_{t=1}^{T_{s_0}}\Gamma_{t-1}\geq\frac{\gamma pT_{s_0}}{\alpha}
		\end{equation}
		Plugging (\ref{t2ceq7}) into (\ref{t2ceq8}), we get
		\begin{equation}	\label{t2ceq9}
		\begin{aligned}
		\mathbb{E}\left[f(\tilde{x}^S)-f(x^*)\right] \leq&\left(\frac{1}{\Gamma_{T_s}}\right)^{S-s_0}\left[\left[f(\tilde{x}^{s_0})-f(x^*)\right]+\frac{\alpha}{\gamma T_{s_0}}\|x^{s_0}-x^*\|^2\right]\\
		&+\sum_{k=s_0+1}^S\left(\frac{1}{\Gamma_{T_s}}\right)^{S+1-k}\left[\frac{2\alpha\sum_{t=1}^{T_s}\Gamma_{t-1}\eta_{k,t}}{\gamma\sum_{t=1}^{T_s}\Gamma_{t-1}} +\frac{2\alpha\mu^2L^2d}{\tau}+3\mu^2Ld\right]\\
		\overset{\textrm{\ding{172}}}{\leq}&\left(\frac{1}{\Gamma_{T_s}}\right)^{S-s_0}\left[\left[f(\tilde{x}^{s_0})-f(x^*)\right]+\frac{\alpha}{\gamma T_{s_0}}\|x^{s_0}-x^*\|^2\right]\\
		&+\left(\frac{1}{\Gamma_{T_s}}\right)^{S-s_0}\frac{3(\log S + 1)}{n}D_0 +\frac{1}{\Gamma_{T_{s_0}}-1}\left[\frac{2\alpha\mu^2L^2d}{\tau}+3\mu^2Ld\right]\\
		\end{aligned}
		\end{equation}
		where \ding{172} comes from the choice of $\eta_{s,t}$ that \[\sum_{k=s_0+1}^S\left(\frac{1}{\Gamma_{T_s}}\right)^{S+1-k}\frac{2\alpha\sum_{t=1}^{T_s}\Gamma_{t-1}\eta_{k,t}}{\gamma\sum_{t=1}^{T_s}\Gamma_{t-1}} =\sum_{k=s_0+1}^{S}\left(\frac{1}{\Gamma_{T_s}}\right)^{S-s_0}\frac{3}{kn}D_0\leq \left(\frac{1}{\Gamma_{T_s}}\right)^{S-s_0}\frac{3(\log S + 1)}{n}D_0\] and \[\sum_{k=s_0+1}^S\left(\frac{1}{\Gamma_{T_s}}\right)^{S+1-k}\leq \frac{1}{\Gamma_{T_{s_0}}-1}\]
		From (\ref{t2aeq8}) we know
		\begin{equation}
		\begin{aligned}
		\mathbb{E}\left[f(\tilde{x}^{s_0})-f(x^*)\right]+\frac{5L}{8T_{s_0}}\mathbb{E}\left[\|x^{s_0}-x^*\|^2\right] \leq \frac{5D_0(\log s_0 + 2)}{2^{s_0+1}}+\frac{\mu^2L^2d}{2\tau}+3\mu^2Ld\\
		\end{aligned}
		\end{equation}
		From the definition of $\alpha, \gamma$ and the assumption that $n<\frac{5L}{4\tau}$, we have $\frac{\alpha}{\gamma}=5L\alpha^2=n\tau\leq \frac{5L}{4}$. Thus we get
		\begin{equation}
		\begin{aligned}	\label{t2ceq10}
		&\mathbb{E}\left[f(\tilde{x}^{s_0})-f(x^*)\right]+\frac{\alpha}{\gamma T_{s_0}}\mathbb{E}\left[\|x^{s_0}-x^*\|^2\right]\\
		\leq&2\left(\mathbb{E}\left[f(\tilde{x}^{s_0})-f(x^*)\right]+\frac{5L}{8T_s}\mathbb{E}\left[\|x^{s_0}-x^*\|^2\right]\right)\\
		\leq&\frac{5D_0(\log s_0 + 2)}{2^{s_0}}+\frac{\mu^2L^2d}{\tau}+6\mu^2Ld\overset{\textrm{\ding{172}}}{\leq}\frac{5D_0(\log s_0 + 2)}{n}+\frac{\mu^2L^2d}{\tau}+6\mu^2Ld\\
		\end{aligned}
		\end{equation}
		where \ding{172} holds since $2^{s_0}\geq n$. Plugging (\ref{t2ceq10}) into (\ref{t2ceq9}), we get
		\begin{equation}	\label{t2ceq11}
		\begin{aligned}
		&\mathbb{E}\left[f(\tilde{x}^S)-f(x^*)\right]\\
		\leq&\left(\frac{1}{\Gamma_{T_s}}\right)^{S-s_0}\left[\frac{8D_0(\log S + 2)}{n}+\frac{\mu^2L^2d}{\tau}+6\mu^2Ld\right] +\frac{1}{\Gamma_{T_{s_0}}-1}\left[\frac{2\alpha\mu^2L^2d}{\tau}+3\mu^2Ld\right]\\
		\end{aligned}
		\end{equation}
		From the definition of $\Gamma_{T_{s_0}}$, we have
		\begin{equation}
		\Gamma_{T_{s_0}}=\left(1+\frac{\tau\gamma}{2}\right)^{T_{s_0}}\geq1+\frac{\tau\gamma}{2}T_{s_0}\geq1+\frac{\tau\gamma n}{4}\overset{\textrm{\ding{172}}}{=}1+\frac{1}{4}\sqrt{\frac{n\tau}{5L}}
		\end{equation}
		where \ding{172} comes from the definition of $\gamma$. Then we have
		\begin{equation}	\label{t2ceq12}
		\frac{1}{\Gamma_{T_{s_0}}-1}\leq 4\sqrt{\frac{5L}{n\tau}}, \; \textrm{ and } \left(\frac{1}{\Gamma_{T_s}}\right)^{S-s_0}\leq \left(1+\frac{1}{4}\sqrt{\frac{n\tau}{5L}}\right)^{-(S-s_0)}
		\end{equation}
		Plugging (\ref{t2ceq12}) into (\ref{t2ceq11}), using the definition of $\alpha, \gamma$, we get
		\begin{equation}
		\begin{aligned}
		&\mathbb{E}\left[f(\tilde{x}^S)-f(x^*)\right]\\
		\leq&\left(1+\frac{1}{4}\sqrt{\frac{n\tau}{5L}}\right)^{-(S-s_0)}\frac{8D_0(\log S + 2)}{n}+\left(\frac{\mu^2L^2d}{\tau}+6\mu^2Ld\right)\left(1+8\sqrt{\frac{5L}{n\tau}}\right)\\
		\end{aligned}
		\end{equation}
		Then we get the desired result for the zeroth-order case. Then we complete the proof.
	\end{proof}

\section{Auxillary Lemmas}	\label{app_d}

\begin{lemma}[Coordinate-wise Gradient Estimator]	\label{lemma7}
	For all $x\in\mathcal{C}$, we have
	\[\|\hat{\nabla}_{coord}f(x)-\nabla f(x)\|^2\leq \mu^2L^2d\]
\end{lemma}
\begin{proof}
	See [\cite{ji2019improved}, Appendix, Lemma 3].
\end{proof}

\begin{lemma}	\label{lemma6}
	Suppose each $f_{i\in[n]}$ is $L$-smooth, for any $x, y\in\mathcal{C}$, we have \[\mathbb{E}\left[\|\nabla f_i(x)-\nabla f_i(y)\|^2\right]\leq 2L\left(f(x)-f(y)-\langle\nabla f(y), x-y\rangle\right)\]
\end{lemma}
\begin{proof}
	Denote $\phi_i(x)= f_i(x)-f(y)-\langle\nabla f(y), x-y\rangle$. It is easy to verify that $\phi_i$ is also $L$-smooth. Clearly $\nabla \phi_i(y) = 0$ and hence $\min_{x\in\mathcal{C}}\phi_i(x) = \phi_i(y) = 0$. Then for $\alpha\in\mathbb{R}$, we have
	\begin{equation}
	\begin{aligned}
	\phi_i(y)&\leq\min_\alpha\left\{\phi_i(x-\alpha \nabla \phi_i(x))\right\}\\
	&\overset{\textrm{\ding{172}}}{\leq} \min_\alpha\left\{\phi_i(x)-\alpha\|\nabla \phi_i(x)\|^2+\frac{L\alpha^2}{2}\|\nabla \phi_i(x)\|^2\right\} = \phi_i(x)-\frac{1}{2L}\|\nabla \phi_i(x)\|^2
	\end{aligned}
	\end{equation}
	where \ding{172} comes from the smoothness of $\phi_i$. Rearranging the terms and using the definition of $\phi_i$ we get
	\begin{equation}
	\|\nabla f_i(x)-\nabla f_i(y)\|^2\leq 2L\left(f_i(x)-f_i(y)-\langle\nabla f_i(y), x-y\rangle\right)
	\end{equation}
	Taking expectation with respect to $i$, we get
	\begin{equation}
	\mathbb{E}\left[\|\nabla f_i(x)-\nabla f_i(y)\|^2\right]\leq 2L\left(f(x)-f(y)-\langle\nabla f(y), x-y\rangle\right)
	\end{equation}
	Then we complete the proof.
\end{proof}

\begin{lemma}	\label{lemma4}
	Suppose each $f_{i\in[n]}$ is $L$-smooth. Conditioning on $x_1, ..., x_{t-1}$
	
	$\\\bullet$ For the first-order case, we have \[\mathbb{E}\left[\delta_t\right] = 0\] and \[\mathbb{E}\left[\|\delta_t-\mathbb{E}\left[\delta_t\right]\|^2\right]\leq 2L\left[f(\tilde{x})-f(\underline{x}_t)-\langle\nabla f(\underline{x}_t), \tilde{x}-\underline{x}_t\rangle\right]\]
	$\bullet$ For the zeroth-order case, we have
	\[\mathbb{E}\left[\delta_t\right] = \hat{\nabla}_{coord}f(\underline{x}_t)-\nabla f(\underline{x}_t) \neq 0\] and
	\[\mathbb{E}\left[\|\delta_t-\mathbb{E}\left[\delta_t\right]\|^2\right]\leq 8L\left(f(\tilde{x})-f(\underline{x}_t)-\langle\nabla f(\underline{x}_t), \tilde{x}-\underline{x}_t\rangle\right) + 12\mu^2L^2d\]
	where the expectation is taken with respect to all variables.
\end{lemma}
\begin{proof}
	The part for the first-order case is proved in [\cite{lan2019unified}, Lemma 3]. Now we give a proof to the zeroth-order case. For the zeroth-order case, we have
	\begin{equation}
	\begin{aligned}
	\mathbb{E}\left[\delta_t\right] &= \mathbb{E}\left[\hat{\nabla}_{coord}f_{i_t}(\underline{x}_t)-\hat{\nabla}_{coord}f_{i_t}(\tilde{x})+\tilde{g}-\nabla f(\underline{x}_t)\right]\\
	&=\mathbb{E}\left[\hat{\nabla}_{coord}f_{i_t}(\underline{x}_t)-\nabla f(\underline{x}_t)\right] =\hat{\nabla}_{coord}f(\underline{x}_t)-\nabla f(\underline{x}_t)
	\end{aligned}
	\end{equation}
	Then we prove the upper bound of $\mathbb{E}\left[\|\delta_t-\mathbb{E}\left[\delta_t\right]\|^2\right]$. We have
	\begin{equation}	\label{l4eq1}
	\begin{aligned}
	&\mathbb{E}\left[\|\delta_t-\mathbb{E}\left[\delta_t\right]\|^2\right] \overset{\textrm{\ding{172}}}{\leq}\mathbb{E}\left[\|\delta_t\|^2\right]\\
	=& \mathbb{E}\left[\|\left(\nabla f_{i_t}\left(\underline{x}_t\right)-\nabla f_{i_t}(\tilde{x})-\left[\nabla f(\underline{x}_t) - \nabla f(\tilde{x})\right]\right) + \left(\hat{\nabla}_{coord}f_{i_t}(\underline{x}_t)-\nabla f_{i_t}(\underline{x}_t)\right)\right. \\ &\left.- \left(\hat{\nabla}_{coord}f_{i_t}(\tilde{x})-\nabla f_{i_t}(\tilde{x})\right) + \left(\hat{\nabla}_{coord}f(\tilde{x})-\nabla f(\tilde{x})\right)\|^2\right] \\
	\overset{\textrm{\ding{173}}}{\leq}& 4\mathbb{E}\left[\|\nabla f_{i_t}\left(\underline{x}_t\right)-\nabla f_{i_t}(\tilde{x})-\left[\nabla f(\underline{x}_t) - \nabla f(\tilde{x})\right]\|^2 + \|\hat{\nabla}_{coord}f_{i_t}(\underline{x}_t)-\nabla f_{i_t}(\underline{x}_t)\|^2\right. \\ &\left.+ \|\hat{\nabla}_{coord}f_{i_t}(\tilde{x})-\nabla f_{i_t}(\tilde{x})\|^2 + \|\hat{\nabla}_{coord}f(\tilde{x})-\nabla f(\tilde{x})\|^2\right] \\
	\overset{\textrm{\ding{174}}}{\leq}& 4\mathbb{E}\left[\|\nabla f_{i_t}\left(\underline{x}_t\right)-\nabla f_{i_t}(\tilde{x})\|^2\right] + 12\mu^2L^2d \\
	\overset{\textrm{\ding{175}}}{\leq}& 8L\left(f(\tilde{x})-f(\underline{x}_t)-\langle\nabla f(\underline{x}_t), \tilde{x}-\underline{x}_t\rangle\right) + 12\mu^2L^2d
	\end{aligned}
	\end{equation}
	where \ding{172} comes from $\mathbb{E}\left[\|x-\mathbb{E}\left[x\right]\|^2\right]=\mathbb{E}\left[\|x\|^2\right]-\mathbb{E}\left[x\right]^2\leq\mathbb{E}\left[\|x\|^2\right]$, \ding{173} comes from the Cauchy-Schwarz inequality, \ding{174} comes from $\mathbb{E}\left[\|x-\mathbb{E}\left[x\right]\|^2\right]\leq\mathbb{E}\left[\|x\|^2\right]$ and Lemma \ref{lemma7}, \ding{175} comes from Lemma \ref{lemma6}. Then we complete the proof.
\end{proof}

\section{The STORC Algorithm}	\label{app_e}

In this section. we include the STORC algorithm proposed by \cite{hazan2016variance} and its key theorems for completeness.

\numberwithin{algorithm}{section}

\begin{algorithm}
	\caption{STOchastic variance-Reduced Conditional gradient sliding (STORC)}
	\label{storc}
	\begin{algorithmic}[1]
		\STATE {\bfseries Input:} $x_0\in\mathcal{C}, \{T_s\}, \{\gamma_{s,t}\}, \{\alpha_{s,t}\}, \{\eta_{s,t}\}$
		\STATE Set $\tilde{x}^0 = x^0$.
		\FOR {$s=1, 2, ...$}
		\STATE Set $x_0 = \bar{x}_0 = \tilde{x} = \tilde{x}^{s-1}$ and $\tilde{g}=\nabla f(\tilde{x})$
		\STATE Set $T=T_s$.
		\FOR {$t = 1, ..., T$}
		\STATE Pick $\mathcal{I}_t\subset\{1, ..., n\}$ randomly with $|\mathcal{I}_t| = m_{s,t}$
		\STATE Set $\underline{x}_t = (1-\alpha_{s,t})\bar{x}_{t-1}+\alpha_{s,t}x_{t-1}$
		\STATE Set $G_t = \frac{1}{m_{s,t}}\sum_{i\in\mathcal{I}_{t}}\left[\nabla f_{i}(\underline{x}_t)-\nabla f_{i}(\tilde{x})+\tilde{g}\right]$
		\STATE $x_t=\textrm{CondG}(G_t, x_{t-1}, 0, \gamma_{s,t}, 0, \eta_{s,t})$ \label{condg} \hspace{60pt}\textit{// Algorithm \ref{algo2}}
		\STATE $\bar{x}_t = (1-\alpha_{s,t})\bar{x}_{t-1}+\alpha_{s,t}x_t$.
		\ENDFOR
		\STATE Set $\tilde{x}^s = \bar{x}_t$.
		\ENDFOR
	\end{algorithmic}
\end{algorithm}

%\begin{algorithm}
%	\caption{CondG for STORC}
%	\label{algoCondg}
%	\begin{algorithmic}[1]
%		\STATE {\bfseries Input:} $(g, u, \gamma, \eta)$
%		\STATE Set $u_1=u$.
%		\FOR {$t=1, 2, ...$}
%		\STATE Let $v_t$ be an optimal solution for the subproblem of \begin{equation}
%			V_{g, u, \gamma}(u_t) = \max_{x\in\mathcal{C}}\langle g+\frac{u_t-u}{\gamma}, u_t-x\rangle
%		\end{equation}
%		\IF {$V_{g, u, \gamma}(u_t)\leq\eta$}
%		\STATE {\bfseries Return} $u^+=u_t$.
%		\ELSE
%		\STATE Set $u_{t+1}=(1-\alpha_t)u_t+\alpha_tv_t$ with \[\alpha_t = \max\left\{0, \min\left\{1, \frac{\langle u-u_t-\gamma g, v_t-u_t\rangle}{\|v_t-u_t\|^2}\right\}\right\}\]  \label{alpha}
%		\ENDIF
%		\ENDFOR
%	\end{algorithmic}
%\end{algorithm}

\begin{theorem}[ 2 of \cite{hazan2016variance}]
	With the following parameters (where $D_s$ is defined later below):
	\[\alpha_{s,t} = \frac{2}{t+1}, \hspace{20pt} \gamma_{s,t} = \frac{t}{3L}, \hspace{20pt} \eta_{s,t} = \frac{2D_s^2}{3T_s}\]
	Algorithm \ref{storc} ensures $\mathbb{E}\left[f(\tilde{x}^S)-f(x^*)\right]\leq \frac{LD^2}{2^{S+1}}$ if any of the following three cases holds:
	\begin{itemize}
		\item[(a)] $\nabla f(x^*) = 0$ and $D_s = D$, $T_s = \lceil 2^{s/2+2}\rceil$, $m_{s,t} = 900T_s$.
		\item[(b)] $f$ is $G$-Lipschitz and $D_s = D$, $T_s = \lceil 2^{s/2+2}\rceil$, $m_{s,t} = 700T_s + \frac{24T_sG(t+1)}{LD}$.
		\item[(c)] $f$ is $\tau$-strongly convex and $D_s = \frac{LD^2}{\tau 2^{s-1}}$, $T_s = \lceil\sqrt{\frac{32L}{\tau}}\rceil$, $m_{s,t} = \frac{5600T_sL}{\tau}$.
	\end{itemize}
\end{theorem}

From the following proof (especially (\ref{key})), we can see clearly how the decrease of $\alpha_{s,t}$ helps lower down the linear oracle complexity and raise the gradient query complexity.
\begin{lemma}[ 3 of \cite{hazan2016variance}]
	Suppose $0\leq D_s\leq D$ is such that $\mathbb{E}\left[\|\bar{x}_0-x^*\|^2\right] \leq D_s^2$. For any $t$, we have $\mathbb{E}\left[f(\bar{x}_t)-f(x^*)\right]\leq\frac{8LD_s^2}{t(t+1)}$ if $\mathbb{E}\left[\|G_k-\nabla f(\underline{x}_k)\|^2\right]\leq\frac{L^2D_s^2}{T_s(k+1)^2}$ for all $k\leq t$.
\end{lemma}
\begin{proof}
	Since $f$ is $L$-smooth, then we have
	\begin{equation}
		\begin{aligned}
			&f(\bar{x}_t)\leq l_{f}(\underline{x}_t,\bar{x}_t)+\frac{L}{2}\|\bar{x}_t-\underline{x}_t\|^2\\
			\overset{\textrm{\ding{172}}}{=}&(1-\alpha_{s,t})l_{f}(\underline{x}_t,\bar{x}_{t-1})+\alpha_{s,t}l_{f}(\underline{x}_t,x^*)+\alpha_{s,t}\langle \nabla f(\underline{x}_t), x_t-x^*\rangle+\frac{L\alpha_{s,t}^2}{2}\|x_t-x_{t-1}\|^2\\
			\overset{\textrm{\ding{173}}}{\leq}&(1-\alpha_{s,t})f(\bar{x}_{t-1})+\alpha_{s,t}f(x^*)+\alpha_{s,t}\langle \nabla f(\underline{x}_t), x_t-x^*\rangle+\frac{L\alpha_{s,t}^2}{2}\|x_t-x_{t-1}\|^2\\
			=& (1-\alpha_{s,t})f(\bar{x}_{t-1})+\alpha_{s,t}f(x^*)+\alpha_{s,t}\langle G_t, x_t-x^*\rangle+\frac{L\alpha_{s,t}^2}{2}\|x_t-x_{t-1}\|^2 + \alpha_{s,t}\langle \delta_t, x^*-x_t\rangle\\
			\overset{\textrm{\ding{174}}}{\leq}&(1-\alpha_{s,t})f(\bar{x}_{t-1})+\alpha_{s,t}f(x^*)+\frac{\alpha_{s,t}}{\gamma_{s,t}}\eta_{s,t} - \frac{\alpha_{s,t}}{\gamma_{s,t}}\langle x_t-x_{t-1}, x_t-x^*\rangle\\&+\frac{L\alpha_{s,t}^2}{2}\|x_t-x_{t-1}\|^2 + \alpha_{s,t}\langle \delta_t, x^*-x_t\rangle\\
			=&(1-\alpha_{s,t})f(\bar{x}_{t-1})+\alpha_{s,t}f(x^*)+\frac{\alpha_{s,t}}{\gamma_{s,t}}\eta_{s,t} + \frac{\alpha_{s,t}}{2\gamma_{s,t}}\left(\|x_{t-1}-x^*\|^2-\|x_t-x^*\|^2\right)\\&+\frac{\alpha_{s,t}}{2}\left[\left(L\alpha_{s,t}-\frac{1}{\gamma_{s,t}}\right)\|x_t-x_{t-1}\|^2 + 2\langle \delta_t, x_{t-1}-x_t\rangle + 2\langle \delta_t, x^*-x_{t-1}\rangle\right]\\
			\overset{\textrm{\ding{175}}}{\leq}&(1-\alpha_{s,t})f(\bar{x}_{t-1})+\alpha_{s,t}f(x^*)+\frac{\alpha_{s,t}}{\gamma_{s,t}}\eta_{s,t} + \frac{\alpha_{s,t}}{2\gamma_{s,t}}\left(\|x_{t-1}-x^*\|^2-\|x_t-x^*\|^2\right)\\&+\frac{\alpha_{s,t}}{2}\left[\frac{\gamma_{s,t}\|\delta_t\|^2}{1-L\alpha_{s,t}\gamma_{s,t}} + 2\langle \delta_t, x^*-x_{t-1}\rangle\right]\\
		\end{aligned}
	\end{equation}
	where \ding{172} comes from the definition of $\underline{x}_t$ and $x_t$, \ding{173} comes from the convexity of $f$, \ding{174} comes from Line \ref{condg} of Algorithm \ref{storc}, \ding{175} comes from the fact that $b\langle u, v\rangle-a\|v\|^2/2\leq b^2\|u\|^2/(2a)$. Note that $\mathbb{E}\left[\langle \delta_t, x^*-x_{t-1}\rangle\right] = 0$. So with the condition $\mathbb{E}\left[\|\delta_t\|^2\|\right]\leq\frac{L^2D_s^2}{T_s(k+1)^2} \overset{\textrm{def}}{=} \sigma_t^2$ we arrive at
	\begin{equation}
		\begin{aligned}
			&\mathbb{E}\left[f(\bar{x}_t)-f(x^*)\right]\\
			\leq& (1-\alpha_{s,t})\mathbb{E}\left[f(\bar{x}_{t-1})-f(x^*)\right] \\&+ \alpha_{s,t}\left[\frac{1}{\gamma_{s,t}}\eta_{s,t} + \frac{1}{2\gamma_{s,t}}\left(\mathbb{E}\left[\|x_{t-1}-x^*\|^2\right]-\mathbb{E}\left[\|x_t-x^*\|^2\right]\right) + \frac{\gamma_{s,t}\sigma_t^2}{2\left(1-L\alpha_{s,t}\gamma_{s,t}\right)}\right]
		\end{aligned}
	\end{equation}
	Now we define $\Gamma_{t} = \Gamma_{t-1}\left(1-\alpha_{s,t}\right)$ when $t>1$ and $\Gamma_1 = 1$. By induction, one can verify $\Gamma_{t} = \frac{2}{t(t+1)}$ and the following:
	\begin{equation}
		\begin{aligned}
			&\mathbb{E}\left[f(\bar{x}_t)-f(x^*)\right]\\
			\leq& \Gamma_{t}\sum_{k=1}^{t}\frac{\alpha_{s,k}}{\Gamma_k}\left[\frac{1}{\gamma_{s,k}}\eta_{s,k} + \frac{1}{2\gamma_{s,k}}\left(\mathbb{E}\left[\|x_{k-1}-x^*\|^2\right]-\mathbb{E}\left[\|x_k-x^*\|^2\right]\right) + \frac{\gamma_{s,k}\sigma_k^2}{2\left(1-L\alpha_{s,k}\gamma_{s,k}\right)}\right]
		\end{aligned}
	\end{equation}
	which is at most
	\begin{equation}	\label{key}
		\begin{aligned} &\Gamma_{t}\sum_{k=1}^{t}\frac{\alpha_{s,k}}{\Gamma_k}\left[\frac{1}{\gamma_{s,k}}\eta_{s,k} + \frac{\gamma_{s,k}\sigma_k^2}{2\left(1-L\alpha_{s,k}\gamma_{s,k}\right)}\right]\\ &+ \frac{\Gamma_t}{2}\left[\frac{\alpha_{s,1}}{\gamma_{s,1}\Gamma_1}\mathbb{E}\left[\|x_0-x^*\|^2\right]+\sum_{k=2}^{t}\left(\frac{\alpha_{s,k}}{\gamma_{s,k}\Gamma_k}-\frac{\alpha_{s,k-1}}{\gamma_{s,k-1}\Gamma_{k-1}}\right)\mathbb{E}\left[\|x_{k-1}-x^*\|^2\right]\right]
		\end{aligned}
	\end{equation}
	Finally plugging in the parameters $\alpha_{s,k}, \gamma_{s,k}, \eta_{s,k}, \Gamma_k$ and the bound $\mathbb{E}\left[\|\bar{x}_0-x^*\|^2\right] \leq D_s^2$ concludes the proof:
	\begin{equation}
		\begin{aligned}
			&\mathbb{E}\left[f(\bar{x}_t)-f(x^*)\right]\leq\frac{2}{t(t+1)}\sum_{k=1}^{t}k\left[\frac{2LD_s^2}{T_sk} + \frac{LD_s^2}{2T_s(k+1)}\right] + \frac{3LD_s^2}{t(t+1)} \leq \frac{8LD_s^2}{t(t+1)}
		\end{aligned}
	\end{equation}
\end{proof}

In (\ref{key}), the factor before $\eta_{s,k}$ is $\frac{1}{\gamma_{s,k}}
$, which is $\mathcal{O}\left(\frac{1}{k}\right)$. Thus $\eta_{s,t}$ can be chosen $\mathcal{O}\left(t\right)$ larger, which leads to lower linear oracle complexity. However, the factor before the variance $\sigma_k^2$ is $\gamma_{s,k}$, which is $\mathcal{O}\left(k\right)$. Thus $\sigma_k^2$ has to be $\mathcal{O}\left(k\right)$ smaller. From \citep{hazan2016variance} we know $\sigma_t^2$ is proportional to $\frac{1}{m_{s,t}}$. Thus $m_{s,t}$ has to be chosen $\mathcal{O}\left(t\right)$ larger, which leads to higher gradient complexity.

\end{document}